%% file: main.tex
\documentclass[11pt, a4paper, oneside, reqno]{amsart}

\input{style}
\input{comments}

\usepackage{booktabs}
\usepackage{setspace}
\usepackage[margin=1in]{geometry}

\title[A Geometric Unification of Distributionally Robust Covariance Estimators]{ {\Large A Geometric Unification\\of Distributionally Robust Covariance Estimators:\\Shrinking the Spectrum by Inflating the Ambiguity Set}}

 \author{Man-Chung Yue, Yves Rychener, Daniel Kuhn, Viet Anh Nguyen}
 \thanks{The authors are with the University of Hong Kong (\texttt{mcyue@hku.hk}), the Ecole Polytechnique F\'{e}d\'{e}rale de Lausanne (\texttt{yves.rychener, daniel.kuhn@epfl.ch}), and the Chinese University of Hong Kong (\texttt{nguyen@se.cuhk.edu.hk}).
 }

\date{\today}

\begin{document}

\begin{abstract}
The state-of-the-art methods for estimating high-dimensional covariance matrices all shrink the eigenvalues of the sample covariance matrix towards a data-insensitive shrinkage target. The underlying shrinkage transformation is either chosen {\em heuristically}---without compelling theoretical justification---or {\em optimally} in view of restrictive distributional assumptions. In this paper, we propose a principled approach to construct covariance estimators without imposing restrictive assumptions. That is, we study distributionally robust covariance estimation problems that minimize the worst-case Frobenius error with respect to all data distributions close to a nominal distribution, where the proximity of distributions is measured via a divergence on the space of covariance matrices. {We identify conditions} on this divergence under which the resulting minimizers represent shrinkage estimators. We show that the corresponding shrinkage transformations are intimately related to the geometrical properties of the underlying divergence. We also prove that our robust estimators are efficiently computable and asymptotically consistent and that they enjoy finite-sample performance guarantees. We exemplify our general methodology by synthesizing explicit estimators induced by the Kullback-Leibler, Fisher-Rao, and Wasserstein divergences. Numerical experiments based on synthetic and real data show that our robust estimators are competitive with state-of-the-art estimators. 
\end{abstract}

\maketitle


\section{Introduction}
The covariance matrix $\cov_0$ of a random vector $\xi\in\R^p$ is a fundamental summary statistic that captures the dispersion of~$\xi$. Together with the mean vector~$\mu_0$, it characterizes a unique member of the family of Gaussian distributions, which occupies the central stage in statistics and probability theory. Hence, any probabilistic model involving Gaussian distributions requires an estimate of~$\cov_0$ as an input. For example, Gaussian distributions are ubiquitous in finance ({\em e.g.}, in portfolio theory~\cite{markowitzportfolio}), in statistical learning ({\em e.g.}, in linear and quadratic discriminant analysis~\cite[\S~4.3]{hastie2009elements}) or control and signal processing ({\em e.g.}, in Kalman filtering~\cite{kalmanfilter}). In addition, $\cov_0$ is intimately related to the correlation matrix, including the Pearson correlation coefficients~\cite{pearson1895correlation}, and it permeates medical statistics~\cite{taylor1990interpretation} and correlation network analysis~\cite{eguiluz2005scale, mantegna1999hierarchical} etc.

If the distribution~$\trueP$ of $\xi$ is {\em known}, then the mean vector $\mu_0 = \EE_{\trueP}[\xi]$ and the covariance matrix $\cov_0 = \EE_{\trueP}[(\xi - \mu_0)(\xi - \mu_0)^\top]$ can be obtained by evaluating the relevant integrals with respect to~$\trueP$---either analytically or via numerical integration quadratures. If~$\trueP$ is {\em unknown}, however, one typically has to estimate~$\mu_0$ and~$\cov_0$ from~$n$ independent samples $\widehat{\xi}_1,\ldots,\widehat{\xi}_n \sim \trueP$. Arguably the simplest estimators for~$\mu_0$ and~$\cov_0$ are the sample mean $\widehat{\mu}_{\mathrm{SA}}=\frac{1}{n}\sum_{i=1}^n\widehat{\xi}_i$ and the sample covariance matrix $\covsa_{\mathrm{SA}} =\frac{1}{n-1}\sum_{i=1}^n(\widehat{\xi}_i-\widehat{\mu}_{\mathrm{SA}})(\widehat{\xi}_i-\widehat{\mu}_{\mathrm{SA}})^\top$, respectively. An elementary calculation shows that~$\covsa_{\mathrm{SA}}$ is unbiased. Up to scaling, $\covsa_{\mathrm{SA}}$ further coincides with the maximum likelihood estimator for~$\cov_0$ provided that~$\trueP$ constitutes a normal distribution.  
In~1975, much to the surprise of statisticians, Charles Stein showed that one can strictly reduce the mean squared error of~$\covsa_{\mathrm{SA}}$ by shrinking it towards a constant matrix independent of the data~\cite{james1992estimation, stein1975estimation}. Even though it improves the mean squared error, Stein's shrinkage transformation suffers from two major shortcomings, that is, it may alter the order of the estimator's eigenvalues and may even render some eigenvalues negative~\cite{ref:rajaratnam2016theoretical}. 
Nonetheless, since Stein's surprising discovery, the study of shrinkage estimators embodies an important research area in statistics.

Note also that~$\covsa_{\mathrm{SA}}$ is ill-conditioned if~$p\lesssim n$ and even singular if~$p>n$~\cite{van1961certain}. Indeed, as~$\covsa_{\mathrm{SA}}$ is unbiased and as the maximum eigenvalue function is convex on the space of symmetric matrices, Jensen's inequality ensures that the largest eigenvalue of $\covsa_{\mathrm{SA}}$ exceeds, in expectation, the largest eigenvalue of $\cov_0$. Similarly, the smallest eigenvalue of $\covsa_{\mathrm{SA}}$ undershoots, in expectation, the smallest eigenvalue of $\cov_0$. Hence, the condition number of $\covsa_{\mathrm{SA}}$, defined as the ratio of its largest to its smallest eigenvalue, tends to exceed the condition number of~$\cov_0$. This effect is most pronounced if $\cov_0$ is (approximately) proportional to the identity matrix~$I_p$ and is exacerbated with increasing dimension~$p$. A simple and effective method to improve the condition number is to construct a {\em linear} shrinkage estimator by forming a convex combination of~$\covsa_{\mathrm{SA}}$ and a data-insensitive shrinkage target such as $\frac{1}{p}\Tr{\covsa_{\mathrm{SA}}} I_p$ \cite{ref:ledoi2004well}. Other popular shrinkage targets include the constant correlation model~\cite{ledoit2004honey}, that is, a modified sample covariance matrix under which all pairwise correlations are equalized, the single index model~\cite{ledoit2003improved}, that is, the sum of a rank-one and a diagonal matrix representing systematic and idiosyncratic risk factors as in Sharpe's single index model~\cite{sharpe1963simplified}, and the diagonal matrix model~\cite{touloumis2015nonparametric}, that is, the diagonal matrix that contains all sample eigenvalues on its main diagonal. The shrinkage weight of~$\covsa_{\mathrm{SA}}$ is usually tuned to minimize the Frobenius risk, that is, the expected squared Frobenius norm distance between the estimator and~$\cov_0$.  Linear shrinkage estimators can be computed highly efficiently, improve the condition number of the sample covariance matrix, and are guaranteed to have full rank even if~$p>n$. 

In the remainder of the paper, we focus on covariance estimators that depend on the samples only indirectly through the sample covariance matrix. This assumption is unrestrictive. Indeed, it is satisfied by all commonly used covariance estimators. Moreover, it comes at no loss of generality if $\trueP$ is a normal distribution, in which case~$\covsa_{\mathrm{SA}}$ constitutes a sufficient statistic for~$\cov_0$. Without prior information about the eigenvectors of $\cov_0$, it is natural to restrict attention to rotation equivariant estimators. Rotation equivariance means that evaluating the estimator $\covsa$ on the rotated dataset $\{R\widehat{\xi}_i\}_{I=1}^N$ is equivalent to evaluating the rotated estimator $R\covsa R^\top$ on the the original dataset $\{\widehat{\xi}_i\}_{i=1}^n$ for any rotation matrix $R$. One can show that any rotation equivariant estimator~$\covsa$ commutes with the sample covariance matrix~$\covsa_{\mathrm{SA}}$, that is, $\covsa_{\mathrm{SA}}$ and $\covsa$ share the same eigenvectors, and the spectrum of~$\covsa$ can be viewed as a transformation of the spectrum of~$\covsa_{\mathrm{SA}}$\cite[Lemma~5.3]{perlman2007stat}. Such spectral transformations are referred to as shrinkage transformations. Note that the linear shrinkage estimators discussed above are rotation equivariant only if the shrinkage target commutes with~$\covsa_{\mathrm{SA}}$.

If $\trueP$ is governed by a spiked covariance model, that is, if $\trueP$ is Gaussian, $p$ and $n$ tend to infinity at an asymptotically constant ratio and 
$\cov_0$ constitutes a fixed-rank perturbation of the identity matrix, then one can use results from random matrix theory to construct the best rotation equivariant estimators in closed form for a broad range of different loss functions \cite{donoho2018spiked}. Nonlinear shrinkage estimators that are asymptotically optimal with respect to the Frobenius loss can also be constructed in the absence of any normality assumptions, and they can significantly improve on linear shrinkage estimators if the eigenvalue spectrum of~$\cov_0$ is dispersed~\cite{ledoit2012nonlinear, ledoit2020analytical}. Similarly, one can construct optimal shrinkage estimators for the {\em inverse} covariance matrix~$\cov_0^{-1}$, which is usually termed the precision matrix; see~\cite{bodnar2016direct,ledoit2022quadratic}. However, the available statistical guarantees for all shrinkage estimators described above are {\em asymptotic} and depend on assumptions about the structure of~$\trueP$ and/or the convergence properties of the spectral distribution of~$\covsa_{\mathrm{SA}}$, which may be difficult to check in practice.

In this paper, we propose a flexible and principled approach to estimate the covariance matrix~$\cov_0$ by using ideas from distributionally robust optimization (DRO). Specifically, our approach generates a rich family of covariance matrix estimators corresponding to different ambiguity sets that can encode prior distributional information. All emerging estimators are rotation equivariant and thus represent nonlinear shrinkage estimators. In addition, they all improve the condition number of the sample covariance matrix, are invertible, and preserve the order of the sample eigenvalues. They also offer finite sample guarantees on the prediction loss and are asymptotically consistent. These appealing properties are not enforced ad hoc but emerge naturally from the solution of a principled distributionally robust estimation model. We emphasize that our results do not rely on any restrictive assumptions such as the requirement that~$\trueP$ is Gaussian or that the spectral distribution of~$\covsa_{\mathrm{SA}}$ converges to a well-defined limit as~$p$ and~$n$ tend to infinity at a constant ratio.

To develop the distributionally robust estimation model to be studied in this paper, we first express the unknown true covariance matrix~$\cov_0$ as the minimizer of a stochastic optimization problem involving the unknown probability distribution~$\trueP$. Specifically, adopting the standard assumption that $\mu_0=\mathbb E_{\trueP}[\xi]=0$ \cite{ref:ledoi2004well,ledoit2012nonlinear, ledoit2017nonlinear,ledoit2022quadratic} and noting that the squared Frobenius norm is strictly convex, we obtain
\[
    \{\cov_0\} =\Argmin_{X \in \PSD^p}~\| X - \cov_0 \|_{\mathrm{F}}^2= \Argmin_{X \in \PSD^p}~\Tr{X^2} - 2 \Tr{X \cov_0} = \Argmin\limits_{X \in \PSD^p}~\Tr{X^2} - 2 \Tr{X \EE_{\trueP}[\xi \xi^\top]}.
\]
If we could solve the stochastic optimization problem on the right-hand side of the above expression, we could precisely recover the ideal estimator~$X\opt=\cov_0$. This is impossible, however, because the distribution~$\trueP$ needed to evaluate the stochastic optimization problem's objective function is unknown. Nevertheless, replacing $\trueP$ with a nominal distribution~$\psa$ constructed from the $n$ training samples yields the nominal estimation model
\begin{equation}
\label{eq:nominal}
    \min\limits_{X \in \PSD^p}~\Tr{X^2} - 2 \EE_{\psa}\left[\xi^\top X \xi \right],
\end{equation}
which requires no unavailable inputs. An elementary calculation shows that~\eqref{eq:nominal} is uniquely solved by~$\covsa=\EE_{\psa}[\xi \xi^\top]$, which is the covariance matrix of~$\xi$ under the nominal distribution~$\psa$, provided that~$\msa=\mathbb E_{\psa}[\xi]=0$. Of course, characterizing~$\covsa$ as a minimizer of~\eqref{eq:nominal} has no conceptual or computational benefits because we have to compute the integral~$\EE_{\psa}[\xi \xi^\top]$ already to evaluate the objective function of~\eqref{eq:nominal}. Nevertheless, the nominal estimation problem~\eqref{eq:nominal} is useful because it allows us to construct a broad range of nonlinear shrinkage estimators in a principled and systematic manner by robustifying the prediction loss.

Any nominal distribution~$\psa$ constructed from a finite dataset must invariably differ from the true data-generating distribution~$\trueP$. Estimation errors in~$\psa$ are conveniently captured by an ambiguity set of the form
\begin{equation}\label{eq:uncertainty-set}
    \BB_\eps(\psa) = \left\{ \mathbb{Q}: \mathbb{Q} \sim (0, \cov),  \; D(\cov, \covsa) \leq \eps \right\},
\end{equation}
where $\mathbb{Q} \sim (0, \cov)$ indicates that $\xi$ has mean $0$ and covariance matrix $\cov$ under $\mathbb{Q}$, and $D$ represents a divergence on the space of positive semidefinite matrices. Divergences are general distance-like functions that are non-negative and satisfy the identity of indiscernibles (that is, they satisfy $D(\cov, \covsa) = 0$ if and only if $\cov = \covsa$). However, divergences may fail to be symmetric and may violate the triangle inequality. Intuitively, $\BB_\eps(\psa)$ can be viewed as a divergence ball of radius~$\eps\geq 0$ around~$\psa$ in the space of probability distributions. Robustifying the nominal estimation problem~\eqref{eq:nominal} against all distributions in~$\BB_\eps(\psa)$ yields the following DRO problem.
\be \label{eq:dro}
    \min\limits_{X \in \PSD^p}~\sup\limits_{\mathbb{Q} \in \mbb{U}_\eps(\psa)}~\Tr{X^2} - 2 \EE_{\mathbb{Q}}\left[\xi^\top X \xi \right].
\ee
Problem~\eqref{eq:dro} seeks an estimator~$X$ that minimizes the worst-case expected prediction loss across all distributions in~$\mbb{U}_\eps(\psa)$. Note that if~$\eps=0$, then the DRO problem~\eqref{eq:dro} collapses to the nominal estimation problem~\eqref{eq:nominal} because the divergence~$D$ satisfies the identity of indiscernibles, which ensures that~$\mbb{U}_0(\psa)=\{\psa\}$. Hence,~\eqref{eq:dro} embeds~\eqref{eq:nominal} into a family of estimation models parametrized by~$D$ and~$\eps$. Moreover, DRO models naturally bridge optimization and statistics in that they offer an intuitive way to derive generalization bounds. Indeed, if~$\eps$ is tuned to ensure that~$\mbb{U}_\eps(\psa)$ contains the data-generating distribution~$\trueP$ with high confidence~$1-\beta$, then the optimal value of the DRO problem~\eqref{eq:dro} provides a $(1-\beta)$-upper confidence bound on the prediction loss of its unique minimizer~$X\opt$ under~$\trueP$ \cite{ref:esfahani2018data}. Stronger generalization bounds that do not require~$\trueP$ to belong to~$\mbb{U}_\eps(\psa)$ are provided in \cite{blanchet2021confidence, gao2022curse}. Even if the ambiguity set does not contain $\trueP$, DRO models tend to yield high-quality solutions because there is a deep connection between robustification and regularization~\cite{ref:gao2017wasserstein, shafiee2023ot-regularization, ref:shafieezadeh2017regularization}. This connection may also explain the empirical success of DRO in statistical estimation~\cite{ref:blanchet2021statistical,ref:kuhn2019wasserstein, taskesen2021sequential}.

The flexibility to choose the divergence~$D$ underlying the ambiguity set~$\mbb{U}_\eps(\psa)$ is both a blessing and a curse. On the one hand, $D$ can encode prior distributional information and thus lead to better estimators. On the other hand, the family of divergences is vast. Hence, the choice of a suitable instance could overwhelm the modeler. Given the statistical estimation task at hand, it makes sense to restrict attention to divergences that admit a statistical interpretation. Many popular divergences on the space of covariance matrices are obtained by restricting a divergence on the space of probability distributions to the family of normal distributions. For example, the Kullback-Leibler divergence, the 2-Wasserstein distance, or the Fisher-Rao distance between zero-mean normal distributions all admit closed-form formulas in terms of the distributions' covariance matrices. These `Gaussian' divergences are popular because they are conducive to tractable DRO models in risk management~\cite{ghaoui2003worst,ref:nguyen2021mean}, ethical machine learning~\cite{bui2023coverage, vu2022distributionally}, likelihood evaluation~\cite{nguyen2019calculating, nguyen2019optimistic}, Kalman filtering~\cite{ref:zorzi2017robust, ref:shafieezadeh2018wasserstein} and control~\cite{ref:taskesen2023lqg} etc. In addition, the shrinkage estimator for the {\em inverse} covariance matrix proposed in~\cite{ref:nguyen2018distributionally} also leverages a `Gaussian' divergence. Nonetheless, the approach proposed in this paper does {\em not} rely on the assumption that~$\mathbb{P}$ is Gaussian.

The main contributions of this paper can be summarized as follows.
\begin{itemize}[leftmargin=5mm]
    \item We propose a rich family of distributionally robust covariance matrix estimators. Each estimator is defined as a solution of~\eqref{eq:dro} for a particular ambiguity set of the form~\eqref{eq:uncertainty-set}. Here, the nominal covariance matrix~$\covsa$ characterizes the {\em center}, the divergence~$D$ determines the {\em geometry}, and the radius~$\eps$ determines the {\em size} of the ambiguity set. We demonstrate that all such estimators are well-defined, unique and efficiently computable under {few} structural assumptions on~$D$ and mild regularity conditions on~$\covsa$ and~$\eps$. 
    
    \item We prove that our distributionally robust covariance matrix estimators constitute nonlinear shrinkage estimators, that is, they have the same eigenbasis as~$\covsa$, and their eigenvalues are obtained by shrinking the spectrum of~$\covsa$ towards~$0$ by using a nonlinear shrinkage transformation depending on~$D$ and a shrinkage intensity depending on~$\eps$. We further prove that these estimators improve the condition number of~$\covsa$.
    
    \item We identify various divergences commonly used in statistics, machine learning and information theory that satisfy the requisite regularity conditions. To this end, we invoke a generalization of Sion's classical minimax theorem from Euclidean spaces to Riemannian manifolds. We also exemplify our framework by deriving explicit analytical formulas for the distributionally robust covariance estimators induced by the Kullback-Leibler divergence, the 2-Wasserstein distance and the Fisher-Rao distance.

    \item We prove that, if~$\eps$ scales with the sample size~$n$ as~$\mc O(n^{-\frac{1}{2}})$, then the proposed estimators are strongly consistent and enjoy finite-sample performance guarantees at a fixed confidence level. Numerical experiments based on synthetic as well as real data for portfolio optimization and binary classification tasks suggest that our robust estimators are competitive with state-of-the-art estimators from the literature. 
\end{itemize}

The first robustness interpretation of a shrinkage estimator was discovered in the context of {\em inverse} covariance matrix estimation~\cite{ref:nguyen2018distributionally}. Specifically, it was shown that a particular nonlinear shrinkage estimator can be obtained by robustifying the maximum likelihood estimator for~$\cov_0^{-1}$ across all Gaussian distributions of the training samples within a prescribed Wasserstein ball. This result critically relies on the restrictive assumption that the unknown data-generating distribution, the nominal distribution as well as all other distributions in the Wasserstein ball are Gaussian. In addition, this result has not been extended to more general ambiguity sets based on other divergences beyond the 2-Wasserstein distance, thus limiting the modeler's flexibility. 

In this paper we show that a broad spectrum of shrinkage estimators for~$\cov_0$ can be obtained from a versatile DRO model that does not rely on restrictive normality assumptions. That is, we seek the most general conditions on the DRO model under which a shrinkage effect emerges. In addition, we uncover a deep connection between the geometry of the ambiguity set, which is determined by the choice of the divergence~$D$, and the nonlinear shrinkage transformation of the corresponding distributionally robust estimator.

\textbf{Notation.} We use $\overline{\R} = \R\cup\{+\infty \}$ as a shorthand for the extended real line. The space of $p$-dimensional real vectors and its subsets of (entry-wise) non-negative and positive vectors are denoted by~$\R^p$, $\R^p_+$, and~$\R^p_{++}$, respectively. Similarly, the space of symmetric matrices in~$\R^{p\times p}$, as well as its subsets of positive semidefinite and positive definite matrices, are denoted by~$\Sym^p$, $\PSD^p$, and~$\PD^p$, respectively. The group of orthogonal matrices in~$\R^{p\times p}$ is denoted by $\mathcal{O}_p$, and $I_p$ stands for the identity matrix in~$\R^{p\times p}$. For any $x\in\R^p$, we use~$x^\downarrow$ and~$x^\uparrow$ to denote the vectors obtained by rearranging the entries of $x$ in non-increasing and non-decreasing order, respectively. The trace of a matrix $S\in\Sym^p$ is defined as $\Tr{S}=\sum_{i=1}^p S_{ii}$. 
Finally, $\|M\|=\sup_{\|v\|_2=1}\|Mv\|_2$ and $\|M\|_\mathrm{F}=\Tr{M^\top M}^{\half}$ stand for the spectral norm and the Frobenius norm of~$M$, respectively.

\section{Overview of Main Results} \label{sec:cse}

The distributionally robust estimation problem~\eqref{eq:dro} perturbs---and thereby hopefully improves---the nominal estimator $\covsa$ in view of the divergence~$D$. We now derive a simple reformulation of~\eqref{eq:dro} as a standard minimization problem, and we informally outline the main properties of the corresponding optimal solution, which will be established rigorously in the remainder of the paper. From now on, the nominal covariance matrix $\covsa$ can be viewed as any na\"ive initial estimator for the covariance matrix~$\cov_0$. The construction of $\covsa$ from the samples $\widehat{\xi}_1,\ldots,\widehat{\xi}_n$ is immaterial for most of our discussion. As the loss function underlying problem~\eqref{eq:dro} is quadratic in~$\xi$ and as~$\mathbb E_{\mathbb{Q}}[\xi]=0$, its expected value depends on~$\mathbb{Q}$ only indirectly through the covariance matrix~$\cov=\mathbb E_{\mathbb{Q}}[\xi\xi^\top]$. Thus, the DRO problem~\eqref{eq:dro} is equivalent to the robust covariance estimation~problem
\begin{equation}
\label{eq:CSE}
\min_{\X \in \PSD^p} \; \max_{\cov \in \B_\eps (\covsa)} \Tr{X^2} - 2 \Tr{ \cov  X }
\end{equation}
with uncertainty set
\begin{equation}\label{eq:uncertainty_set}
    \B_\eps (\covsa) = \left\{ \cov\in \PSD^p: D(\cov, \covsa) \le \eps \right\}.
\end{equation}
We stress that the divergence function~$D$ may fail to be symmetric, that is, $D(X, Y)$ may differ from~$D(Y, X)$. It is therefore important to remember the convention that $\covsa$ is the {\em second} argument of~$D$ in the definition of~$\mc B_\eps(\covsa)$. Note also that~$\mc B_\eps(\covsa)$ grows with the size parameter~$\eps$ and collapses to the singleton~$\{\covsa\}$ for~$\eps = 0$. The robust estimation problem~\eqref{eq:CSE} constitutes a zero-sum game between the statistician, who moves first and chooses the estimator~$X$, and nature, who moves second and chooses the covariance matrix~$\cov$. The following dual estimation problem is obtained by interchanging the order of minimization and maximization in~\eqref{eq:CSE}.
\begin{equation}
\label{eq:dual-CSE}
\max_{\cov \in \B_\eps (\covsa)} \;\min_{\X \in \PSD^p}\Tr{X^2} - 2 \Tr{ \cov  X }
\end{equation}
From now on, we denote by~$X\opt$ and~$\cov\opt$ the optimal solutions of the primal and dual estimation problems~\eqref{eq:CSE} and~\eqref{eq:dual-CSE}, respectively. In Section~\ref{sec:a} below, {we will identify few conditions on~$D$} and~$\covsa$ under which~$X\opt$ and~$\cov\opt$ are indeed guaranteed to exist and to be unique. If the uncertainty set~$\B_\eps (\covsa)$ is convex and compact, then strong duality prevails (that is, \eqref{eq:CSE} and~\eqref{eq:dual-CSE} share the same optimal value) by Sion's classical minimax theorem. As several popular divergence functions are non-convex in their first argument and thus induce a non-convex uncertainty set~$\mc B_\eps(\covsa)$; however, we will invoke a generalized minimax theorem that guarantees strong duality under significantly more general conditions. Whenever strong duality holds, $(X\opt,\cov\opt)$ constitutes a Nash equilibrium of the zero-sum game between the statistician and nature~\cite[Lemma~36.2]{ref:rockafellar1997convex}. 

A cursory glance at its first-order optimality condition reveals that the inner minimization problem in~\eqref{eq:dual-CSE} is solved by~$X=\cov$. Hence, the inner minimum evaluates to~$-\Tr{\cov^2}=-\norm{ \cov }_{\mathrm{F}}^2$. Eliminating the factor~$-1$ further shows that~$\cov\opt$ solves the maximization problem~\eqref{eq:dual-CSE} if and only if it solves the minimization problem
\begin{align}
    \label{eq:matrix}\tag{P$_{\text{Mat}}$}
    \min_{\cov \in \PSD^p}~\left\{ \norm{ \cov }_{\mathrm{F}}^2 ~:~ \displaystyle D(\cov, \covsa) \le \eps
\right\}.
\end{align}
Thus, nature's Nash strategy~$\cov\opt$ can be computed by solving~\eqref{eq:matrix} instead of~\eqref{eq:dual-CSE}. By the defining properties of Nash strategies, the statistician's Nash strategy~$X\opt$ must be a best response to~$\cov\opt$, that is, $X\opt$ must solve the inner minimization problem in~\eqref{eq:dual-CSE} for~$\cov=\cov\opt$. However, the unique optimal solution of this minimization problem is~$\cov\opt$. In summary, this reasoning implies that if strong duality holds, then the Nash strategies~$X\opt$ and~$\cov\opt$ of the statistician and nature coincide and are both given by the unique minimizer of problem~\eqref{eq:matrix}.

Problem~\eqref{eq:matrix} is reminiscent of a ridge regression problem~\cite{hoerl1970ridge, van2015lecture}, which seeks an estimator that minimizes a weighted sum of a least squares fidelity term and a Frobenius norm regularization term. Indeed, problem~\eqref{eq:matrix} seeks a covariance matrix~$\cov$ with minimum Frobenius norm and a fidelity error of at most~$\eps$, where the fidelity of~$\cov$ with respect to the nominal covariance estimator~$\covsa$ is measured by the divergence~$D(\cov, \covsa)$. 

\bgroup
\def\arraystretch{1.8}
\begin{table}[H]
\resizebox{\linewidth}{!}{\begin{tabular}{|l||c|c|}
\hline
Divergence function & $D(\cov, \covsa)$  & Domain \\ \hline\hline
Kullback-Leibler / Stein~\cite{kullback1997information} & $\frac{1}{2}\left( \Tr{\covsa^{-1} \cov } - p + \log \det (\covsa \cov^{-1}) \right)$  & $\PD^p\times\PD^p$ \\ \hline
Wasserstein~\cite{givens1984class} & $\Tr{\cov + \covsa - 2 \big( \cov \covsa \big)^\half }$  & $\PSD^p\times\PSD^p$ \\ \hline
Fisher-Rao~\cite{atkinson1981rao} & $\left\| \log ( \covsa^{-\half} \cov \covsa^{-\half} ) \right\|_{\mathrm{F}}^2$  & $\PD^p\times\PD^p$ \\ \hline
Inverse Stein~\cite{kullback1997information} & $\frac{1}{2}\left( \Tr{\cov^{-1} \covsa} - p + \log \det (\cov \covsa^{-1}) \right)$  & $\PD^p\times\PD^p$ \\ \hline
Symmetrized Stein / Jeffreys divergence~\cite{jeffreys1946invariant} & $\frac{1}{2}\left( \Tr{\cov\covsa^{-1} + \covsa\cov^{-1}} -2p \right)$  & $\PD^p\times\PD^p$ \\ \hline
Quadratic / Squared Frobenius & $ \Tr{( \cov - \covsa )^2}$  & $\PSD^p\times\PSD^p$ \\ \hline
Weighted quadratic & $\Tr{ (\cov - \covsa)^2\covsa^{-1} }$  & $\PSD^p\times\PD^p$ \\ \hline
\end{tabular}}
\caption{Popular divergence functions and their domains. We adopt the convention from convex analysis that each divergence evaluates to~$+\infty$ outside of its domain.}
\label{table:structured_divergence}
\end{table}
\egroup

We now informally state our key result, which applies, among others, to all divergence functions of Table~\ref{table:structured_divergence}.
\begin{informaltheorem}[Distributionally robust estimator (informal)]\label{thm:general_CSE:informal}
    If $D$ is any divergence function from Table~\ref{table:structured_divergence}, the nominal covariance matrix~$\covsa$ satisfies a regularity condition, and~$\eps>0$ is not too large, then the distributionally robust estimator~$X\opt$ exists, is unique, and can be computed efficiently via the following procedure.
        \begin{enumerate} \item Compute the eigenvalues and the eigenvectors of the nominal covariance matrix~$\covsa$.
        \item  Construct the inverse shrinkage intensity~$\gamma\opt$ by solving a univariate nonlinear equation that depends only on the spectrum of $\covsa$.
        \item Shrink the eigenvalues of $\covsa$ by applying a nonlinear transformation that depends only on $\gamma\opt$.
        \item  Construct $X\opt$ by combining the eigenvectors found in step~(1) with the eigenvalues found in step~(3).
        \end{enumerate}
    The estimator~$X\opt$ constructed in this manner preserves the eigenvectors of $\covsa$, shrinks the eigenvalues of $\covsa$, and reduces the condition number of $\covsa$. Thus, it represents a nonlinear shrinkage estimator.
\end{informaltheorem}

Theorem~\ref{thm:general_CSE:informal} reveals that a wide range of nonlinear shrinkage estimators admit a robustness interpretation in the sense that they correspond to solutions of the distributionally robust estimation problem~\eqref{eq:dro} for different divergence functions. This insight is of interest from a statistical point of view because it relates nonlinear shrinkage estimators to distributional ambiguity sets, which can be used to derive new generalization bounds. Theorem~\ref{thm:general_CSE:informal} also implies that the distributionally robust estimation problem~\eqref{eq:dro} can be solved efficiently by diagonalizing~$\covsa$ and solving a univariate nonlinear equation, both of which are computationally cheap.

\section{Distributionally Robust Covariance Shrinkage Estimators}\label{sec:DRO-estimators}

This section formally introduces our distributionally robust estimation framework. Specifically, Section~\ref{sec:a} details all technical assumptions needed throughout the paper, Section~\ref{sec:form} formally states the main result, and Section~\ref{sec:property} describes several desirable properties of the emerging distributionally robust estimators.

\subsection{Assumptions} \label{sec:a}

The uncertainty set~$\mc B_\eps(\covsa)$ is non-convex for some choices of the divergence function~$D$. In these cases, we cannot use Sion's minimax theorem to establish strong duality between the primal and dual estimation problems~\eqref{eq:CSE} and~\eqref{eq:dual-CSE}, respectively. Instead, we will have to invoke a more nuanced minimax theorem. For now, we assume that such a minimax theorem is readily available.

\begin{assumption}[Minimax property]\label{assu:inf_sup}
The minimum of the primal estimation problem~\eqref{eq:CSE} coincides with the maximum of the dual estimation problem~\eqref{eq:dual-CSE}.
\end{assumption}

We will later see that Assumption~\ref{assu:inf_sup} is satisfied for all divergence functions listed in Table~\ref{table:structured_divergence}. In addition, we require~$D$ to constitute a spectral divergence in the sense of the following assumption.

\begin{assumption}[Spectral divergence]\label{assu:D_form}
The divergence function $D: \PSD^p\times \PSD^p \to \overline{\R}$ is non-negative, and satisfies the identity of indiscernibles, that is, for any $(X,Y) \in \mathrm{dom}(D)$ we have $D(X,Y) = 0$ if and only if $X = Y$. In addition, $D$ satisfies the following structural conditions.
\begin{enumerate}[label=(\alph*)]

\item\label{assu:D_form_ii} (Orthogonal equivariance) For any $X, Y \in \PSD^p$ and $V\in{\mathcal{O}_p}$ we have that $D(X, Y) = D(VX V^\top, VY V^\top)$. 

\item\label{assu:D_form_iii} (Spectrality) There exists a function $d:\R_+ \times \R_+ \to \overline{\R}$ such that 
\[
    D\left(\Diag(x), \Diag(y)\right) = \sum_{i=1}^p d(x_i, y_i) \quad \forall x, y\in \R_+^p
\]
and $d(a,b)$ is continuous\footnote{By convention, a continuous extended real-valued function must tend to~$\infty$ when approaching the boundary of its domain.} in $a$ for every $b>0$. In the following, we refer to $d$ as the generator of $D$.

\item\label{assu:D_form_iv} (Rearrangement property) For any $x, y \in \R_+^p$ and $V\in {\mathcal{O}_p}$ we have
\[ D\left( V \Diag(x^\uparrow) V^\top, \Diag(y^\uparrow ) \right) \ge D\left( \Diag(x^\uparrow) , \Diag(y^\uparrow ) \right). \]
If its left side is finite, this inequality becomes an equality if and only if $\Diag(x^\uparrow) = V \Diag(x^\uparrow) V^\top$.
\end{enumerate}
\end{assumption}

Assumptions~\ref{assu:D_form}\ref{assu:D_form_ii} and~\ref{assu:D_form}\ref{assu:D_form_iii} imply that if~$X$ and~$Y$ are simultaneously diagonalizable, then the divergence of $X$ with respect to~$Y$ depends only on the spectra of~$X$ and~$Y$ and the generator~$d$. Specifically, we have
\begin{equation}
\label{eq:D-versus-d}
 D(X, Y) = D(V\Diag(x)V^\top, V \Diag(y)V^\top) = D(\Diag(x),  \Diag(y)) = \sum_{i=1}^p d(x_i, y_i),  
\end{equation}
where the entries of the vectors~$x$ and~$y$ represent the eigenvalues and where the columns of the orthonormal matrix~$V$ represent the (common) eigenvectors of~$X$ and~$Y$, respectively. Note that the last two equalities in~\eqref{eq:D-versus-d} readily follow from~\ref{assu:D_form}\ref{assu:D_form_ii} and~\ref{assu:D_form}\ref{assu:D_form_iii}. Assumption~\ref{assu:D_form}\ref{assu:D_form_iii} further implies that if~$D$ is a spectral divergence on~$\PSD^p$, then its generator~$d$ is a spectral divergence on~$\R_+$. Indeed, restricting $x$ and $y$ to multiples of the vector of all ones reveals via Assumption~\ref{assu:D_form}\ref{assu:D_form_iii} that $\mathrm{dom}(d) = \{(a,b)\in \R_+^2 : (a I_d, b I_d) \in \mathrm{dom}(D)\}$ and that~$d$ inherits continuity, non-negativity and the identity of indiscernibles from~$D$. Orthogonal equivariance, spectrality, and the rearrangement inequality are trivially satisfied in the one-dimensional case. Finally, we point out that Assumption~\ref{assu:D_form}\ref{assu:D_form_iv} is reminiscent of the Hardy-Littlewood-Polyak rearrangement inequality~\cite{hardy1952inequalities}, which asserts that $(x^\uparrow)^\top {y^\downarrow} \le x^\top y \le (x^\uparrow)^\top {y^\uparrow}$ for any vectors~$x,y\in \mathbb{R}^p$.

Our results also require the following assumptions about the eigenvalues~$\xsa_1,\ldots,\xsa_p$ of the nominal covariance matrix $\covsa$ as well as about the radius~$\eps$ of the uncertainty set~$\mc B_\eps(\covsa)$.

\begin{assumption}[Regularity of input parameters]\label{assu:data}
The following hold.
\begin{enumerate}[label=(\alph*)]
\item\label{assu:data-d} For any $i = 1,\dots,p$ we have $(\xsa_i,\xsa_i)\in \mathrm{dom}(d)$. 
\item\label{assu:data-eps} The radius $\eps$ of the uncertainty set satisfies $0 < \eps < \bar{\eps}$, where $\bar{\eps} = \sum_{i=1}^p d(0, \xsa_i)$.
\end{enumerate}
\end{assumption}

Together with Assumptions~\ref{assu:D_form}\ref{assu:D_form_ii} and~\ref{assu:D_form}\ref{assu:D_form_iii}, Assumption~\ref{assu:data}\ref{assu:data-d} ensures that the nominal covariance matrix $\covsa$ is feasible in problem~\eqref{eq:matrix}. Indeed, inserting~$X=Y=\covsa$ into~\eqref{eq:D-versus-d} implies that~$D(\covsa,\covsa)=0$. This implies that $(\covsa,\covsa)\in\mathrm{dom}(D)$ and, more importantly, that the feasible region of problem~\eqref{eq:matrix} is non-empty. 
This assumption is not entirely innocent because some divergence functions from Table~\ref{table:structured_divergence} have domain~$\PD^p\times\PD^p$. In all these cases, Assumption~\ref{assu:data}\ref{assu:data-d} requires that~$\covsa$ has full rank and, if $\covsa$ is the sample covariance matrix, that the sample size~$n$ is at least as large as the dimension~$p$.
{
We emphasize that Assumption~\ref{assu:data}\ref{assu:data-d} does {\em not} generally imply that $n \ge p$. For instance, if $(0,0) \in \mathrm{dom}(d)$, then Assumption~\ref{assu:data}\ref{assu:data-d} holds even if~$n < p$. This situation arises if~$D$ is the Wasserstein or the quadratic divergence. Conversely, Assumption~\ref{assu:data}\ref{assu:data-d} may fail to hold even when~$n > p$. This happens, for example, if $(0,0) \notin \mathrm{dom}(d)$ and the nominal covariance matrix $\covsa$ is singular even though~$n > p$.} Assumption~\ref{assu:data}\ref{assu:data-eps} ensures that the radius $\eps > 0$ is small enough for the feasible region of the reformulated dual estimation problem~\eqref{eq:matrix} not to contain~$0$. Otherwise, problem~\eqref{eq:matrix} would trivially be solved by the nonsensical estimator~$X\opt=0$.

\begin{assumption}[Smoothness and convexity of the generator~$d$]\label{assu:d_convex}
For any $b > 0$, the function $d(\cdot, b)$ is twice continuously differentiable throughout~$\R_{++}$ and convex on the interval~$[0,b]$.
\end{assumption}

Assumption~\ref{assu:d_convex} implies that the domain of~$d(\cdot, b)$ contains $\R_{++}$ for every $b>0$. Hence, $d(a, b)$ can evaluate to~$+\infty$ only at~$a=0$, which means that the domain of $d(\cdot,b)$ is either~$\R_{+}$ or~$\R_{++}$. We emphasize that the convexity of $d(\cdot, b)$ on the interval~$[0,b]$ does {\em not} imply that problem~\eqref{eq:matrix} is convex. However, we will see below that this restricted convexity assumption helps us to reduce problem~\eqref{eq:matrix} to a convex program.

\subsection{Construction of the Distributionally Robust Estimator} \label{sec:form}

We need the following notation to restate Theorem~\ref{thm:general_CSE:informal} rigorously. We denote the $i$-th smallest eigenvalue of a symmetric matrix $S\in \Sym^p$ by ~$\lambda_i(S)$, and we use~$\lambda(S)=(\lambda_1(S),\ldots, \lambda_p(S))$ as a shorthand for the spectrum of~$S$. We also reserve the symbols~$\xsa_i = \lambda_i (\covsa)$ and~$\vsa_i$ for the non-negative eigenvalues and the corresponding orthonormal eigenvectors of the nominal covariance matrix~$\covsa$. In addition, we use~$\xsa = \lambda(\covsa)$ and~$\Vsa =(\vsa_1,\ldots,\vsa_p)$ to denote the nominal spectrum and the orthogonal matrix of the nominal eigenvectors, respectively. The nominal covariance matrix thus admits the spectral decomposition~$\covsa = \Vsa \Diag(\xsa) \Vsa^\top$. We also define the auxiliary function  $s:\mathbb{R}_+^2 \to \mathbb{R}$ corresponding to a divergence function with generator~$d$ via
\begin{equation} \label{eq:s}
    s(\gamma, b) = \begin{cases}
    \text{the unique solution $a\opt\geq 0$ of the equation }
    0 = 2a\opt + \gamma \frac{\partial d}{\partial a}  (a\opt, b)  & \text{if } b > 0\text{ and } \gamma > 0, \\
    0 & \text{if } b = 0 \text{ or } \gamma = 0.
    \end{cases}
\end{equation}
In the remainder of the paper, we refer to~$s$ as the {\em eigenvalue map}. We will see below that it is well-defined under Assumption~\ref{assu:d_convex}, which implies that the nonlinear equation in~\eqref{eq:s} has a unique solution whenever~$b>0$. We will also prove that~$s(\gamma, b) \le b$ for every~$\gamma,b\geq 0$, which means that it can be viewed as a shrinkage transformation that maps any input eigenvalue $b\geq 0$ to a smaller output eigenvalue~$s(\gamma,b)$ for every fixed~$\gamma$. Given these conventions, we are now ready to restate Theorem~\ref{thm:general_CSE:informal} formally.

\setcounter{theorem}{0}
\begin{theorem}[Distributionally robust estimator (formal)] \label{thm:general_CSE}
If Assumptions~\ref{assu:inf_sup}--\ref{assu:d_convex} hold, then the distributionally robust estimator $X\opt$ exists and is unique. If, additionally, $\dualvar\opt$ is the unique positive root of the equation
\[
    \sum_{i = 1}^p d(s(\gamma,\xsa_i), \xsa_i) -\eps = 0,
\]
then the distributionally robust estimator admits the spectral decomposition~$X\opt = \Vsa \Diag(x\opt) \Vsa^\top$ with eigenvalues $x\opt_i = s(\gamma\opt, \xsa_i)$, $i = 1,\dots,p$, where $0<x\opt_i<\xsa_i$ whenever $\xsa_i > 0$ and $x\opt_i=0$ whenever $\xsa_i=0$.
\end{theorem}

Theorem~\ref{thm:general_CSE} provides a quasi-closed form expression for the optimal covariance estimator~$X\opt$ that solves the robust estimation problem~\eqref{eq:CSE} as well as its dual reformulation~\eqref{eq:matrix}. In particular, it shows that~$X\opt$ has the same eigenvectors as~$\covsa$ and that all positive eigenvalues of $X\opt$ can be computed by solving a nonlinear equation parametrized by~$\gamma\opt$. Remarkably, this nonlinear equation admits a closed-form solution for all divergences listed in Table~\ref{table:structured_divergence}. In addition, we will see that~$\gamma\opt$ can be computed efficiently by bisection. All of this implies that the complexity of computing~$X\opt$ is largely determined by the complexity of diagonalizing~$\covsa$. In addition, we will see that $x\opt_i = s(\gamma\opt, \xsa_i)$ decreases with~$\gamma\opt$. Thus, $X\opt$ and~$\gamma\opt$ are naturally interpreted as a nonlinear shrinkage estimator and inverse shrinkage intensity, respectively. 

We now outline the high-level structure of the proof of Theorem~\ref{thm:general_CSE}; see Figure~\ref{fig:sketch} for a visualization. The proof is divided into three steps that give rise to three propositions. Proposition~\ref{prop:exist} below first shows that there is a one-to-one relationship between the minimizers of the robust estimation problem~\eqref{eq:CSE} and problem~\eqref{eq:matrix}. 

\begin{proposition}[Dual characterization of~$X\opt$] \label{prop:exist}
If Assumption~\ref{assu:inf_sup} holds, then the primal and dual robust estimation problems~\eqref{eq:CSE} and~\eqref{eq:dual-CSE} are equivalent to problem~\eqref{eq:matrix} in the following sense.
\begin{enumerate}[label=(\roman*)]
    \item\label{prop:exist-i} If $\cov\opt$ solves~\eqref{eq:matrix}, then $X\opt=\cov\opt$ solves~\eqref{eq:CSE}, and $\cov\opt$ solves~\eqref{eq:dual-CSE}.

    \item\label{prop:exist-ii} If $X\opt$ solves~\eqref{eq:CSE} and $\cov\opt$ solves~\eqref{eq:dual-CSE}, then $X\opt$ coincides with $\cov\opt$ and solves~\eqref{eq:matrix}.
\end{enumerate}
\end{proposition}

The proof of Proposition~\ref{prop:exist} follows immediately from the discussion in Section~\ref{sec:cse} and is thus omitted. Next, we show that problem~\eqref{eq:matrix}, which optimizes over all matrices in the positive semidefinite cone~$\PSD^p$, is equivalent to problem~\eqref{eq:vector} below, which optimizes over all vectors in the non-negative orthant~$\R^p_+$:
\begin{equation}\label{eq:vector}\tag{P$_{\text{Vec}}$}
\min_{x \in \R^p_+}  ~\left\{ \norm{ x }_2^2 ~:~ \displaystyle \sum_{ i = 1 }^p d\left( x_i, \xsa_i \right)  \le \eps\right\}.
\end{equation}
We henceforth use~$x\opt$ to denote the unique minimizer of problem~\eqref{eq:vector} if it exists. 

\begin{figure}[!h]
    \centering
    \resizebox{\linewidth}{!}{\input{Fig1}}
    \caption{Structure of the proof of Theorem~\ref{thm:general_CSE}. An arc indicates that the solution to the problem at the arc's tail can be used to construct a solution for the problem at the arc's~head.
    }
    \label{fig:sketch}
\end{figure}

{
\begin{proposition}[Equivalence of~\eqref{eq:matrix} and~\eqref{eq:vector}] \label{prop:vec-equivalence}
If Assumption~\ref{assu:D_form} holds, then problem~\eqref{eq:matrix} is equivalent to problem~\eqref{eq:vector} in the following sense.
\begin{enumerate}[label=(\roman*)]
\item\label{prop:master_1} Problem~\eqref{eq:matrix} is feasible if and only if problem~\eqref{eq:vector} is feasible.    
    \item\label{prop:master_3} If $x\opt$ solves~\eqref{eq:vector}, then $\Vsa \Diag(x\opt) \Vsa^\top$ solves~\eqref{eq:matrix}.
    \item\label{prop:master_4} If $\cov\opt$ solves~\eqref{eq:matrix}, then $\lambda(\cov\opt)$ solves~\eqref{eq:vector}.

\item\label{prop:master_5} \eqref{eq:matrix} and~\eqref{eq:vector} share the same optimal value.
\end{enumerate}
\end{proposition}
}

In the third and last step, we solve problem~\eqref{eq:vector} in quasi-analytical form. To this end, we denote the Lagrange multiplier associated with the divergence constraint $\sum_{ i = 1 }^p d\left( x_i, \xsa_i \right)  \le \eps$ by~$\dualvar\opt$. The following proposition characterizes the unique solution of problem~\eqref{eq:vector} through an explicit function of~$\dualvar\opt$ and shows that~\eqref{eq:vector} can be computed by solving a single nonlinear equation.

\begin{proposition}[Solution of~\eqref{eq:vector}] \label{prop:gamma-reconstruction}
If Assumptions~\ref{assu:D_form}, \ref{assu:data} and~\ref{assu:d_convex} hold, then problem~\eqref{eq:vector} admits a unique optimal solution~$x\opt$ with components $x_i\opt = s(\dualvar\opt, \wh x_i)$, $i = 1, \ldots, p$, where $\gamma\opt$ is the unique positive root of the equation $\sum_{i = 1}^p d(s(\gamma,\xsa_i), \xsa_i) -\eps = 0$. We also have~$0<x\opt_i<\xsa_i$ whenever~$\xsa_i > 0$ and~$x\opt_i=0$ whenever~$\xsa_i=0$.
\end{proposition}

In summary, Proposition~\ref{prop:gamma-reconstruction} provides a simple characterization of~$\gamma\opt$ and shows how one can use~$\gamma\opt$ to construct a unique solution~$x\opt$ for problem~\eqref{eq:vector}. Proposition~\ref{prop:vec-equivalence} reveals how~$x\opt$ can be used to construct a unique solution~$X\opt$ for problem~\eqref{eq:vector}, and Proposition~\ref{prop:exist} guarantees that $X\opt$ is uniquely optimal in the robust estimation problem~\eqref{eq:CSE}. Taken together, Propositions~\ref{prop:exist}, \ref{prop:vec-equivalence} and~\ref{prop:gamma-reconstruction} therefore prove Theorem~\ref{thm:general_CSE}.

\subsection{Properties of the Distributionally Robust Estimator} \label{sec:property}

We now highlight several desirable characteristics of the distributionally robust covariance estimator~$X\opt$.

\subsubsection{Efficient Computation}\label{sec:compute}

We have seen that~$X\opt$ can be constructed from~$x\opt$, which can be constructed from~$\dualvar\opt$. In addition, we have seen that the Lagrange multiplier~$\dualvar\opt$ is the unique positive root of the equation~$F(\dualvar)=0$, where the function~$F: \R_+ \to \overline{\R}$ is defined through $F(\gamma) = \sum_{i = 1}^p d(s(\gamma,\xsa_i), \xsa_i) -\eps$. The following proposition suggests that this root-finding problem can be solved highly efficiently by bisection or Newton's method.

\begin{proposition}[Structural properties of~$F$]
\label{prop:nonlinear_equation_gamma_opt}
    If Assumptions~\ref{assu:D_form}, \ref{assu:data} and~\ref{assu:d_convex} hold, then the function $F$ is differentiable and strictly decreasing on $\R_{++}$. In addition, we have $\lim_{\gamma\to 0} F(\gamma)>0$ and $\lim_{\gamma\to \infty}F(\gamma)<0$.
\end{proposition}

Suppose now that we have access to an {\em a priori} upper bound~$\bar{\gamma}>0$ on the Lagrange multiplier~$\gamma\opt$. Note that~$\bar\gamma$ is guaranteed to exist under the assumptions of the proposition. Section~\ref{sec:new_CSE} shows that~$\bar{\gamma}$ can be constructed explicitly for several popular divergence functions. The structural properties of~$F$ established in Proposition~\ref{prop:nonlinear_equation_gamma_opt} allow us to estimate the number of function evaluations needed to compute~$\gamma\opt$. For example, $\gamma\opt$ can be computed via bisection to within an absolute error of~$\delta>0$ using $\mathcal \log_2(\bar\gamma/\delta)$ function evaluations. Under additional mild conditions, $\gamma\opt$ can also be computed via Newton's method to within an absolute error of~$\delta > 0$ using merely $O(\log_2\log_2 ( \bar{\gamma} / \delta ))$ function and derivative evaluations~\cite[Theorem~2.4.3]{dennis1996numerical}.

\subsubsection{Shrinkage Properties} \label{sec:shrink}

Proposition~\ref{prop:gamma-reconstruction} asserts that if Assumptions~\ref{assu:D_form}, \ref{assu:data} and~\ref{assu:d_convex} hold, then the optimal solution~$x\opt$ of problem~\eqref{eq:vector} is unique and can thus be seen as a function~$x\opt(\eps)$ of the radius~$\eps \in(0, \bar\eps)$ of the uncertainty set, where $\bar\eps$ is defined as in Assumption~\ref{assu:data}\ref{assu:data-eps}. In fact, $x\opt(\eps)$ can naturally be extended to a function on~$[0, \bar\eps]$. As~$d$ satisfies the identity of indiscernibles, we can define~$x\opt(0) = \xsa$ as the unique solution of problem~\eqref{eq:vector} for~$\eps=0$. In addition, we may define $x\opt(\bar{\eps}) = 0$. One can then show that each component of~$x\opt(\eps)$ monotonically decreases to~$0$ on~$[0, \bar{\eps}]$. By Theorem~\ref{thm:general_CSE}, the distributionally robust estimator~$X\opt = \Vsa \Diag(x\opt) \Vsa^\top$ inherits the eigenbasis from the nominal covariance matrix~$\covsa$. Hence, $X\opt$ and $\covsa$ commute, and $X\opt$ is rotation equivariant. In summary, these insights imply that~$X\opt$ essentially shrinks the eigenvalues of~$\covsa$ towards zero.

\begin{proposition}[Shrinkage estimator]
    \label{prop:eigen_decreasing}
    If Assumptions~\ref{assu:D_form}, \ref{assu:data} and~\ref{assu:d_convex} hold, then $x\opt_i (\eps)$ is non-increasing on~$[0, \bar{\eps}]$ and satisfies $\lim_{\eps\uparrow \bar{\eps}} x\opt_i (\eps) = 0$ for every~$i= 1,\dots,p$.
    If additionally Assumption~\ref{assu:inf_sup} holds, then~$X\opt$ constitutes a shrinkage estimator, that is, it has the same eigenvectors as~$\covsa$ and satisfies~$0 \preceq X\opt\preceq \covsa$.
\end{proposition}

Proposition~\ref{prop:eigen_decreasing} asserts that the eigenvalues of~$X\opt$ are bounded above by the corresponding nominal eigenvalues. This shrinkage property persists across a remarkably broad class of estimators. The shrinkage effects of robustification were first discovered in a distributionally robust \textit{inverse} covariance estimation problem with a Wasserstein ambiguity set~\cite{ref:nguyen2018distributionally}. The results presented here are significantly more general. Indeed, they reveal that a broad class of divergence functions gives rise to diverse shrinkage estimators.

\subsubsection{Improvement of the Condition Number}
The condition number~$\kappa(X)$ of a positive definite matrix~$X \in \PD^p$ is defined as the ratio of its largest to its smallest eigenvalue. It is well known that unless~$n\gg p$, the sample covariance matrix~$\covsa_{\mathrm{SA}}$ tends to be ill-conditioned, that is, $\kappa(\covsa_{\mathrm{SA}})\gg 1$ \cite{van1961certain}. Therefore, most shrinkage estimators are designed to improve the condition number of an ill-conditioned baseline estimator~$\covsa$. Below we will show that the distributionally robust estimator~$X\opt$ is also guaranteed to improve the condition number of~$\covsa$ whenever the generator~$d$ of the divergence~$D$ satisfies a second-order differential inequality.

\begin{assumption}[Differential inequality]\label{assu:d_b_second_derivative}
The generator~$d$ of the divergence function~$D$ is twice continuously differentiable on~$\R_{++}^2$ and satisfies the following differential inequality for all $a,b\in\R_{++}$ with~$a<b$.
\begin{equation*}
a \, \frac{\partial^2}{\partial a^2}d(a,b) + b\, \frac{\partial^2}{\partial a\partial b}d(a,b) \ge \frac{\partial}{\partial a}d(a,b)
\end{equation*}
\end{assumption}

Assumption~\ref{assu:d_b_second_derivative} may be difficult to check. In Theorem~\ref{thm:verification} below, we will show, however, that it is satisfied by all divergence functions of Table~\ref{table:structured_divergence}. We can now prove that robustification improves the condition number. 

\begin{proposition}[Improved condition number] \label{prop:condition}
If Assumptions~\ref{assu:inf_sup}--\ref{assu:d_b_second_derivative} hold and $\covsa\in\PD^p$, then $\kappa (X\opt) \le \kappa (\covsa)$.
\end{proposition}

The proof of Proposition~\ref{prop:condition} exploits a generalized monotonicity property of the eigenvalue map $s(\gamma, b)$.

\begin{lemma}[Generalized monotonicity property of the eigenvalue map~$s$]
\label{lemma:cond_decrease}
If Assumptions~\ref{assu:D_form}, \ref{assu:d_convex} and~\ref{assu:d_b_second_derivative} hold, then we have $s(\gamma, b_2)/s(\gamma, b_1)  \le b_2/b_1$ for all $\gamma > 0$ and $b_1,b_2\in \R_{++}$ with $b_2 \ge b_1$.
\end{lemma}

Recall from Theorem~\ref{thm:general_CSE} that $x\opt_i = s(\gamma\opt, \xsa_i)$ for all $i = 1,\dots, p$ and that $\gamma\opt > 0$. Therefore, Proposition~\ref{prop:condition} follows immediately from Lemma~\ref{lemma:cond_decrease}.

\subsubsection{Statistical Guarantees}
We finally show that the distributionally robust estimator is consistent and enjoys a finite-sample performance guarantee. To this end, we make the dependence on~$n$ explicit, that is, we let~$X\opt_n$ be the unique solution of~\eqref{eq:CSE}, where the nominal estimator is any covariance estimator~$\covsa_n$ constructed from~$n$ i.i.d.~training samples, and where the radius is set to a non-negative number~$\eps_n$ that may depend on~$n\in\mathbb N$. We say a covariance estimator is strongly consistent if it converges almost surely to~$\cov_0$ { for a fixed~$p$ as}~$n$ tends to~infinity.

\begin{proposition}[Consistency] \label{prop:consistency}
        Suppose that Assumptions~\ref{assu:inf_sup}--\ref{assu:d_convex} hold and that~$d$ is continuous on $\mathbb{R}_+\times \mathbb{R}_{++}$. If $\covsa_n$ is a strongly consistent estimator and~$\eps_n$ converges to~$0$ as~$n$ grows, then~$X\opt_n$ is strongly consistent. 
\end{proposition}

Proposition~\ref{prop:consistency} is intuitive because the uncertainty set is assumed to shrink with~$n$, and the nominal covariance matrix at its center is assumed to be consistent. As the uncertainty set is defined as a generic divergence ball, however, the proof is perhaps surprisingly tedious. The standard example of a consistent nominal covariance estimator~$\covsa_n$ is the sample covariance matrix. Note that Proposition~\ref{prop:consistency} analyzes the asymptotics of~$X\opt_n$ as~$n$ tends to infinity for a fixed~$p$, which is referred to as the \emph{low-dimensional regime} in statistics.

Next, we establish finite-sample performance guarantees, that is, we show that the uncertainty set of radius~$\eps_n\propto n^{-\half}$ around the sample covariance matrix constitutes a confidence region for~$\cov_0$. In the following we say that the probability distribution $\trueP$ is sub-Gaussian if there exists { a variance proxy} $\sigma^2\geq 0$ with $\mathbb E_{\trueP}[\exp(z^\top \xi)]\leq \exp(\half\sigma^2\|z\|_2^2)$ for every~$z\in\R^p$. As both sides of this inequality are differentiable and coincide at~$z=0$, one can show that any sub-Gaussian distribution~$\trueP$ must have mean~$0$.

\begin{proposition}[Finite-sample performance guarantee]\label{proposition:finite-sample-guarantees}
    Suppose that~$\mathbb{P}$ is sub-Gaussian with covariance matrix~$\cov_0\in\Sym^p_{++}$, and let $\covsa_n$ be the sample covariance matrix corresponding to $n$ i.i.d.\ samples from~$\trueP$. For any divergence function~$D$ from Table~\ref{table:structured_divergence} there exist {functions $n_{\min}(p, \eta) = \mathcal{O}(p+\log \eta^{-1})$ and $\eps_{\min} (p,n,\eta) = \mathcal{O}(p n^{-\half} (p+ \log\eta^{-1})^{\half})$, which may depend on~$\p$ only through the variance proxy $\sigma^2$ and the smallest eigenvalue~$\lambda_1(\cov_0)$ of~$\cov_0$, such that $\trueP^n[\cov_0\in \B_\eps (\covsa_n)]\geq 1-\eta$ for every $n\geq n_{\min}(p, \eta)$ and $\eps\geq \eps_{\min} (p,n,\eta)$}.
\end{proposition}

Proposition~\ref{proposition:finite-sample-guarantees} implies that if {$n\geq n_{\min}(p, \eta)$ and $\eps\geq \eps_{\min} (p,n,\eta)$}, then the optimal value of the robust covariance estimation problem~\eqref{eq:CSE} provides a $(1-\eta)$-upper confidence bound on the actual estimation error with respect to the true covariance matrix~$\cov_0$. Explicit formulas for {$ n_{\min}(p, \eta)$ and $ \eps_{\min} (p,n,\eta)$} tailored to different divergence functions can be found in the proof of Proposition~\ref{proposition:finite-sample-guarantees} in the appendix. { The finite-sample guarantee of Proposition~\ref{proposition:finite-sample-guarantees} directly yields an asymptotic guarantee in a high-dimensional regime where~$p$ grows with~$n$. Specifically, it implies that the population covariance~$\cov_0$ remains within the uncertainty set~$\B_\eps(\covsa_n)$ with constant confidence~$1 - \eta$ as the dimension~$p$ scales like~$n^{1/3}$. This stands in contrast to standard high-dimensional performance guarantees, which permit the dimension to grow linearly with~$n$.}


\section{A Zoo of New Covariance Shrinkage Estimators}\label{sec:new_CSE}

In this section, we first show that the assumptions of Theorem~\ref{thm:general_CSE} are satisfied by a broad spectrum of divergence functions commonly used in statistics, information theory, and machine learning. Next, we explicitly construct the shrinkage estimators corresponding to three popular divergence functions.

\begin{theorem}[Validation of assumptions] \label{thm:verification}
All divergences in Table~\ref{table:structured_divergence} satisfy Assumptions~\ref{assu:inf_sup}, \ref{assu:D_form}, \ref{assu:d_convex} and~\ref{assu:d_b_second_derivative}.
\end{theorem}

We emphasize that the uncertainty sets corresponding to the Fisher-Rao and inverse Stein divergences fail to be convex, in which case one cannot use standard minimax results to prove Assumption~\ref{assu:inf_sup}. However, perhaps surprisingly, in Appendix~\ref{subsec:RG_GC}, we show that the uncertainty sets corresponding to these divergences are geodesically convex with respect to a particular Riemannian geometry on the space of positive definite matrices. This in turn allows us to apply a Riemannian minimax theorem (see Theorem~\ref{thm:Sion-manifold} in Appendix~\ref{sec:riemannian_sion}) and prove the desired minimax property even for robust estimation problems based on the Fisher-Rao and inverse Stein divergences.

To showcase the richness of our framework, we now focus on three popular divergence functions and analyze the corresponding robust covariance estimators. Specifically, we will derive the optimal solutions of problem~\eqref{eq:vector} in quasi-closed form for the Kullback-Leibler, Wasserstein, and Fisher-Rao divergences. In doing so, we develop a general recipe for the other divergence functions listed in Table~\ref{table:structured_divergence}.

\subsection{The Kullback-Leibler Covariance Shrinkage Estimator}

Table~\ref{table:structured_divergence} defines the Kullback-Leibler (KL) divergence between two matrices $\cov_1,\cov_2\in\PD^p$ as
\[ 
D_{\rm KL}(\cov_1, \cov_2) = \frac{1}{2}\left( \Tr{\cov_2^{-1} \cov_1 } - p + \log \det (\cov_2 \cov_1^{-1}) \right) . 
\]
This KL divergence between matrices is intimately related to the KL divergence between distributions.

\begin{definition}[KL divergence] \label{def:KL}
    If $\trueP_1$ and $\trueP_2$ are two probability distributions on $\R^p$, and $\trueP_1$ is absolutely continuous with respect to $\trueP_2$, then the KL divergence from $\trueP_1$ to $\trueP_2$ is $\mathrm{KL}(\trueP_1 \| \trueP_2) = \int_{\R^p}\log (\frac{\mathrm{d}\trueP_1}{\mathrm{d}\trueP_2}(x) )\mathrm{d}\trueP_2(x)$.
\end{definition}

The following lemma shows that the KL divergence between two non-degenerate zero-mean Gaussian distributions coincides with the KL divergence between their positive definite covariance matrices. 

\begin{lemma}[KL divergence between Gaussian distributions~\cite{kullback1997information}] \label{lem:KL-Gauss}
    The KL divergence from $\trueP_1=\mc N(0, \Sigma_1)$ to $\trueP_2=\mc N(0, \Sigma_2)$ with $\Sigma_1,\Sigma_2\in\PD^p$ is given by $\mathrm{KL}(\trueP_1 \| \trueP_2) =D_{\rm KL}(\cov_1,\cov_2)$.
\end{lemma}

Lemma~\ref{lem:KL-Gauss} justifies our terminology of referring to~$D_{\rm KL}$ as the KL divergence and suggests that~$D_{\rm KL}$ inherits many properties of the KL divergence between distributions. For example, it is easy to verify that~$D_{\rm KL}$ satisfies the identity of indiscernibles but fails to be symmetric. Indeed, for any $\Sigma \in \PD^p$ we have $D_{\rm KL}(\Sigma,2\Sigma)=\frac{p}{2}(1-\log(2))\approx0.15p$, whereas $D_{\rm KL}(2\Sigma,\Sigma)=\frac{p}{2}(\log(2)-\frac{1}{2})\approx0.1p$. An elementary calculation further reveals that the generator~$d$ corresponding to the KL divergence~$D_{\rm KL}$ can be expressed as
\[
    d(a,b) = \frac{1}{2}\left( \frac{a}{b} - 1 - \log\left( \frac{a}{b} \right) \right).
\]
The following corollary of Theorem~\ref{thm:general_CSE} characterizes the eigenvalue map and the inverse shrinkage intensity corresponding to the KL divergence, which determines the KL covariance shrinkage estimator.

\begin{corollary}[KL covariance shrinkage estimator]\label{thm:KL}
If $D$ is the KL divergence, $\covsa\in\PD^p$ and $\varepsilon > 0$, then problem~\eqref{eq:CSE} is uniquely solved by the KL covariance shrinkage estimator $X \opt= \Vsa \Diag (x\opt) \Vsa^\top$ with shrunk eigenvalues $x\opt_i = s(\gamma\opt, \xsa_i)$, $i = 1,\dots,p$. The underlying eigenvalue map is given by
\begin{subequations}
\be
    s(\gamma,b)= 
    \frac{-\gamma+ \sqrt{\gamma^2 + 16 b^2 \gamma}}{8b},  
     \label{eq:KL-xiopt}
\ee
and the inverse shrinkage intensity $\dualvar\opt\in (0,\dualvar_{\KL}] $ is the unique positive solution of the nonlinear equation
\be \label{eq:KL-FOC}
2\varepsilon +p + \sum_{i=1}^p \left[ - \frac{s(\gamma\opt, \xsa_i)}{\xsa_i} +\log \frac{s(\gamma\opt, \xsa_i)}{\xsa_i} \right]=0,
\ee
where
\[
    \dualvar_{\KL} = \frac{4\,\xsa_p^2\exp(-4\varepsilon/p)}{1-\exp(-2\varepsilon/p)} > 0.
\]
\end{subequations}
\end{corollary}

\subsection{The Wasserstein Covariance Shrinkage Estimator}

Table~\ref{table:structured_divergence} defines the Wasserstein divergence between two matrices $\cov_1,\cov_2\in\PSD^p$ as
\[ 
    D_{\rm W}(\cov_1, \cov_2) = \Tr{\cov_1 + \cov_2 - 2 \big( \cov_1 \cov_2 \big)^\half } .
\]
In the following, we will show that the Wasserstein distance between matrices is closely related to the squared 2-Wasserstein distance between distributions, where the transportation cost is defined via the 2-norm.
\begin{definition}[Wasserstein distance]
The 2-Wasserstein distance between two probability distributions $\mathbb{P}_1$ and $\mathbb{P}_2$ on $\R^p$ with finite second moments is defined as
\[
W_2(\mathbb{P}_1, \mathbb{P}_2) =\left(\inf_{\pi\in\Pi(\mathbb{P}_1,\mathbb{P}_2)}\int_{\R^p\times\R^p}\|x_1-x_2\|_2^2 \, \mathrm{d}\pi(x_1,x_2) \right)^\half,
\]
where $\Pi(\mathbb{P}_1,\mathbb{P}_2)$ denotes the set of probability distributions on $\R^p\times\R^p$ with marginals $\mathbb{P}_1$ and $\mathbb{P}_2$, respectively.
\end{definition}

One can show that Wasserstein distance $W_2$ is a metric on the space of probability distributions with finite second moments \cite[\S~6]{ref:villani2008optimal}. However, the {\em squared} Wasserstein distance $W_2^2$ is only a divergence as it fails to satisfy the triangle inequality. The following lemma shows that the squared 2-Wasserstein distance between two zero-mean Gaussian distributions matches the Wasserstein divergence between their covariance matrices. 

\begin{lemma}[Squared Wasserstein distance between Gaussian distributions~\cite{givens1984class}]~\label{lem:WS-Gauss}
    The squared 2-Wasserstein distance between $\mathbb{P}_1=\mc N(0, \Sigma_1)$ and $\mathbb{P}_2=\mc N(0, \Sigma_2)$ evaluates to
        $W_2(\mathbb{P}_1, \mathbb{P}_2)^2 = D_{\rm W}(\cov_1,\cov_2)$.
\end{lemma}
Lemma~\ref{lem:WS-Gauss} justifies our terminology of referring to $D_{\rm W}$ as the Wasserstein divergence and suggests that $D_{\rm W}$ inherits many properties from the Wasserstein distance between distributions. Note that~$D_{\rm W}$ remains well-defined even if $\cov_1$ or $\cov_2$ are rank-deficient. The generator $d$ of the Wasserstein divergence $D_{\rm W}$ is given~by
\[
d(a,b)=a + b - 2 \sqrt{ab}.
\]
The following corollary of Theorem~\ref{thm:general_CSE} characterizes the eigenvalue map and inverse shrinkage intensity corresponding to the Wasserstein divergence, which determines the Wasserstein covariance shrinkage estimator.

\begin{corollary}[Wasserstein covariance shrinkage estimator]\label{thm:Wass}
If $D$ is the Wasserstein divergence, $\covsa\in\PSD^p$ and  $\varepsilon\in (0, \Tr{\covsa})$, then problem~\eqref{eq:CSE} is uniquely solved by the Wasserstein covariance shrinkage estimator $X \opt= \Vsa \Diag (x\opt) \Vsa^\top$ with eigenvalues $x\opt_i  = s(\gamma\opt, \xsa_i)$, $i=1,\dots,p$. The underlying eigenvalue map is given~by 
\begin{subequations}
		\be \label{eq:Wass-xiopt}
		s(\gamma, b)=
			\left( \left\{\frac{\dualvar}{4}\left(\sqrt{b} + \sqrt{b + \frac{2}{27} \dualvar} \right)\right\}^{\frac{1}{3}} - \frac{\dualvar }{6} \left\{\frac{\dualvar}{4}\left(\sqrt{b} + \sqrt{b + \frac{2}{27} \dualvar} \right)\right\}^{-\frac{1}{3}}\right)^2 
		\ee
		and the inverse shrinkage intensity $\dualvar \opt\in (0,\dualvar_{\W}]$ is the unique positive solution of the nonlinear equation
		\be \label{eq:Wass-FOC}
		    \eps - \sum_{i=1}^p \left(\sqrt{\xsa_i} - \sqrt{s(\gamma\opt, \xsa_i)} \right)^2 = 0 ,
		\ee
	\end{subequations}
	where $\dualvar_{\W} = 2\sqrt{p\,\xsa_p^3/\eps}  > 0$.
\end{corollary}

The requirement that $\varepsilon$ be strictly smaller than~$\Tr{\covsa}$ is equivalent to Assumption~\ref{assu:data}\ref{assu:data-eps}. It is needed to prevent problem~\eqref{eq:vector} from admitting the trivial solution~$x\opt = 0$. To see this, note that the condition $\eps \geq \Tr{\covsa}$ is equivalent to $\sum_{i=1}^p d(0,\xsa_i)\le \eps$, which in turn implies that~$0$ is feasible and even optimal in~\eqref{eq:vector}. In this case, the trivial (and essentially nonsensical) estimator~$X\opt = 0$ would be optimal in problem~\eqref{eq:CSE}.

\subsection{The Fisher-Rao Covariance Shrinkage Estimator}
Table~\ref{table:structured_divergence} defines the Fisher-Rao divergence between two matrices $\cov_1,\cov_2\in\PD^p$ as
\[ 
    D_{\rm FR}(\cov_1,\cov_2)= \| \log (\cov_2^{-\frac{1}{2}} \cov_1 \cov_2^{-\frac{1}{2}}) \|_{\mathrm{F}}^2 . 
\]
The Fisher-Rao divergence can be interpreted as the Fisher-Rao distance on a particular statistical manifold. 

\begin{definition}[Fisher-Rao distance]
Consider a family of probability density functions~$\{f_\theta(\xi)\}_{\theta\in \Theta}$ whose parameter~$\theta$ ranges over a Riemannian manifold~$\Theta$ with metric $$I_\theta=\int_\Xi f_\theta(\xi)\,\nabla_\theta\log(f_\theta(\xi))\nabla_\theta\log(f_\theta(\xi))^\top\mathrm{d}\xi.$$ The geodesic distance~$\mathrm{FR}(\theta_1, \theta_2)$ on~$\Theta$ induced by this metric is referred to as the Fisher-Rao distance.
\end{definition}

Note that~$I_\theta$ represents the Fisher information matrix corresponding to the parameter~$\theta$. Next, we show that the squared Fisher-Rao distance between two non-degenerate zero-mean Gaussian probability density functions is proportional to the Fisher-Rao divergence between their positive definite covariance matrices.

\begin{lemma}[Fisher-Rao distance between positive definite covariance matrices~\cite{atkinson1981rao}] \label{lem:FR-Gauss}
    Let~$\{f_\theta(\xi)\}_{\theta\in \Theta}$ be the family of all non-degenerate zero-mean Gaussian probability density functions encoded by their covariance matrices~$\theta=\cov$, which range over the Riemannian manifold~$\Theta=\PD^p$ equipped with the Fisher-Rao distance. If $\theta_1=\cov_1$ and $\theta_2=\cov_2$ belong to~$\PD^p$, then~$\mathrm{FR}(\theta_1, \theta_2)^2 = \half D_{\rm FR}(\cov_1,\cov_2)$.
\end{lemma}

Lemma~\ref{lem:FR-Gauss} justifies our terminology of referring to $D_{\rm FR}$ as the Fisher-Rao divergence. As~$D_{\rm FR}$ is proportional to the {\em squared} Fisher-Rao distance ${\rm FR}^2$, it fails to satisfy the triangle inequality and is indeed only a divergence. Moreover, Example~\ref{example:quasi-convex} in Appendix~\ref{subsec:non-quasi-convex} reveals that $D_{\rm FR}$ is neither convex nor quasi-convex. However, it is geodesically convex. The generator~$d$ corresponding to~$D_{\rm FR}$ can be expressed as
\[
d(a,b)=(\log(a/b))^2.
\]
The following corollary of Theorem~\ref{thm:general_CSE} characterizes the eigenvalue map and inverse shrinkage intensity corresponding to the Fisher-Rao divergence, which characterizes the Fisher-Rao covariance estimator.

\begin{corollary}[Fisher-Rao covariance shrinkage estimator]\label{thm:FR}
If $D$ is the Fisher-Rao divergence, $\covsa\in\PD^p$ and $\varepsilon > 0$, then problem~\eqref{eq:CSE} is uniquely solved by the Fisher-Rao covariance shrinkage estimator $X \opt= \Vsa \Diag (x\opt) \Vsa^\top$ with eigenvalues $x\opt_i  = s(\gamma\opt, \xsa_i)$, $i=1,\dots,p$. 
The underlying eigenvalue map is given by 
\begin{subequations}
\be
    s(\gamma, b) = 
    b \exp\left( - \frac{1}{2} W_0 \left( 2b^2/\gamma \right) \right), \label{eq:FR-xiopt}
\ee
and $W_0$ denotes the principal branch of the Lambert-W function. In addition, the inverse shrinkage intensity $\dualvar\opt \in (0,\dualvar_{\FR}]$ with $\dualvar_{\FR} =\|\covsa\|_{\mathrm{F}}^2/\sqrt{\eps} > 0$ is the unique positive solution of the nonlinear equation
\begin{equation}\label{eq:FR_opt_gamma}
        \sum_{i = 1}^p W_0^2 \left(2\,\xsa_i^2 /\gamma\right) = 4\eps.
\end{equation}
\end{subequations}
\end{corollary}

\subsection{Other Covariance Shrinkage Estimators}

Theorem~\ref{thm:verification} ensures that all divergence functions from Table~\ref{table:structured_divergence} satisfy Assumptions~\ref{assu:inf_sup}, \ref{assu:D_form}, \ref{assu:d_convex} and~\ref{assu:d_b_second_derivative} and thus induce via Theorem~\ref{thm:general_CSE} a distributionally robust covariance shrinkage estimator. The generators and eigenvalue maps corresponding to all these divergences can be derived by using similar techniques as in Corollaries~\ref{thm:KL}, \ref{thm:Wass}, and~\ref{thm:FR}. Details are omitted for brevity. All generators and eigenvalue maps are provided in Table~\ref{table:a}.
\bgroup
\def\arraystretch{1.8}
\begin{table}[H]
\resizebox{\linewidth}{!}{\begin{tabular}{|l||c|c|c|}
\hline
Divergence & $d(a, b)$ & $\text{dom}(d)$ & $s(\gamma, b)$ for $b> 0$ \\ \hline\hline
Kullback-Leibler/Stein & $\frac{1}{2}\left( \frac{a}{b} - 1 - \log \frac{a}{b}  \right)$ & $\R_{++}\times\R_{++}$ & $\frac{-\gamma+ \sqrt{\gamma^2 + 16 b^2 \gamma}}{8b}$ \\ \hline
Wasserstein & $a + b - 2 \sqrt{ab}$ & $\R_{+}\times\R_{+}$ & $\left(\left( \frac{\dualvar}{4}\left(\sqrt{b} + \sqrt{b + \frac{2}{27} \dualvar} \right)\right)^{\frac{1}{3}} - \frac{\dualvar }{6} \left(\frac{\dualvar}{4}\left(\sqrt{b} + \sqrt{b + \frac{2}{27} \dualvar} \right)\right)^{-\frac{1}{3}}\right)^2$ \\ \hline
Fisher-Rao & $( \log \frac{a}{b} )^2$ & $\R_{++}\times\R_{++}$ & $b \exp( - \frac{1}{2} W_0 ( 2b^2/\gamma ) )$ \\ \hline
Inverse Stein & $\frac{1}{2} \left( \frac{b}{a} - 1 - \log\frac{b}{a} \right)$ & $\R_{++}\times\R_{++}$ & $\frac{3^{1/3}\big(\sqrt{3\gamma^2(27b^2+\gamma)}+9\gamma b\big)^{2/3}-3^{2/3}\gamma}{6\big(\sqrt{3\gamma^2(27b^2+\gamma)}+9\gamma b\big)^{1/3}}$ \\ \hline
\hspace{-1.5mm}$\begin{array}{l}
     \text{Symmetrized Stein/}\\[-2.5ex]
     \text{Jeffreys divergence}
\end{array}$ & $\frac{1}{2} \left( \frac{b}{a} + \frac{a}{b} - 2\right)$ & $\R_{++}\times\R_{++}$ & $\begin{aligned}{\textstyle\frac{1}{12}\Big(}&{\textstyle\frac{\gamma^2}{b\big(216\gamma b^4+12\sqrt{3(108(\gamma^2 b^8-3(\gamma b)^4)}-\gamma^3\big)}}\\&{\textstyle+\frac{\big(216\gamma b^4+12\sqrt{3(108(\gamma^2 b^8-3(\gamma b)^4)}-\gamma^3\big)}{b}-\frac{\gamma}{b}\Big)}\end{aligned}$ \\ \hline
\hspace{-1.5mm}$\begin{array}{l}
     \text{Quadratic/}\\[-2.5ex]
     \text{Squared Frobenius}
\end{array}$ & $(a - b)^2$ & $\R_{+}\times\R_{+}$ & $\frac{b}{\gamma+b}$ \\ \hline
Weighted quadratic & $\frac{(a - b)^2}{b}$ & $\R_{+}\times\R_{++}$ & $\frac{\gamma b}{\gamma+b}$ \\ \hline
\end{tabular}}
\caption{Generators and eigenvalue maps of the divergences from Table~\ref{table:structured_divergence}.}
\label{table:a}
\end{table}
\egroup

\section{Numerical Experiments} \label{sec:numerical}

We now compare our distributionally robust covariance estimators against the linear shrinkage estimator with shrinkage target  $\frac{1}{n} \Tr{\covsa} I_n$~\cite{ref:ledoi2004well} as well as a state-of-the-art nonlinear shrinkage estimator proposed by Ledoit and Wolf~\cite{ledoit2020analytical}, henceforth referred to as the \textit{NLLW} estimator. The performance of the linear shrinkage estimator depends on the choice of the mixing parameter $\alpha\in[0,1]$, which we calibrate via cross-validation.

We first study the dependence of our estimators on the radius~$\eps$ of the uncertainty set{, and we numerically validate the asymptotic consistency and finite-sample guarantees of Propositions~\ref{prop:consistency} and~\ref{proposition:finite-sample-guarantees}, respectively. Using synthetic data, we then assess the Frobenius risk of our estimators} as a function of the sample size. Using real data, we further test the performance of minimum variance portfolios constructed from our estimators. In addition, we illustrate the use of covariance estimators in the context of linear and quadratic discriminant analysis.
The code for all experiments as well as an implementation of our methods can be found on GitHub.\footnote{\url{https://github.com/yvesrychener/covariance_DRO}}

\subsection{Dependence on the Radius of the Uncertainty Set}

We first study the decay of the eigenvalues and the condition number of the Kullback-Leibler, Wasserstein, and Fisher-Rao covariance shrinkage estimators with the radius~$\eps$ of the uncertainty set. To this end, we set~$p=3$ and consider a nominal covariance matrix with eigenvalue spectrum $\lambda(\covsa)=[1,2,3]$. Figure~\ref{fig:eig_path} visualizes the eigenvalues of~$X\opt$ as a function of~$\eps$. In agreement with Proposition~\ref{prop:eigen_decreasing}, we observe that~$X\opt$ shrinks the eigenvalues of the underlying nominal estimator~$\covsa$ towards~$0$ as~$\eps$ grows. Recall from Assumption~\ref{assu:data}\ref{assu:data-eps} and the subsequent discussion that~$X\opt=0$ whenever~$\eps\geq\sum_{i=1}^p d(0, \xsa_i)$. As the generator of the Wasserstein divergence satisfies~$d(0,b)=b$, the eigenvalues of the Wasserstein covariance shrinkage estimator thus vanish for any~$\eps \geq \Tr{\covsa}$. In contrast, the eigenvalues of the Kullback-Leibler and Fisher-Rao covariance shrinkage estimators remain strictly positive for all~$\eps$. We further observe that, for small values of~$\varepsilon$, the Wasserstein and Fisher-Rao covariance shrinkage estimators primarily shrink the large eigenvalues of~$\covsa$ and keep the small ones constant. Figure~\ref{fig:CN_path} visualizes the condition number~$\kappa(X\opt)$ as a function of~$\eps$. As predicted by Proposition~\ref{prop:condition}, $\kappa(X\opt)$ is at most as large as~$\kappa(\covsa)$.  Note also that~$\kappa(X\opt)$ is undefined for $\eps\geq\sum_{i=1}^p d(0, \xsa_i)$. Figure~\ref{fig:CN_path} indicates that the condition number of~$X\opt$ decreases monotonically as~$\eps$ tends to~$\sum_{i=1}^p d(0, \xsa_i)$. \begin{figure}
    \centering
    \subfigure[Kullback-Leibler]{\includegraphics[width=0.3\textwidth, trim= 0cm 0cm 13.2cm 1.cm, clip]{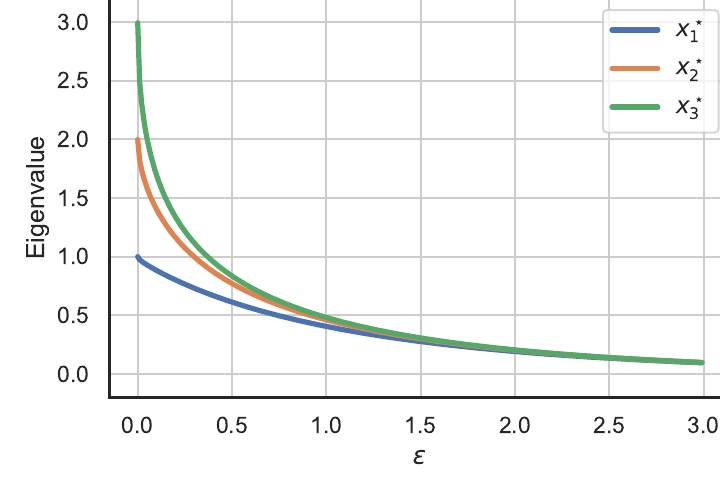}}
    \subfigure[Wasserstein]{\includegraphics[width=0.3\textwidth, trim= 0cm 0cm 13.2cm 1.cm, clip]{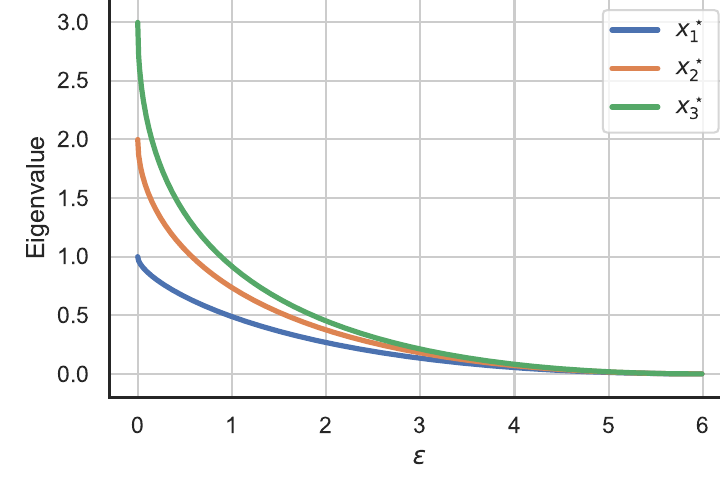}}
    \subfigure[Fisher-Rao]{\includegraphics[width=0.3\textwidth, trim= 0cm 0cm 13.2cm 1.cm, clip]{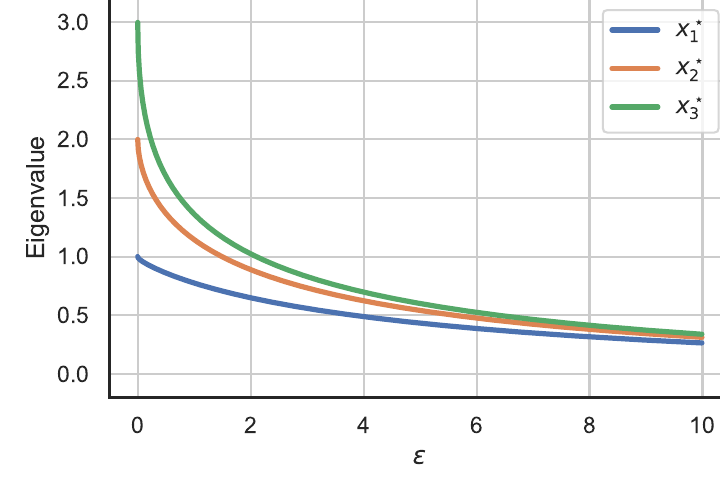}}
    \caption{Eigenvalues of three different distributionally robust covariance estimators as a function of the radius~$\eps$ for $\lambda(\covsa)=[1,2,3]$.}
    \label{fig:eig_path}
\end{figure}
\begin{figure}
    \centering
    \subfigure[Kullback-Leibler]{\includegraphics[width=0.3\textwidth, trim= 12.5cm 0cm 0.5cm 0.8cm, clip]{plots/sensitivity/kl_123.pdf}}
    \subfigure[Wasserstein]{\includegraphics[width=0.3\textwidth, trim=  12.5cm 0cm 0.5cm 0.8cm, clip]{plots/sensitivity/ws_123.pdf}}
    \subfigure[Fisher-Rao]{\includegraphics[width=0.3\textwidth, trim= 12.5cm 0cm 0.5cm 0.8cm, clip]{plots/sensitivity/fr_123.pdf}}
    \caption{Condition number of three different distributionally robust covariance estimators as a function of the radius~$\eps$ for $\lambda(\covsa)=[1,2,3]$.}
    \label{fig:CN_path}
\end{figure}

{
\subsection{Consistency and Finite-Sample Performance}
To validate both the asymptotic consistency and the finite-sample guarantees established in Propositions~\ref{prop:consistency} and~\ref{proposition:finite-sample-guarantees}, we examine the behavior of the estimation error as~$n$ tends to infinity both in the low-dimensional regime with fixed~$p$ and the high-dimensional regime where the ratio~$p/n$ remains constant. In both cases, we evaluate our estimators under two scenarios: (i) when the true covariance matrix is~$\cov_0 = I_{p}$, and (ii) when~$\cov_0$ is a banded $p \times p$ matrix with ones on the diagonal and $0.5$ on the immediate off-diagonals above and below.

\subsubsection{Consistency (Low-Dimensional Regime)}
\label{sec:consistency}
Assume that~$p = 10$ is fixed and that~$\covsa_n$ is the sample covariance matrix constructed from~$n$ independent samples drawn from the distribution $\mathbb{P} = \mathcal{N}(0, \cov_0)$. According to Proposition~\ref{proposition:finite-sample-guarantees}, finite-sample guarantees require uncertainty set radii of order $\mathcal{O}(n^{-1/2})$. This motivates us to set~$\varepsilon_n = 5 n^{-1/2}$. Proposition~\ref{prop:consistency} asserts that the distributionally robust estimator~$X^\star_n$ converges almost surely to~$\cov_0$ as~$n$ tends to infinity, given that the sample covariance matrix is consistent and $\varepsilon_n$ tends to zero. To empirically verify this result, we plot the Frobenius distance between~$X^\star_n$ and~$\cov_0$ as a function of~$n$. Figure~\ref{fig:consistency:big-data} displays the mean Frobenius losses (solid lines) along with one-standard-deviation bands (shaded regions), computed over~$10$ independent datasets of size~$n$. The results reveal that the Frobenius errors of the Wasserstein, Kullback-Leibler and Fisher-Rao estimators all approximate straight lines with negative slopes on a log-log scale, indicating polynomial decay in~$n$. This observed behavior is consistent with the theoretical convergence guarantee of Proposition~\ref{prop:consistency}. However, the empirical covariance matrix converges faster than all tested distributionally robust estimators.

\begin{figure}
    \centering
    {\subfigure[True covariance $\cov_0$ is the identity matrix]{\includegraphics[width=0.45\textwidth]{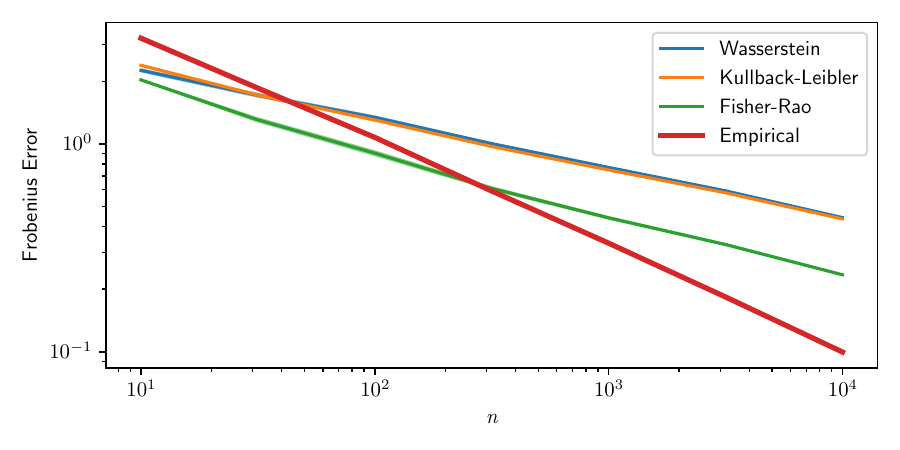}}}
    {\subfigure[True covariance $\cov_0$ is band-diagonal]{\includegraphics[width=0.45\textwidth]{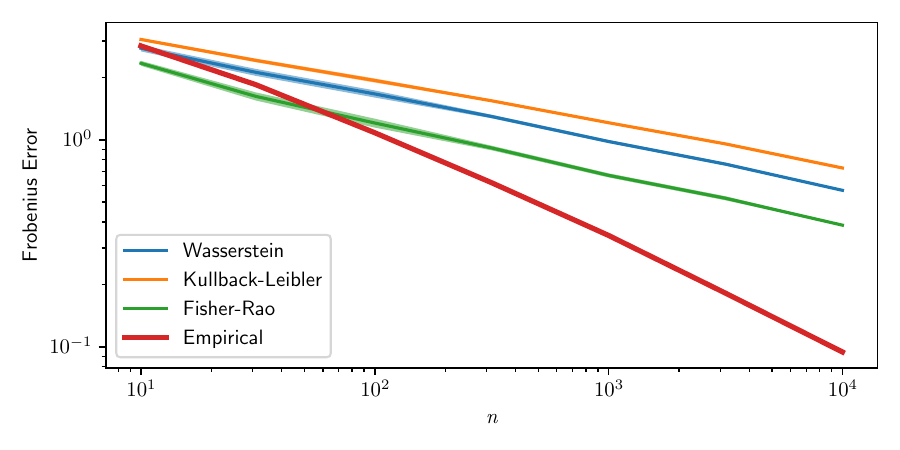}}}
    {\caption{Consistency of $X\opt_n$ and $\covsa_n$ in the low-dimensional regime when $p$ is fixed.}
    \label{fig:consistency:big-data}}
\end{figure}

\subsubsection{Finite-Sample Performance (High-Dimensional Regime)}
We adopt the same experimental setup as in Section~\ref{sec:consistency} but now focus on a high-dimensional regime where the dimension~$p_n = 0.8n$ grows linearly with $n$. Proposition~\ref{proposition:finite-sample-guarantees} states that if $\varepsilon_n = \mathcal{O}(p_n^{3/2} n^{-1/2}) = \mathcal{O}(n)$ and~$n$ is sufficiently large, then the true covariance matrix~$\cov_0$ lies within the uncertainty set $\mathcal{B}_{\varepsilon_n}(\covsa_n)$ with constant confidence. By construction, the distributionally robust estimator~$X^\star_n$, which essentially minimizes the worst-case Frobenius error over all $\cov \in \mathcal{B}_{\varepsilon_n}(\covsa_n)$, is expected to exhibit a small Frobenius error. We now empirically investigate this hypothesis. Specifically, for each~$n$ we use a ternary search algorithm to determine a radius~$\widehat{\varepsilon}_n$ such that the corresponding distributionally robust estimator~$X^\star_n$ minimizes the average Frobenius distance to~$\cov_0$ over~$10$ independent datasets of size~$n$. Figure~\ref{fig:consistency:small-data-optimal radius} shows the empirically optimal radius~$\widehat{\varepsilon}_n$ as a function of~$n$ for the banded covariance matrix~$\cov_0$ (the results are qualitatively similar when~$\cov_0$ is the identity matrix). We observe that~$\widehat{\varepsilon}_n$ grows approximately linearly with~$n$, consistent with the theoretical scaling of~$\varepsilon_n$ from Proposition~\ref{proposition:finite-sample-guarantees} when~$p_n = 0.8n$. This observation suggests that, in the standard high-dimensional regime where $p_n = \mathcal{O}(n)$, one may simply approximate the optimal radius by $c\, p_n^{3/2} n^{-1/2}$ and select the tuning constant~$c>0$ via cross-validation, for example. Figure~\ref{fig:consistency:small-data-optimized} plots the normalized Frobenius loss $\|X^\star_n - \cov_0\|_F / \|\cov_0\|_F$ as a function of~$n$ for the distributionally robust estimator corresponding to~$\widehat{\varepsilon}_n$. The normalization by~$\|\cov_0\|_F$ accounts for increasing dimension, allowing for meaningful comparison across different values of~$n$. We find that the Wasserstein, Kullback-Leibler and Fisher-Rao estimators all achieve significantly smaller relative Frobenius error than the empirical covariance matrix across all values of~$n$. This suggests that robustification is beneficial in high-dimensional regimes where $p_n = \mathcal{O}(n)$, even though the relative Frobenius loss may not decrease with~$n$.

\begin{figure}
    \centering
    \includegraphics[trim={0 0 0 1cm},clip,width=0.9\linewidth]{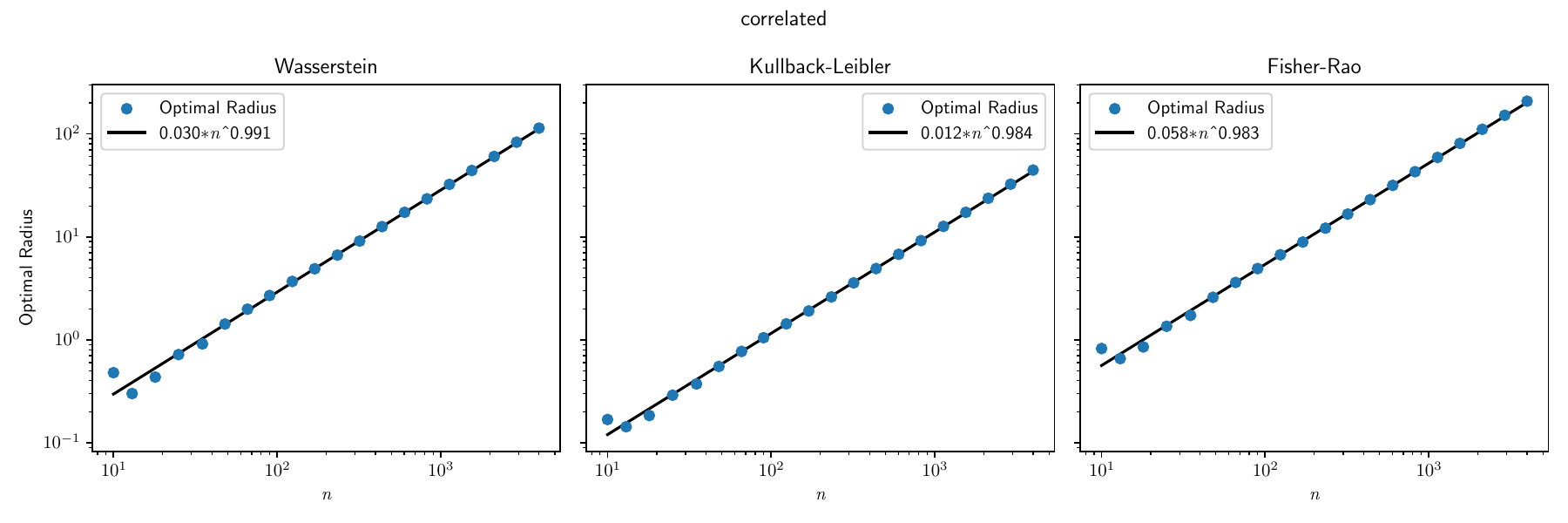}
    {\caption{Optimal radius in the high-dimensional regime with a least-squares fit in log-log space. The plot shows $\widehat \varepsilon_n$ as a function of $n$, along with a fitted curve of the form $c n^\alpha$.}
    \label{fig:consistency:small-data-optimal radius}}
\end{figure}

\begin{figure}
    \centering
    {\subfigure[True covariance $\cov_0$ is the identity matrix]{\includegraphics[trim={0 0 0 1cm},clip, width=0.45\textwidth]{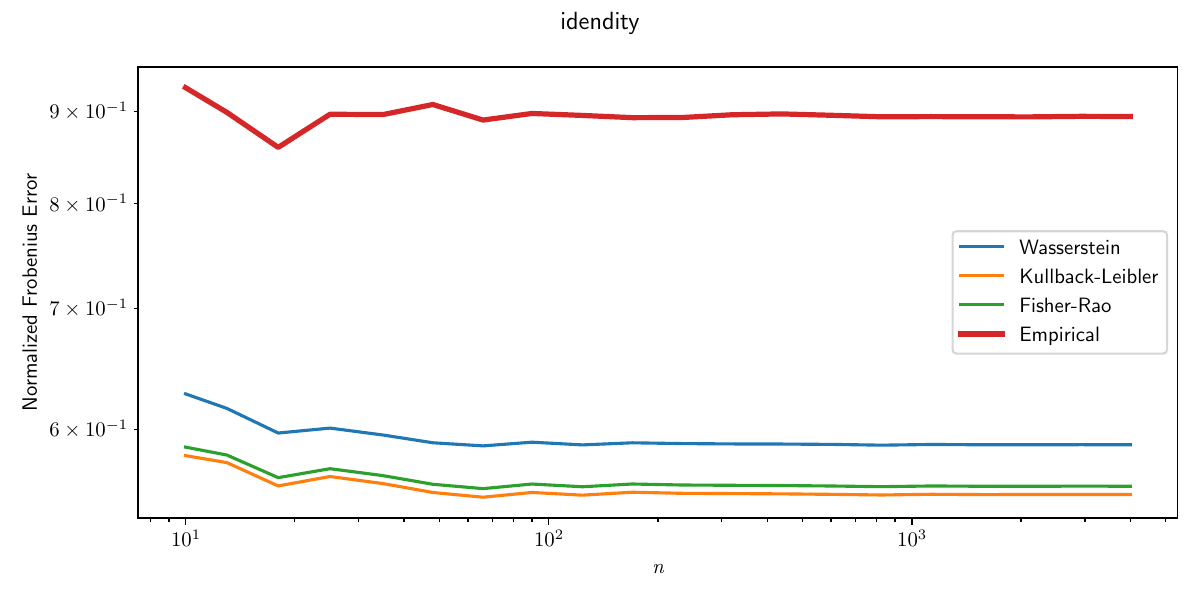}}}
    {\subfigure[True covariance $\cov_0$ is band-diagonal]{\includegraphics[trim={0 0 0 1cm},clip, width=0.45\textwidth]{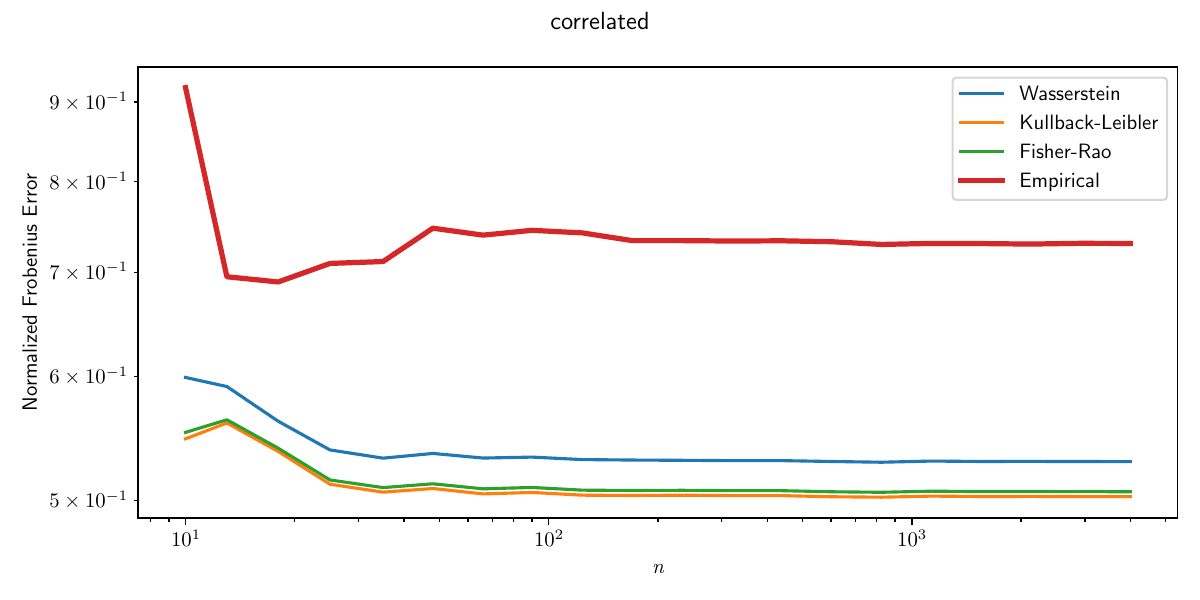}}}
    {\caption{Normalized Frobenius error of the distributionally robust estimator based on the empirically optimal radius~$\widehat{\varepsilon}_n$ in the high-dimensional regime, plotted as a function of~$n$.}
    \label{fig:consistency:small-data-optimized}}
\end{figure}
}

\subsection{Frobenius Error}
\label{sec:ex:synthetic}

In the next experiment, we use synthetic data to analyze the Frobenius risk of different covariance estimators.
Specifically, we construct a diagonal covariance matrix~$\cov_0\in\PD^{100}$ with~$90$ eigenvalues equal to~$1$ and~$10$ `spiking' eigenvalues equal to~$M\in\{10, 100, 500\}$. Thus, we have~$\kappa(\cov_0)=M$. Next, we let~$\covsa$ be the sample covariance matrix constructed from~$n\in\{100, 200, 500\}$ independent samples from~$\mathbb P=\mathcal N(0,\cov_0)$. This experimental setup captures the small to medium sample size regime with~$n\gtrsim p$, in which we expect~$\covsa$ to provide a poor approximation for~$\cov_0$. We thus compare~$\covsa$ against the Kullback-Leibler, Wasserstein, and Fisher-Rao covariance shrinkage estimators as well as against the linear shrinkage estimator with shrinkage target~$\frac{1}{p} \Tr{\covsa} I_p$ and against the NLLW estimator. Figure~\ref{fig:synthetic_rho} visualizes the Frobenius loss of all estimators as a function of the underlying hyperparameters, that is, the radius~$\eps$ of the uncertainty set for the distributionally robust estimators and the mixing weight~$\alpha$ for the linear shrinkage estimator. The NNLW estimator and the sample covariance matrix involve no hyperparameters and are thus visualized as horizontal lines. Figure~\ref{fig:synthetic_rho} shows both the means (solid lines) as well as the areas within one standard deviation of the means (shaded areas) of the Frobenius loss based on~$10$ independent training sets for all possible combinations of~$M$ and~$n$. As~$\eps$ tends to~$0$, all distributionally robust estimators approach the sample covariance matrix. Thus, they overfit the data and display a high variance. As~$\eps$ tends to~$\sum_{i=1}^p d(0, \xsa_i)$, on the other hand, all distributionally robust estimators collapse to~$0$ and thus display a high bias. We thus face a classical bias-variance trade-off. Figure~\ref{fig:synthetic_rho} reveals that the Frobenius loss of the distributionally robust estimators is minimal at intermediate values of~$\eps$. We observe that the linear shrinkage estimator is competitive with the distributionally robust estimators for well-conditioned covariance matrices (small $M$, top row). As the covariance matrix becomes more ill-conditioned (large~$M$, middle and bottom rows), the linear shrinkage estimator is dominated by the distributionally robust estimators, which attain a significantly smaller Frobenius loss. The advantage of the distributionally robust estimators relative to the nominal sample covariance matrix diminishes with increasing sample size~$n$. The NLLW estimator is designed to be asymptotically optimal and, therefore, dominates the other estimators for large sample sizes. However, it is suboptimal if training samples are scarce.

The insights of this synthetic experiment can be summarized as follows. Linear shrinkage estimators are suitable for well-conditioned covariance matrices and small sample sizes, while the NLLW estimator is preferable for large sample sizes, irrespective of the condition number. The distributionally robust estimators perform better when the covariance matrix is ill-conditioned and training samples are scarce.

\begin{figure}
    \centering
    \subfigure[$n=100$, $M=10$]{\includegraphics[width=0.3\textwidth, trim= 0 0 0 0, clip]{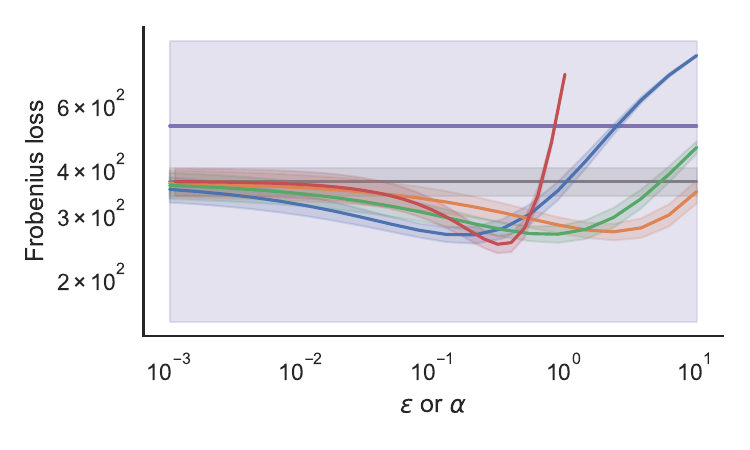}}
    \subfigure[$n=200$, $M=10$]{\includegraphics[width=0.3\textwidth, trim= 0 0 0 0, clip]{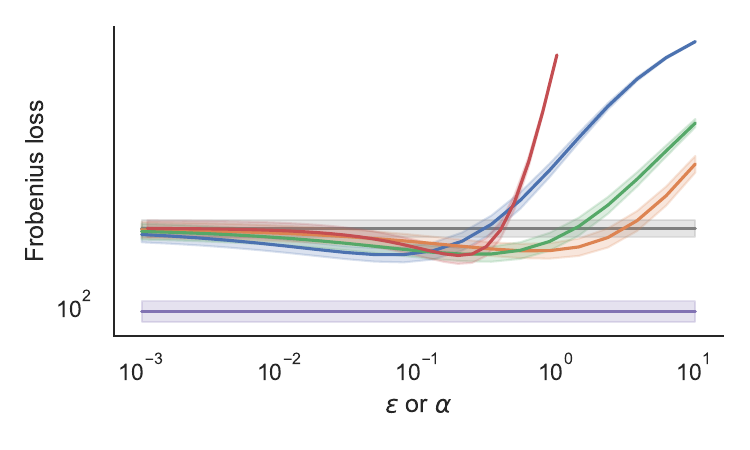}}
    \subfigure[$n=500$, $M=10$]{\includegraphics[width=0.3\textwidth, trim= 0 0 0 0, clip]{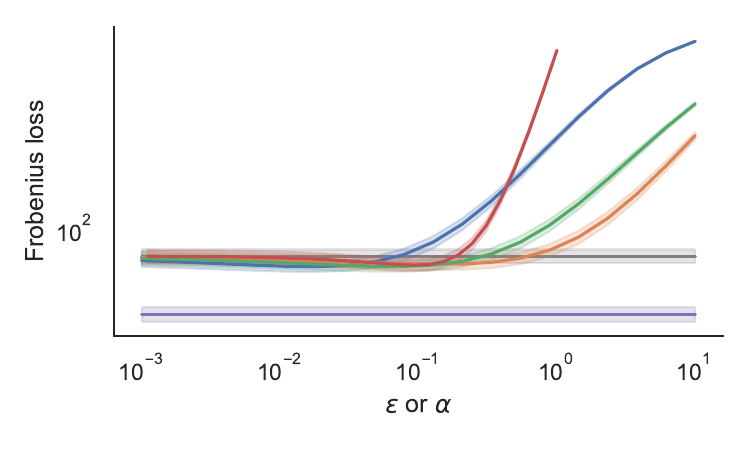}}
    
    \subfigure[$n=100$, $M=100$]{\includegraphics[width=0.3\textwidth, trim= 0 0 0 0, clip]{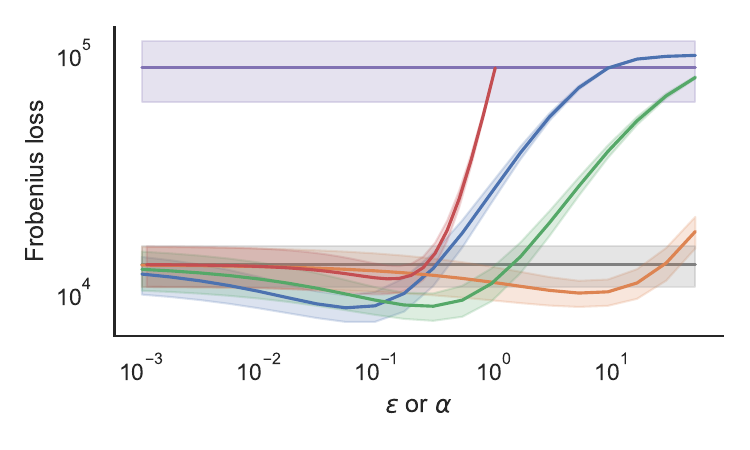}}
    \subfigure[$n=200$, $M=100$]{\includegraphics[width=0.3\textwidth, trim= 0 0 0 0, clip]{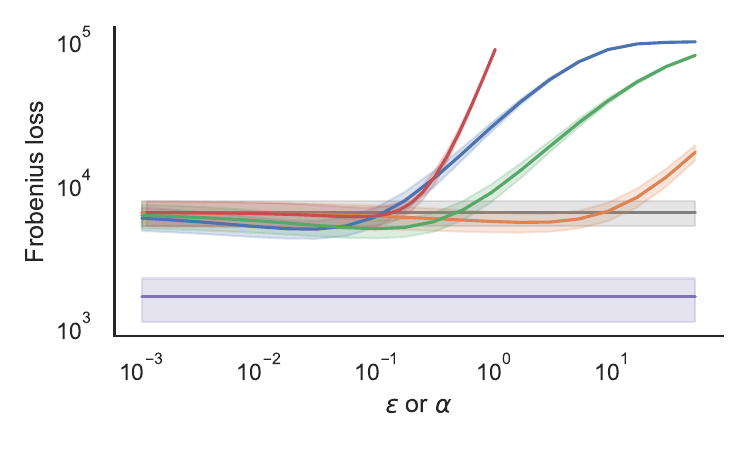}}
    \subfigure[$n=500$, $M=100$]{\includegraphics[width=0.3\textwidth, trim= 0 0 0 0, clip]{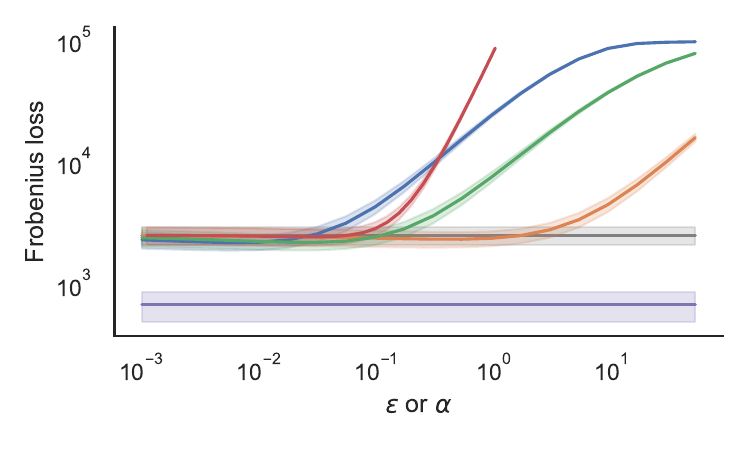}}
    
    \subfigure[$n=100$, $M=500$]{\includegraphics[width=0.3\textwidth, trim= 0 0 0 0, clip]{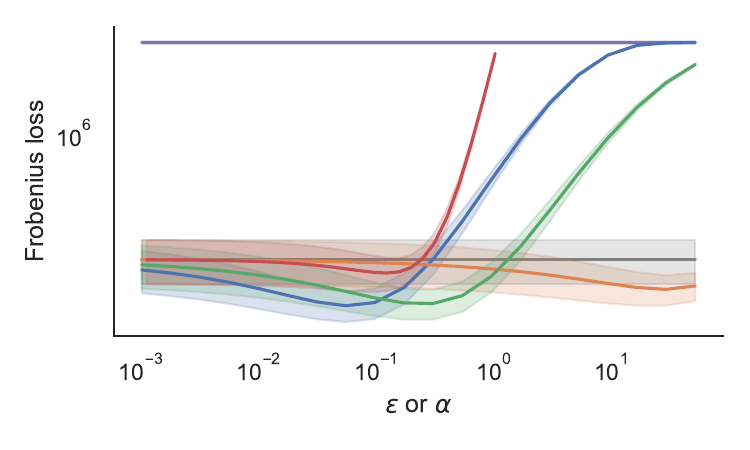}}
    \subfigure[$n=200$, $M=500$]{\includegraphics[width=0.3\textwidth, trim= 0 0 0 0, clip]{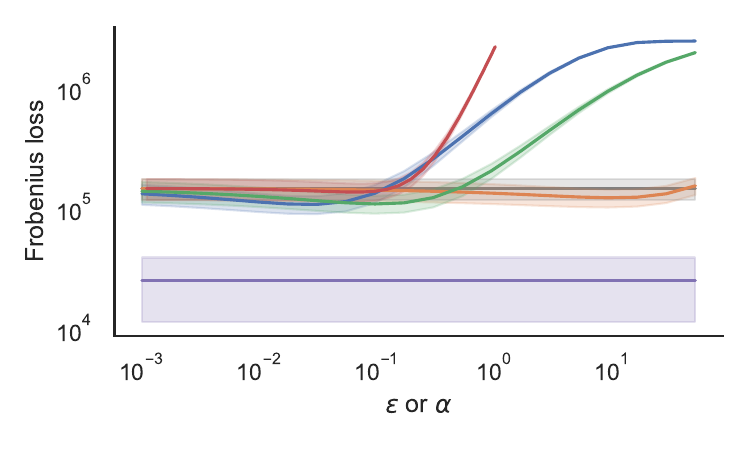}}
    \subfigure[$n=500$, $M=500$]{\includegraphics[width=0.3\textwidth, trim= 0 0 0 0, clip]{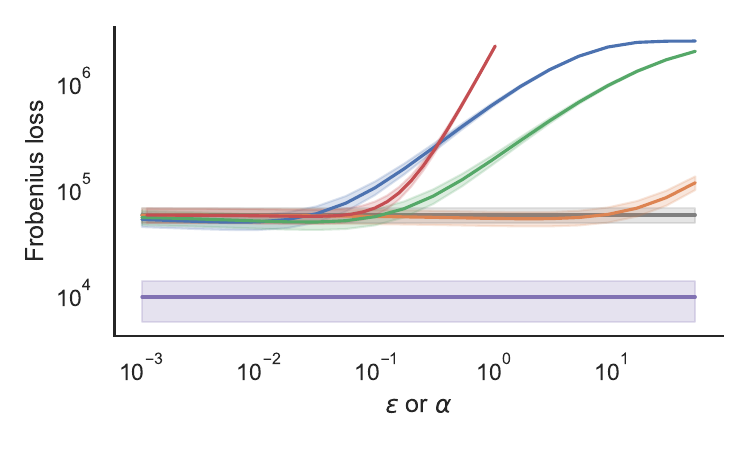}}

    \caption{Frobenius loss of the Kullback-Leibler (blue), Wasserstein (orange), and Fisher-Rao (green) covariance shrinkage estimators and of the linear shrinkage estimator (red) as a function of the underlying hyperparameter (radius~$\eps$ or mixing weight~$\alpha$) for different spike sizes~$M$ and sample sizes~$n$. The sample covariance matrix (gray) and the NLLW estimator (purple) involve no hyperparameters; thus, their Frobenius error is constant.}
    \label{fig:synthetic_rho}
\end{figure}

\subsection{Minimum Variance Portfolio Selection}

We consider the problem of constructing the minimum variance portfolio of $p$~risky assets by solving the convex program $\min_{w \in \R^p} \{ w^\top\cov_0 w: w^\top \mathds{1} = 1\}$ \cite{jagannathan2003risk}, where $\mathds{1}$ denotes the $p$-dimensional vector of ones, and $\cov_0\in\PD^p$ stands for the covariance matrix of the asset returns over the investment horizon. The unique optimal solution of this problem is given by~$w\opt=\cov_0^{-1}\mathds{1}/\mathds{1}^\top\cov_0^{-1}\mathds{1}$. In practice, however, the distribution of the asset returns is unknown, and thus the covariance matrix~$\cov_0$ needs to be estimated from historical data. If the chosen covariance estimator~$\covsa$ is invertible, then it is natural to use $\wh w\opt = \covsa^{-1}\mathds{1}/\mathds{1}^\top\covsa^{-1}\mathds{1}$ as an estimator for the minimum variance portfolio. This approach seems reasonable, provided that the asset return distribution is stationary over the (past) estimation window and the (future) investment horizon.

In the next experiment, we assess the minimum variance portfolios induced by several covariance estimators on the ``48 Industry Portfolios" dataset from the Fama-French online library,\footnote{\url{https://mba.tuck.dartmouth.edu/pages/faculty/ken.french/data_library.html}} which contains monthly returns of 48 portfolios grouped by industry. Specifically, we adopt the following rolling horizon procedure from January~1974 to December~2022. First, we estimate~$\cov_0$ from the historical asset returns within a rolling estimation window of $50$~months and construct the corresponding minimum variance portfolio. We then compute the returns of this portfolio over the $k$~months immediately after the estimation window. Finally, the covariance estimators are recalibrated based on a new estimation window shifted ahead by $k$~months, and the procedure starts afresh. Some covariance estimators involve a hyperparameter, which we calibrate via leave-one-out cross-validation on the $50$~return samples in each estimation window. To this end, we assume that the mixing weight~$\alpha$ of the linear shrinkage estimator with shrinkage target~$\frac{1}{p} \Tr{\covsa} I_p$ ranges from~$10^{-5}$ to~$1$, whereas the radius~$\eps$ of the uncertainty set ranges from~$10^{-5}$ to $10^{2}$ for the Kullback-Leibler shrinkage estimator, from~$10^{-10}$ to $10^{4}$ for the Fisher-Rao covariance shrinkage estimators and from~$10^{-10}$ to $10^{8}$ for the Wasserstein covariance shrinkage estimator. We discretize these parameter ranges into $50$~logarithmically spaced candidate values and select the one that induces the smallest portfolio variance. Given the selected hyperparameter, the covariance estimator corresponding to the current estimation window is computed using all $50$~data points. In the following, we measure the quality of a given covariance estimator by Sharpe ratio and the mean and the standard deviation of the portfolio returns generated by the above rolling horizon procedure over the backtesting period.

Figure~\ref{fig:fama:results} displays the Sharpe ratios, means, and standard deviations  corresponding to different covariance estimators as a function of the length~$k$ of an updating period. All shrinkage estimators produce lower standard deviations and higher Sharpe ratios than the sample covariance matrix. Even though the mean portfolio returns of the sample covariance matrix are---on average---similar to those of the shrinkage estimators, they change rapidly with~$k$, which is troubling for investors who need to select~$k$ before seeing the results of the backtest. The distributionally robust estimators proposed in this paper outperform the other shrinkage estimators in terms of mean returns and Sharpe ratios for most choices of $k$, and the Wasserstein covariance shrinkage estimator results in the globally highest Sharpe ratio. However, the Kullback-Leibler and Fisher-Rao covariance shrinkage estimators result in slightly higher means and standard deviations. 
\begin{figure}
    \centering
    \subfigure[Sharpe ratio]{\includegraphics[height=0.2\textwidth, trim= 0 0 4.2cm 0, clip]{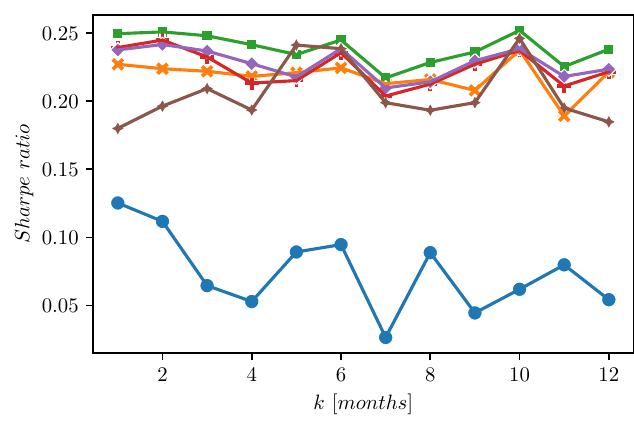}}
    \subfigure[Mean return]{\includegraphics[height=0.2\textwidth, trim= 0 0 4.2cm 0, clip]{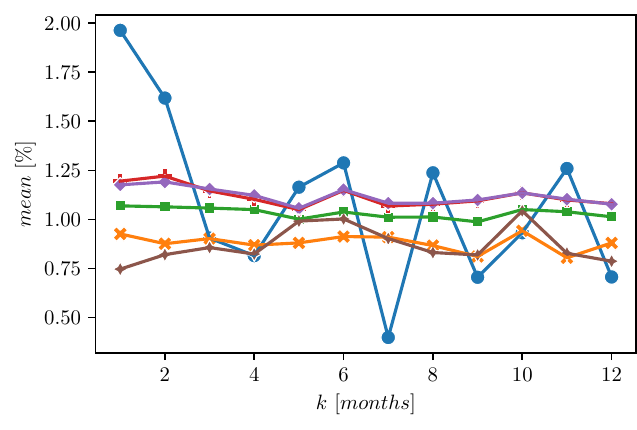}}
    \subfigure[Standard deviation of return]{\includegraphics[height=0.2\textwidth, trim= 0 0 0 0, clip]{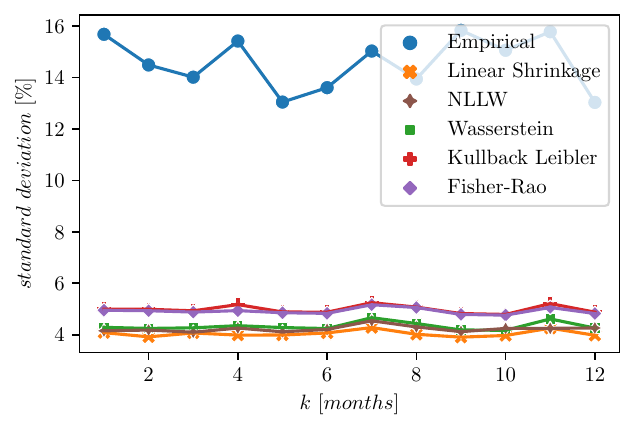}}
    \caption{Sharpe ratios, means, and standard deviations induced by different covariance estimators on the ``48 Industry Portfolios" depending on the length~$k$ of an updating period.}
    \label{fig:fama:results}
\end{figure}

\subsection{Linear and Quadratic Discriminant Analysis}
Quadratic discriminant analysis (QDA) seeks to predict a label~$y\in\{0,1\}$ from a feature vector~$z\in\R^p$ under the assumption that $z|y\sim \mathcal{N}(\mu_y,\cov_y)$ for every~$y\in\{0,1\}$. If the mean~$\mu_y$, the covariance matrix~$\cov_y$ as well as the marginal class probability~$p_y$ are known for all~$y\in\{0,1\}$, then the Bayes-optimal classifier predicts~$y$ as a solution of $\min_{y\in\{0,1\}} (z-\mu_y)\cov_y^{-1}(z-\mu_y)+\log\det(\cov_y)-2\log(p_y)$. Linear discriminant analysis (LDA) operates under the additional assumption that $\cov_0=\cov_1$. The decision boundaries of the resulting LDA and QDA classifiers are thus given by linear hyperplanes and quadratic hypersurfaces, respectively~\cite{hastie2009elements}. 

In the last experiment, we use LDA and QDA to address the breast cancer detection~\cite{breast_cancer_dataset} and banknote authentication~\cite{bankonte_dataset} problems from the UCI Machine Learning Repository. As the distribution governing~$y$ and~$z$ is unobservable, we replace the unknown class probabilities~$p_y$ and class means~$\mu_y$ by the empirical frequencies and sample average estimators, respectively, and we use different shrinkage estimators for the unknown covariance matrices~$\cov_y$. All tested shrinkage estimators use the debiased empirical covariance matrix as the nominal estimator. QDA constructs a separate covariance estimator for each class~$y$ that only uses class-$y$ samples, whereas LDA pools all samples to construct a single joint covariance estimator. 

We use 50\% of each dataset for training and the rest for testing. The hyperparameters~$\varepsilon$ (for the distributionally robust shrinkage estimators) and~$\alpha$ (for the linear shrinkage estimator) are selected by the holdout method with a validation set comprising 20\% of the training data. The quality of a covariance estimator is then measured by the accuracy ({\em i.e.}, the proportion of correct predictions) of the resulting LDA and QDA classifiers. Table~\ref{tab:lda-qda} reports the means and standard errors of the accuracy achieved by different covariance estimators. We observe that shrinking the empirical covariance estimator can improve the performance of LDA and QDA, and that nonlinear shrinkage methods outperform the linear shrinkage method across all experiments. The Kullback-Leibler covariance shrinkage estimator consistently performs well. QDA based on the NLLW estimator attains the highest accuracy on the banknote authentication dataset but performs poorly on the breast cancer dataset. On the other hand, the distributionally robust covariance estimators are consistently on par with or better than the empirical and the linear shrinkage estimator. Note that the best-performing distributionally robust shrinkage estimator changes with the dataset. This highlights the usefulness of our approach, which results in a zoo of complementary covariance shrinkage estimators. 

\begin{table}[t]
\centering
\caption{Mean (standard error) of the LDA and QDA accuracy based on 100 independent permutations of the underlying dataset}
\label{tab:lda-qda}
\resizebox{\linewidth}{!}{
\begin{tabular}{ll|cccccc}
\toprule
 & Dataset      & Empirical    & Linear & NLLW            & Wasserstein       & Kullback-Leibler                & Fisher-Rao                             \\ \midrule
\multirow{2}{*}{LDA} & Banknote     & 0.9751(0.0005) & 0.9754(0.0005)  & 0.9510(0.0011) & 0.9761(0.0005) & \textbf{0.9763(0.0005)} & 0.9759(0.0005)\\
                     &  Cancer & \textbf{0.9520(0.0011)} & 0.9365(0.0015)& 0.8902(0.0015) & \textbf{0.9520(0.0011)} & 0.8874(0.0043) & 0.9515(0.0013)  \\\midrule
\multirow{2}{*}{QDA} & Banknote     & 0.9854(0.0005) & 0.9839(0.0005)& \textbf{0.9877(0.0004)} & 0.9854(0.0005) & 0.9853(0.0005) & 0.9854(0.0005)  \\
                     &  Cancer & 0.9418(0.0012) & 0.8945(0.0027)& 0.6320(0.0052) & 0.9418(0.0012) & \textbf{0.9451(0.0013)} & 0.9414(0.0016)  \\
\bottomrule
\end{tabular}}
\end{table}

\textbf{Acknowledgments.} This research was supported by the Hong Kong Research Grants Council under the GRF project 15304422, by the Swiss National Science Foundation under the NCCR Automation, grant agreement 51NF40\_180545, and by CUHK through the `Improvement on Competitiveness in Hiring New Faculties Funding Scheme' and the CUHK Direct Grant with project number 4055191.

\bibliographystyle{siam}
\bibliography{bibliography}

\newpage

\begin{appendix}
\section*{Appendix}
The appendix is organized as follows. In Appendix~\ref{app:proof-general_CSE}, we prove Theorem~\ref{thm:general_CSE} and derive basic properties of~$\dualvar\opt$ and~$x_i^\star$, which will be used in Appendix~\ref{sec:app-property} to establish the computational, structural and statistical properties of the distributionally robust estimators. Appendices~\ref{app:ass-1-verification} and~\ref{app:verify-ass2c} verify Assumptions~\ref{assu:inf_sup} and~\ref{assu:D_form} for all divergences in Table~\ref{table:structured_divergence}, respectively. The insights of Appendices~\ref{app:ass-1-verification} and~\ref{app:verify-ass2c} are used in Appendix~\ref{app:proofs_section-4} to prove the results of Section~\ref{sec:new_CSE}.

\section{Proof of Theorem~\ref{thm:general_CSE}}\label{app:proof-general_CSE}

\subsection{Proof of Proposition~\ref{prop:vec-equivalence}} \label{sec:app-prop2}

To simplify the subsequent discussions, for any minimization problem designated by~``P,'' say, we use ``$\Minval(\text{P})$,'' ``$\Argmin(\text{P})$'' and ``$\Fea(\text{P})$'' to denote its minimum/infimum, the set of its optimal solutions and its feasible region, respectively.

\begin{proof}[Proof of Proposition~\ref{prop:vec-equivalence}]
Select any $\cov\in \Fea\eqref{eq:matrix} $, and use $\cov = V_\cov\Diag(\lambda(\cov))V_\cov^\top$ to denote its eigenvalue decomposition. By our notational conventions, we have $0 \leq \lambda_1(\cov) \leq \cdots \leq \lambda_p(\cov)$. We then obtain
\begin{equation}\label{ineq:thm:master:proof:1}
\begin{split}
\sum_{i=1}^p d\left(\lambda_i(\cov), \hat{x}_i\right)  = &\, D\left(\Diag(\lambda(\cov)), \Diag(\xsa) \right)   
\le D\left( \Vsa^\top V_\cov \Diag(\lambda(\cov)) V_\cov^\top \Vsa, \Diag(\hat{x}) \right) \\  
=&\, D\left( V_\cov \Diag(\lambda(\cov)) V_\cov^\top, \Vsa \Diag(\hat{x})\Vsa^\top \right)  = D(\cov, \covsa) \le \eps,
\end{split}
\end{equation}
where the first equality follows from Assumption~\ref{assu:D_form}\ref{assu:D_form_iii}, the first inequality follows from Assumption~\ref{assu:D_form}\ref{assu:D_form_iv}, and the second equality follows from Assumption~\ref{assu:D_form}\ref{assu:D_form_ii}.
This implies that $\lambda(\cov) \in \Fea\eqref{eq:vector}$. 

Next, select any $x\in \Fea\eqref{eq:vector}$ such that $\Vsa \Diag(x) \Vsa^\top\in \PSD^p$. We thus have
\begin{equation}\label{ineq:thm:master:proof:2}
\begin{split}
&\, D(\Vsa \Diag(x) \Vsa^\top, \covsa)  = D(\Vsa \Diag(x) \Vsa^\top, \Vsa \Diag(\hat{x}) \Vsa^\top) = D(\Diag(x), \Diag(\hat{x}))=  \sum_{i=1}^p d(x_i, \hat{x}_i)\le \eps,
\end{split}
\end{equation}
where the three equalities follow from the eigenvalue decomposition of $\covsa$, Assumption~\ref{assu:D_form}\ref{assu:D_form_ii} and Assumption~\ref{assu:D_form}\ref{assu:D_form_iii}, respectively.
This implies that $\Vsa \Diag(x) \Vsa^\top \in \Fea\eqref{eq:matrix}$. 
In summary, we have thus shown that problem~\eqref{eq:matrix} is feasible if and only if problem~\eqref{eq:vector} is feasible. This establishes assertion~\ref{prop:master_1}.

As for assertion~\ref{prop:master_3}, assume that $\Argmin\eqref{eq:vector}\neq\emptyset$  for otherwise the claim is trivial. Choose then any $x\opt\in\Argmin\eqref{eq:vector}$, and note that~$\Vsa \Diag({x\opt}) \Vsa^\top \in \Fea\eqref{eq:matrix}$ by virtue of~\eqref{ineq:thm:master:proof:2}. It remains to be shown that $\Vsa \Diag({x\opt}) \Vsa^\top \in \Argmin\eqref{eq:matrix}$. Suppose, for the sake of contradiction, that there is $ \cov'\in \Fea\eqref{eq:matrix}$~with
\[
    \norm{ \cov' }_{\mathrm{F}}^2 < \norm{ \Vsa \Diag({x\opt}) \Vsa^\top }_{\mathrm{F}}^2 ,
\] 
and let $\cov' = V' \Diag(\lambda (\cov')){V'}^\top $ be the eigenvalue decomposition of $\cov'$ for some $V'\in\mathcal{O}_p$. By~\eqref{ineq:thm:master:proof:1}, we then have $\lambda (\cov')\in \Fea\eqref{eq:vector}$, which contradicts the optimality of $x\opt$ in problem~\eqref{eq:vector} because
\[
    \norm{\lambda (\cov')}_2^2 = \norm{ \cov' }_{\mathrm{F}}^2 < \norm{ \Vsa \Diag({x\opt}) \Vsa^\top }_{\mathrm{F}}^2= \norm{x\opt}_2^2 .
\] 
Therefore, $\Vsa \Diag({x\opt}) \Vsa^\top \in \Argmin\eqref{eq:matrix}$. This proves assertion~\ref{prop:master_3}.

As for assertion~\ref{prop:master_4}, assume that $\Argmin\eqref{eq:matrix}\neq\emptyset$ for otherwise the claim is trivial. Choose then any $\cov\opt\in\Argmin\eqref{eq:matrix}$, and note that $\lambda(\cov\opt)\in \Fea\eqref{eq:vector}$ by virtue of~\eqref{ineq:thm:master:proof:1}. It remains to be shown that $\lambda(\cov\opt) \in \Argmin\eqref{eq:vector}$. Suppose, for the sake of contradiction, that there is $x'\in\Fea\eqref{eq:vector}$ with 
\[
    \norm{x'}_2^2 < \norm{\lambda(\cov\opt)}_2^2.
\] 
By~\eqref{ineq:thm:master:proof:2}, we then have $\Vsa \Diag(x') \Vsa^\top \in \Fea\eqref{eq:matrix}$, which contradicts the optimality of $\cov\opt$ in~\eqref{eq:matrix} because
\[ \norm{\Vsa \Diag(x') \Vsa^\top}_{\mathrm{F}}^2 = \norm{x'}_2^2 < \norm{\lambda(\cov\opt)}_2^2  = \norm{\cov\opt}_{\mathrm{F}}^2 .\]
Therefore, $\lambda(\cov\opt) \in \Argmin\eqref{eq:vector}$. This proves assertion~\ref{prop:master_4}.

Finally, in order to prove assertion~\ref{prop:master_5}, we need to show that any $\cov\in \Fea\eqref{eq:matrix} $ corresponds to some $x \in \Fea\eqref{eq:vector}$ with the same objective function value and vice versa. However, this follows in a straightforward manner from the proof of assertion~\ref{prop:master_1}. Details are omitted for brevity.
\end{proof}

\subsection{Proof of Proposition~\ref{prop:gamma-reconstruction}} \label{sec:app-prop3}

The next lemma shows that any solution of problem~\eqref{eq:vector} shrinks $\xsa$ towards the origin. This will imply that our proposed distributionally robust estimators constitute shrinkage estimators. From now on we use $d_b(\cdot)$ as a notational shorthand for the function $d(\cdot, b)$ for any fixed~$b\ge 0$.

\begin{lemma}[Eigenvalue shrinkage]
\label{lem:opt_vector_sol_upper_bound}
If Assumptions~\ref{assu:D_form} and~\ref{assu:data}\ref{assu:data-d} hold and $x\opt$ solves problem~\eqref{eq:vector}, then we have $x\opt_i \in \mathrm{dom}(d_{\xsa_i})$ and $x\opt_i \le \xsa_i$ for all $i = 1,\dots,p$.
\end{lemma}

\begin{proof}[Proof of Lemma~\ref{lem:opt_vector_sol_upper_bound}]
Select any $x\opt\in\Argmin\eqref{eq:vector}$. As $x\opt\in\Fea \eqref{eq:vector}$, it is clear that $x\opt_i \in \mathrm{dom}(d_{\xsa_i})$ for all $i=1,\dots,p$.
Next, suppose that $x\opt_j > \hat{x}_j$ for some $j = 1,\dots,p$, and define $\tilde{x}\in \R^p_+$ through 
\begin{equation*}
\tilde{x}_i = \begin{cases}
\hat{x}_j & \text{if } i = j,\\
x\opt_i &\text{if } i\neq j.
\end{cases}
\end{equation*} 
Recall now that if Assumption~\ref{assu:D_form}\ref{assu:D_form_iii} holds, then~$d$ constitutes a spectral divergence on~$\mathbb{R}_+$. Assumption~\ref{assu:data}\ref{assu:data-d} further implies that $(\xsa_j,\xsa_j)\in \mathrm{dom}(d)$. Hence, $d(\hat{x}_j, \hat{x}_j) = 0 < d(x\opt_j, \hat{x}_j) $, which ensures that $\tilde{x}\in \Fea\eqref{eq:vector}$. However, from the construction of~$\tilde{x}$ it is evident that $\norm{\tilde{x}}_2^2< \norm{x\opt}_2^2$, which contradicts the optimality of~$x\opt$ in~\eqref{eq:vector}. Thus, we have $x\opt_i \le \hat{x}_i$ for all $i=1,\dots,p$. This observation completes the proof.
\end{proof}

Lemma~\ref{lem:opt_vector_sol_upper_bound} allows us to prove the existence and uniqueness of the proposed robust covariance estimators.

{
\begin{proposition}[Existence and uniqueness of optimal solutions]\label{prop:opt_sol_existence}
If Assumptions~\ref{assu:D_form}, \ref{assu:data} and~\ref{assu:d_convex} hold, then problems~\eqref{eq:vector} and~\eqref{eq:matrix} admit a unique optimal solution.
In addition, if Assumptions~\ref{assu:inf_sup}, \ref{assu:D_form}, \ref{assu:data} and~\ref{assu:d_convex} hold, then there exists a unique distributionally robust estimator that solves problem~\eqref{eq:dro}.
\end{proposition}

\begin{proof}[Proof of Proposition~\ref{prop:opt_sol_existence}]
Suppose first that only Assumptions~\ref{assu:D_form}, \ref{assu:data} and~\ref{assu:d_convex} hold. Lemma~\ref{lem:opt_vector_sol_upper_bound} then implies that problem~\eqref{eq:vector} has the same set of optimal solutions as the following variant of~\eqref{eq:vector} with box constraints
\begin{equation}\label{opt:vector'}\tag{P$'_{\text{Vec}}$}
\begin{array}{cl}
    \Inf{x \in \R^p}  & \norm{ x }_2^2 \\
    \st & \displaystyle \sum_{ i = 1 }^p d\left( x_i, \xsa_i \right)  \le \eps\\
    & 0\le x_i \le \xsa_i \quad \forall i=1,\dots,p.
\end{array}
\end{equation}
Note that problem~\eqref{opt:vector'} is feasible due to Assumption~\ref{assu:data}\ref{assu:data-d}, which posits that $d(\xsa_i, \xsa_i) = 0$ for all $i=1,\dots,p$. Next, we show that the feasible region of~\eqref{opt:vector'} is compact. To this end, note that $d(x_i,\xsa_i)$ is continuous in~$x_i$ on the interval $[0,\xsa_i]$ for every $i=1,\ldots,p$. Indeed, continuity trivially holds if $\xsa_i=0$, in which case~$[0,\xsa_i]$ collapses to a point. Otherwise, if $\xsa_i>0$, then continuity follows from Assumption~\ref{assu:D_form}\ref{assu:D_form_iii}. This readily implies that the feasible region of~\eqref{opt:vector'} is closed and---thanks to the box constraints---also compact. The solvability of problem~\eqref{opt:vector'} thus follows from Weierstrass' maximum theorem, which applies because the objective function is continuous. Assumption~\ref{assu:d_convex} further implies that $d(x_i, \xsa_i)$ is convex in~$x_i$ on~$[0,\xsa_i]$ for all $i=1,\dots,p$, which implies that the feasible region of~\eqref{opt:vector'} is convex. The uniqueness of the optimal solution~$x\opt$ thus follows from the strong convexity of the objective function. This shows that problem~\eqref{eq:vector} has a unique optimal solution. 
The other claims immediately follow from Propositions~\ref{prop:exist} and~\ref{prop:vec-equivalence}.
\end{proof}
}

\begin{proposition}[Solution of problem~\eqref{eq:vector}]
\label{prop:opt_vec_solution}
If Assumptions~\ref{assu:D_form}, \ref{assu:data} and~\ref{assu:d_convex} hold, then the unique minimizer~$x\opt$ of problem~\eqref{eq:vector} has the following properties. If~$\xsa_i = 0$, then~$x\opt_i = 0$, and if $\xsa_i > 0$, then $x\opt_i\in ( 0, \xsa_i)$ and 
\begin{equation}
\label{eq:x_star_optimality_condition}
    0 = 2 x\opt_i + \gamma\opt d_{\xsa_i}' (x\opt_i),
\end{equation}
where $\gamma\opt$ is a solution of the nonlinear equation $\sum_{i = 1}^p d(s(\gamma\opt,\xsa_i), \xsa_i) -\eps = 0$.
\end{proposition}

The following lemma shows that $d_b$ is strictly decreasing on $[0,b]$, which will be used to prove Proposition~\ref{prop:opt_vec_solution}.
\begin{lemma}[Derivative of $d_b$]
\label{lem:negative_subgradient}
If Assumptions~\ref{assu:D_form} and~\ref{assu:d_convex} hold, then we have
\[
    d_b' (a) \le -\frac{d(a,b)}{b - a} < -\frac{d(a,b)}{b } < 0 \quad \forall a\in (0,b), ~\forall b>0.
\]
\end{lemma}

\begin{proof}[Proof of Lemma~\ref{lem:negative_subgradient}]
Select any $b>0$. As $d(\cdot,b)$ is finite and convex on $[0,b]$ thanks to Assumption~\ref{assu:d_convex}, we have
\[ 
    0 = d(b, b) \ge d(a, b) + (b -a) \,d_b' (a)\quad \forall a\in(0,b).
\]
The desired inequality then follows from an elementary rearrangement.
\end{proof}

\begin{proof}[Proof of Proposition~\ref{prop:opt_vec_solution}] 
Lemma~\ref{lem:opt_vector_sol_upper_bound} allows us to rewrite problem~\eqref{eq:vector} equivalently as
\begin{equation}
\label{opt:vector''}\tag{P$''_{\text{Vec}}$}
    \begin{array}{cl}
    \Min{x\in \mathcal{C}}  & \norm{ x }_2^2 \\ 
    \st & \displaystyle \sum_{ i = 1 }^p d\left( x_i, \xsa_i \right)  \le \eps,
    \end{array}
\end{equation}
where $\mathcal{C} = \mathcal{C}_1\times\cdots \times \mathcal{C}_p$ with $\mathcal{C}_i = [0,\xsa_i]\cap \mathrm{dom}(d_{\xsa_i})$ for each $i = 1,\dots,p$. 
Note that the objective and constraint functions adopt finite values on $\mathcal{C}$. { By Proposition~\ref{prop:opt_sol_existence} and Lemma~\ref{lem:opt_vector_sol_upper_bound}, problem~\eqref{opt:vector''} has a unique minimizer~$x\opt$ satisfying $x\opt_i = 0$ for all $i$ with $\xsa_i  = 0$.} For such indices~$i$, $d (0, 0) = d(\xsa_i, \xsa_i) = 0$ by Assumption~\ref{assu:data}\ref{assu:data-d}. By removing the corresponding decision variables from~\eqref{opt:vector''} and focusing on the optimization problem in the remaining variables, we can therefore assume without loss of generality that~$\xsa_i > 0$ for all $i = 1,\dots, p$. Hence, problem~\eqref{opt:vector''} can be viewed as an ordinary convex program in the sense of \cite[Section~28]{ref:rockafellar1997convex}.

Following \cite[Section~28]{ref:rockafellar1997convex}, we define the Lagrangian $L:\mathbb{R}\times \mathbb{R}^p \rightarrow\overline{\mathbb R}$ of problem~\eqref{opt:vector''} through
\begin{equation*}
    L( \gamma, x ) = 
    \begin{cases}
        \|x\|_2^2 + \gamma (\sum_{i=1}^p d (x_i , \xsa_i) - \eps) & \text{if } x\in \mathcal{C}, \gamma\ge 0, \\
        -\infty & \text{if } x\in \mathcal{C}, \gamma < 0,\\
        +\infty & \text{if } x\not\in \mathcal{C}.
    \end{cases}
\end{equation*}
By~\cite[Corollary~28.2.1 and Theorem~28.3]{ref:rockafellar1997convex}, problem~\eqref{opt:vector''} is thus equivalent to the minimax problem
\begin{align*}
    \min_{x\in \mathbb{R}^p} \sup_{\gamma \in \mathbb{R}} ~ L(\gamma, x)  =  \max_{\gamma \in \mathbb{R}} \min_{x\in \mathbb{R}^p} ~ L(\gamma, x).
\end{align*}
Specifically, the dual maximization problem on the right-hand side is solvable, and every maximizer~$\gamma\opt \ge 0$ gives rise to a saddle point $(\gamma\opt, x\opt)$ of the minimax problem.
Next, we prove that $\gamma\opt > 0$. Suppose for the sake of contradiction that $\gamma\opt = 0 $. Since $x\opt\in \mathcal{C}$, we find $L (\gamma\opt, x\opt) = L(0, x\opt) = \| x\opt \|_2^2$.  
If~$x\opt_i > 0 $ for some~$i$,~then 
\begin{equation*}
    0 < \| x\opt \|_2^2 = L(0, x\opt) \le L(0, x) = \| x \|_2^2\quad \forall x\in \mathcal{C},
\end{equation*}
where the second inequality holds because $(0, x\opt)$ is a saddle point. However, the discussion after Assumption~\ref{assu:d_convex} implies that $\mathrm{dom}(d_{\xsa_i}) $ either equals $\mathbb{R}_+$ or $\mathbb{R}_{++}$ for every $i = 1,\dots,p$. Hence, we have $\prod_{ i =1}^p (0,\xsa_i] \subseteq \mathcal{C}$, that is, $\mathcal C$ contains points that are arbitrarily close to~$0$. This leads to the contradiction
\[ 
    0 = \inf_{x\in\mathcal{C}} \|x\|_2^2 \ge \|x\opt \|_2^2 > 0.
\]
We may thus conclude that if~$\gamma\opt = 0$, then~$x_i\opt = 0$ for all $i$, that is, $x\opt=0$. However, this contradicts Assumption~\ref{assu:data}\ref{assu:data-eps}, which implies that $0\not\in \Fea\eqref{eq:vector}$. In summary, this shows that~$\gamma\opt > 0$. 

Next, we note that for any dual optimal solution $\gamma\opt > 0$, the minimization problem
\begin{equation}
    \label{opt:min_x}
    \min_{x\in\mathbb{R}^p } L (\gamma\opt, x) = \min_{ x\in \mathcal{C} } \|x\|_2^2 + \gamma\opt \left(\sum_{i=1}^p d(x_i, \xsa_i) - \eps\right)
\end{equation}
admits a unique optimal solution, and by \cite[Corollary~28.1.1]{ref:rockafellar1997convex} this minimizer must coincide with the unique optimal solution~$x\opt$ of problem~\eqref{opt:vector''}. Given~$\gamma\opt$, we can thus solve~\eqref{opt:min_x} instead of~\eqref{opt:vector''}. This is attractive from a computational point of view because~$\mathcal C$ is rectangular, whereby problem~\eqref{opt:min_x} can be simplified to 
\begin{align*}
     -\eps \gamma\opt  +  \sum_{i=1}^p \min_{x_i\in\mathcal{C}_i }\left\lbrace  x_i^2 + \gamma\opt d(x_i, \xsa_i)  \right\rbrace =  -\eps \gamma\opt  +  \sum_{i=1}^p \min_{ x_i \in [0,\xsa_i]}\left\lbrace  x_i^2 + \gamma\opt d(x_i, \xsa_i)  \right\rbrace .
\end{align*} 
Therefore, it suffices to solve the following simple univariate minimization problem for each $i=1,\ldots, p$.
\begin{equation}
    \label{opt:inner_min}
    \min_{ x_i \in [0,\xsa_i]}~ x_i^2 + \gamma\opt d(x_i, \xsa_i) 
\end{equation}
If $\xsa_i = 0$, then $(0,0) \in \mathrm{dom}(d)$ by Assumption~\ref{assu:data}\ref{assu:data-d}, and hence $d(0,0) = d(\xsa_i, \xsa_i) = 0$. In this case, $x\opt_i = 0$ is the only feasible---and thus unique optimal---solution of~\eqref{opt:inner_min}. Assume next that $\xsa_i > 0$. In this case we need to prove that $x\opt_i$ falls within the open interval~$(0,\xsa_i)$ and satisfies \eqref{eq:x_star_optimality_condition}. We will first show that~$x\opt_i > 0$. From the discussion after Assumption~\ref{assu:d_convex} we know that~$d_{x\opt_i}$ can evaluate to~$+\infty$ only at~$0$.  If~$d_{\xsa_i} (0) = +\infty$, then we trivially have $x\opt_i > 0$. Assume next that~$d_{\xsa_i} (0) < +\infty$. 
By Assumption~\ref{assu:D_form}\ref{assu:D_form_iii}, $d_{\xsa_i} $ is continuous and $d_{\xsa_i}(0) > 0$. There exists a threshold~$\delta > 0$ such that $d_{\xsa_i} (a) \ge \delta$ for all sufficiently small $a \in [0, \xsa_i]$. In addition, as the function $a^2 + \gamma\opt d(a,\xsa_i)$ is convex and differentiable in~$a$ by virtue of Assumption~\ref{assu:d_convex}, we have 
\begin{align*}
    0^2 + \gamma\opt d(0,\xsa_i) & \ge a^2 + \gamma\opt d(a,\xsa_i) + (2a + \gamma\opt d_{\xsa_i}'(a) ) (0 - a)  \\
    & > a^2 + \gamma\opt d(a,\xsa_i) - 2a^2 + \frac{a \gamma\opt d(a,\xsa_i)}{\xsa_i } \\
    & \ge a^2 + \gamma\opt d(a,\xsa_i) - 2a^2 + \frac{a \gamma\opt \delta}{\xsa_i } 
\end{align*}
for all sufficiently small~$a\geq 0$. Here, the second inequality follows from Lemma~\ref{lem:negative_subgradient}, and the third inequality holds because~$d_{\xsa_i} (a) \ge \delta$ for all sufficiently small~$a\geq 0$. This reasoning implies that 
\begin{equation}
    \label{ineq:a>0}
    \gamma\opt d(0,\xsa_i) > a^2 + \gamma\opt d(a,\xsa_i) - 2a^2 + \frac{a \gamma\opt \delta}{\xsa_i } > a^2 + \gamma\opt d(a,\xsa_i)
\end{equation}
for all sufficiently small~$a\geq 0$. Thus, small $a > 0$ are strictly preferable to~$0$, that is, $x\opt_i > 0$.

Next, we prove that~$x\opt_i < \xsa_i$. As the differentiable function~$d_b(a)$ is non-negative and attains its minimum~$0$ at~$a = b$, we may conclude that its derivative~$d'_b (a)$ converges to~$0$ as~$a$ tends to~$b$. For any $a < b$ sufficiently close to~$b$ we thus have $(b-a) (2a + \gamma\opt d'_b(a)) > 0$. As $a^2 + \gamma\opt d(a,b)$ is convex in~$a$ on~$[0,b]$, this ensures that
\[ 
    b^2 + \gamma\opt d(b,b) \ge a^2 + \gamma\opt d(a, b) + (b-a) (2a + \gamma\opt d'_b(a) ) > a^2 + \gamma\opt d(a, b). 
\]
Hence, any $a<b$ sufficiently close to~$b$ is strictly preferable to~$b$. Setting~$b=\xsa_i$, we thus find~$x\opt_i < \xsa_i$.

Finally, note that since~$x\opt_i \in (0, \xsa_i)$, the constraints of problem~\eqref{opt:inner_min} are not binding at optimality. Thus, the minimizer of~\eqref{opt:inner_min}
is uniquely determined by the problem's first-order optimality condition~\eqref{eq:x_star_optimality_condition}.

It remains to be shown that~$\gamma\opt$ is unique. As~$0 \not\in \Fea\eqref{eq:vector}$ thanks to Assumption~\ref{assu:data}\ref{assu:data-eps}, there exists at least one $i = 1,\dots,p$ with $x\opt_i > 0$, and hence $\xsa_i > 0$. Since $d_{\xsa_i}$ is differentiable on $\R_{++}$, equation~\eqref{eq:x_star_optimality_condition} implies
\[ 
    \gamma\opt = -\frac{2x\opt_i}{d'_{\xsa_i} (x\opt_i)}. 
\]
Hence, $\gamma\opt$ is unique because $x\opt_i$ is unique. Note also that~$\gamma\opt$ is the Lagrange multiplier associated with the constraint $\sum_{i=1}^p d(x_i, \xsa_i ) \le \eps$ in problem~\eqref{opt:vector''}. As strong duality holds and~$\gamma\opt>0$, we have
\[ 
    \sum_{i=1}^p d (x\opt_i, \xsa_i) - \eps = 0
\]
by complementary slackness.
Using the definition~\eqref{eq:s} of the eigenvalue map $s$, we then obtain
\[ \sum_{i = 1}^p d(s(\gamma\opt,\xsa_i), \xsa_i) -\eps = 0 .\]
This observation completes the proof. 
\end{proof}

\subsubsection{{Properties of $s$ and $\gamma\opt$}}\label{sec:app-basic-gamma-x}
We first provide a detailed analysis of the nonlinear equation that defines the eigenvalue map~$s$.

\begin{lemma}[Properties of $s$]
\label{lem:a}
If Assumptions~\ref{assu:D_form} and~\ref{assu:d_convex} hold, then the the following hold.
\begin{enumerate}[label=(\roman*)]
\item\label{lem:a-3} If $\gamma > 0$ and $b >0$, then the equation $0 = 2a + \gamma d_b'(a)$ admits a unique solution in $(0,b)$. Hence, the eigenvalue map $s(\gamma, b)$ is well-defined on $\R_+^2$.
\item\label{lem:a-1}  If $b>0$, then $ s_b(\gamma)=s(\gamma, b)$ is continuous and strictly increasing on $\R_+$ and differentiable on $\mathbb{R}_{++}$.
\item\label{lem:a-2} If $b>0$, then $\lim_{\gamma \downarrow 0} s_b(\gamma) = 0$ and $\lim_{\gamma \to \infty} s_b(\gamma) = b$.
\end{enumerate} 
\end{lemma}

Recall that, for and fixed~$\gamma>0$, the function $s_\gamma(b)$ shrinks the input~$b$ in the sense that~$s_\gamma(b)\leq b$. Lemma~\ref{lem:a} further shows that, for any fixed~$b>0$, $s_b(\gamma)$ strictly increases from~$0$ to~$b$ as~$\gamma$ grows. Therefore, we can interpret $\gamma$ as an inverse shrinkage intensity.

\begin{proof}[Proof of Lemma~\ref{lem:a}] 
Assertion~\ref{lem:a-3} follows directly from the proof of Proposition~\ref{prop:opt_vec_solution} and is thus not repeated.

Next, we prove assertion~\ref{lem:a-1}.
Recall from Assumption~\ref{assu:d_convex} that~$d_b$ is twice continuously differentiable on~$\R_{++}$. Thus, the function $H(\gamma, a) = 2a + \gamma d_b' (a)$ is continuously differentiable on $\mathbb{R}_{++}^2$.
Assumption~\ref{assu:d_convex} further stipulates that~$d_b$ is convex on~$[0,b]$. Hence, $H(\gamma,a)$ is strictly increasing in~$a$ in the sense that 
\[
    \frac{\partial H(\gamma,a)}{\partial a} = 2 + \gamma d_{b}''(a)\ge 2 > 0 \quad \forall a\in (0,b]. 
\]
As~$s_b(\gamma)\in (0,b)$ by assertion~\ref{lem:a-3}, the implicit function theorem ensures that~$s_b(\gamma)$ is differentiable (and in particular continuous) at any~$\gamma > 0$. It remains to be shown that~$s_b(\gamma)$ is continuous at~$0$. Given any~$\epsilon > 0$ and as~$s_b(0)=0$ by definition, we thus need to show that there is $\delta > 0$ such that $s_b(\gamma) \le \epsilon$ for all~$\gamma\in[0,\delta]$. As $s_b (\gamma) \in (0,b)$ for all $\gamma, b > 0$, we may assume without loss of generality that~$\epsilon \in (0,b)$. By Lemma~\ref{lem:negative_subgradient}, we have~$d_b' (\epsilon) < 0$, which guarantees that~$\delta = -2\epsilon / d_b' (\epsilon)$ is positive. For any~$\gamma\in[0,\delta]$, we thus obtain 
\begin{align*}
    s_b(\gamma) = -\frac{\gamma d_b'(s_b(\gamma))}{2} \le \frac{\epsilon d_b'(s_b(\gamma))}{d_b' (\epsilon)},
\end{align*}
where the equality follows from the definition of~$s_b$ in~\eqref{eq:s}, and the inequality follows from the definition of~$\delta$. This confirms that~$s_b(\gamma)\leq \varepsilon$. Suppose to the contrary that~$s_b (\gamma) > \epsilon$. Then the above inequality implies $ d_b'(s_b(\gamma)) < d_b' (\epsilon)$. As~$d_b'$ is non-decreasing by virtue of the convexity of~$d_b$, this in turn leads to the contradiction~$s_b(\gamma)> \varepsilon$. Thus, $s_b(\gamma)\leq\varepsilon$ for all $\gamma\in[0,\delta]$. We conclude that~$s_b(\gamma)$ is indeed continuous at~$0$.

To show that~$s_b (\gamma)$ is strictly increasing on~$\mathbb{R}_{++}$, recall that $s_b (\gamma)$ is differentiable on $\mathbb{R}_{++}$. We may thus differentiate both sides of the equation $ 0 = 2s_b(\gamma) + \gamma d_b' (s_b(\gamma))$ with respect to $\gamma$ to obtain
\[
    0 = 2 s_b'(\gamma) + d_b'(s_b(\gamma) ) + \gamma d_b'' (s_b(\gamma)) s_b'(\gamma). 
\]
Rearranging terms then yields
\begin{equation}
    \label{eq:a'}
    s_b'(\gamma) = - \frac{d_b'(s_b(\gamma))}{ 2+ \gamma d_b''(s_b(\gamma)) }, 
\end{equation}
which is strictly positive because~$d_b'(s_b(\gamma))<0$ thanks to~Lemma~\ref{lem:negative_subgradient} and $d_b''(s_b(\gamma))\geq 0$ thanks to the convexity of~$d_b$ on~$[0,b]$. Hence, $s_b(\gamma)$ is strictly increasing on $\R_+$. This completes the proof of assertion~\ref{lem:a-1}.

It remains to prove assertion~\ref{lem:a-2}. The continuity of~$s_b(\gamma)$ at~$\gamma=0$ has already been established in assertion~\ref{lem:a-1}. As~$s_b(\gamma)\in(0,b)$ is strictly increasing in~$\gamma$, it is clear that, as~$\gamma$ tends to infinity, $s_b(\gamma)$ has a well-defined limit not larger than~$b$. By the definition of~$s_b$ in~\eqref{eq:s}, we further have
\[ 
    \frac{2s_b(\gamma)}{\gamma } + d_b'(s_b(\gamma)) = 0 \quad \forall \gamma>0.
\]
Driving $\gamma$ to infinity and recalling that $s_b (\gamma) \in (0,b)$ for all $\gamma>0$ thus shows that
\[ 
    0= \lim_{\gamma \to \infty} d_b'( s_b(\gamma) ) = d_b' \left( \lim_{\gamma \to \infty } s_b(\gamma)  \right),
\]
where the second equality follows from the continuity of~$d_b'$ on~$\R_{++}$. Note that $\lim_{\gamma \to \infty } s_b(\gamma)
$ exists and falls within the interval $(0,b]$ because~$s_b$ is a strictly increasing function mapping~$\mathbb{R}_+$ to~$(0,b)$. 
These arguments imply that the limit must be a root of~$d'_b$ within~$(0,b]$. Lemma~\ref{lem:negative_subgradient} implies that~$d_b'$ has {\em no} root in the open interval~$(0,b)$. We may thus conclude that $\lim_{\gamma\to \infty} s_b (\gamma)$ must coincide with~$b$. As a sanity check, one readily verifies that $0 = d_b'(b) $ because $d_b(a)$ attains its minimum of~$0$ at~$a=b$. Thus, assertion~\ref{lem:a-2} follows.
\end{proof}

We now prove that the function $F(\gamma) = \sum_{i = 1}^p d(s(\gamma,\xsa_i), \xsa_i) -\eps $ has one and only one root. By the proof of Proposition~\ref{prop:opt_vec_solution}, this root must coincide with the unique optimal solution~$\gamma\opt$ of the problem dual to~\eqref{eq:vector}.
\begin{lemma}
\label{lem:F_root}
If Assumptions~\ref{assu:D_form}, \ref{assu:data} and~\ref{assu:d_convex} hold, then the equation $F(\gamma) = 0$ has a unique root, which is positive.
\end{lemma}

\begin{proof}[Proof of Lemma~\ref{lem:F_root}]
Recall that~$s(\gamma, 0) = 0$ by the definition of~$s$ in~\eqref{eq:s}. Recall also that if $\xsa_i = 0$, then $d(s(\gamma,\xsa_i), \xsa_i) = d(0,0) = 0$ by virtue of Assumptions~\ref{assu:D_form} and~\ref{assu:data}\ref{assu:data-d}. Therefore, vanishing components of~$\xsa$ do not contribute to the function $F(\gamma)$. In addition,  Assumption~\ref{assu:data}\ref{assu:data-eps} ensures that there exists at least one $i\in \{1,\dots,p \}$ with~$x\opt_i > 0$ and hence also with~$\xsa_i > 0$. For these reasons, we henceforth assume without loss of generality that~$\xsa_i>0$ for all~$i=1,\ldots,p$.
By Lemma~\ref{lem:a}\ref{lem:a-1}, $ s(\gamma, \xsa_i)$ constitutes a continuous real-valued function of~$\gamma\in\mathbb{R}_+$. Similarly, by Assumption~\ref{assu:D_form}\ref{assu:D_form_iii}, $d(x_i, \xsa_i)$ constitutes a continuous extended real-valued function of~$x_i\in\mathbb{R}_+$. Therefore, the extended real-valued function~$F(\gamma)$ is continuous on~$\mathbb R_+$. Assumption~\ref{assu:data}\ref{assu:data-eps} implies that $F(0) = \sum_{i=1}^p d( 0, \xsa_i ) -\eps >0$. Recall now from Lemma~\ref{lem:a}\ref{lem:a-2} that~$s(\gamma,\xsa_i)$ converges to~$\xsa_i$ as~$\gamma$ tends to infinity. By the continuity of $d(x_i, \xsa_i)$ in~$x_i$ we thus have
\[ 
    \lim_{\gamma \to\infty} F(\gamma) = \sum_{i=1}^p d(\xsa_i,\xsa_i) -\eps=-\eps <0 .
\]
All of this implies that the equation $F(\gamma) = 0$ has at least one positive root. In the remainder we prove that this root is unique. As~$\xsa_i > 0$, Lemma~\ref{lem:a} implies that $s(\gamma, \xsa_i)$ strictly increases from $0$ (at $\gamma=0$) to $\xsa_i$ (as~$\gamma$ tends to infinity). Lemma~\ref{lem:negative_subgradient} further implies that~$d_{\xsa_i}$ is strictly decreasing on~$[0,\xsa_i]$. 
Thus, the composite function $d( s(\gamma, \xsa_i) ,\xsa_i)$ is strictly decreasing in~$\gamma$ for every~$i$. This readily shows that~$F(\gamma)$ is strictly decreasing in~$\gamma$ throughout~$\R_+$, thus implying that the equation~$F(\gamma) =  0$ has only one root.
\end{proof}

We are now ready to prove Proposition~\ref{prop:gamma-reconstruction}.
\begin{proof}[Proof of Proposition~\ref{prop:gamma-reconstruction}]
The proof is a direct consequence of Propositions~\ref{prop:opt_sol_existence} and~\ref{prop:opt_vec_solution} and Lemmas~\ref{lem:a} and~\ref{lem:F_root}.
\end{proof}

\section{Proofs of Section~\ref{sec:property}} \label{sec:app-property}
\begin{proof}[Proof of Proposition~\ref{prop:nonlinear_equation_gamma_opt}]
In view of the proof of Lemma~\ref{lem:F_root}, it only remains to be shown that~$F(\gamma)$ is differentiable at any $\gamma > 0$. Towards that end, recall that vanishing components of~$\xsa$ do not contribute to~$F(\gamma)$ such that
\[
    F(\gamma) = \sum_{i = 1}^p d(s(\gamma,\xsa_i), \xsa_i) -\eps = \sum^p_{\substack{i=1: \\ \xsa_i > 0}} d(s(\gamma,\xsa_i), \xsa_i) -\eps.
\] 
For any fixed~$\xsa_i > 0$, $s(\gamma, \xsa_i)$ is differentiable with respect to~$\gamma\in\mathbb{R}_{++}$ by Lemma~\ref{lem:a}\ref{lem:a-1}, and~$d(x, \xsa_i)$ is differentiable with respect to~$x\in\mathbb{R}_{++}$ by Assumption~\ref{assu:d_convex}. Therefore, $F(\gamma)$ is differentiable at any $\gamma > 0$. 
\end{proof}

From the proof of Proposition~\ref{prop:opt_vec_solution} we know that the problem dual to~\eqref{eq:vector} has a unique optimal solution~$\gamma\opt$.
Thus, $\gamma\opt$ can be viewed as a function~$\gamma\opt (\eps)$ of the radius~$\eps >0$ of the divergence ball~\eqref{eq:uncertainty-set}.

\begin{lemma}[Monotonicity of $\gamma\opt$]
\label{lem:gamma_star}
If Assumptions~\ref{assu:D_form}, \ref{assu:data} and~\ref{assu:d_convex} hold, then $\gamma\opt (\eps)$ is non-increasing on $(0, \bar{\eps})$.
\end{lemma}

\begin{proof}[Proof of Lemma~\ref{lem:gamma_star}]
The proof of Proposition~\ref{prop:opt_vec_solution} implies that~$\gamma\opt (\eps)$ is the unique maximizer of the problem dual to~\eqref{eq:vector}. By inverting its objective function, this problem can be recast as the minimization problem 
\begin{equation}
    \label{opt:gamma}
    \min_{ \gamma > 0 } ~ \eps \gamma  + G (\gamma),
\end{equation}
where the function $G:\mathbb R_{++}\rightarrow \overline{\mathbb R}$ is defined through
\[
    G (\gamma) = - \sum^p_{\substack{i=1: \\ \xsa_i > 0}} \min_{ x_i \in [0,\xsa_i]}\left\lbrace  x_i^2 + \gamma d(x_i, \xsa_i) \right\rbrace  = -\sum^p_{\substack{i=1: \\ \xsa_i > 0}} \left( \left(s_{\xsa_i} (\gamma)\right)^2 + \gamma d_{\xsa_i} (s_{\xsa_i} (\gamma)) \right).
\]
Note also that the non-negativity constraint on~$\gamma$ in~\eqref{opt:gamma} is strict because~$\gamma=0$ cannot be optimal, or, dually, because the constraint in~\eqref{eq:vector} must be binding at optimality for $\eps<\bar{\eps}$.
By construction, $G(\gamma)$ constitutes a pointwise maximum of multiple linear functions and is, therefore, convex. Next, select $\eps_1,\eps_2 \in (0, \bar{\eps}]$ with $0<\eps_1<\eps_2$, and introduce the notational shorthands $\gamma_1 = \gamma\opt (\eps_1)$ and $\gamma_2 = \gamma\opt (\eps_2)$. By the optimality of~$\gamma_1$ and~$\gamma_2$ in problem~\eqref{opt:gamma} at~$\eps_1$ and~$\eps_2$, there exist subgradients $g_1\in \partial G (\gamma_1)$ and $g_2\in \partial G(\gamma_2)$ satisfying the first-order optimality conditions $\eps_1 + g_1=0$ and $\eps_2 + g_2=0$, respectively. Since $G(\gamma)$ is convex, its subdifferential is monotone, whereby $(\gamma_2 - \gamma_1) (g_2 - g_1) \ge 0$. Together with the first-order optimality conditions, this implies that $(\gamma_2 - \gamma_1) (\eps_1 - \eps_2) \ge 0$. As $\eps_1 < \eps_2$, we may thus conclude that $\gamma_2 \le \gamma_1$. Hence, the claim follows. 
\end{proof}

\begin{proof}[Proof of Proposition~\ref{prop:eigen_decreasing}]
Note that $x\opt_i (\eps) = s( \gamma\opt(\eps), \xsa_i)$ for every~$\eps\in(0,\bar\eps)$ thanks to Proposition~\ref{prop:gamma-reconstruction}, and recall that~$x\opt_i(\bar\eps)=0$ by definition. We aim to show that $x_i\opt (\eps)$ is non-increasing on $[0,\bar{\eps}]$ and that $\lim_{\eps\uparrow \bar{\eps}} x_i\opt (\eps) = 0$. To this end, note first that both claims are trivially satisfied if~$\xsa_i = 0$, in which case $x_i\opt (\eps) = 0$ for all $\eps \in (0,\bar\eps)$ thanks to Proposition~\ref{prop:gamma-reconstruction} and our conventions that~$x_i\opt (0) = \xsa_i$ and $x_i\opt (\bar\eps) = 0$. Assume next that $\xsa_i > 0$. Recall that $\gamma\opt(\eps)$ is non-increasing on~$(0,\bar{\eps})$ thanks to Lemma~\ref{lem:gamma_star}, while $s_{\xsa_i}(x_i)=s(x_i , \xsa_i)$ is strictly increasing on~$\mathbb{R}_{+}$ thanks to Lemma~\ref{lem:a}\ref{lem:a-1}, which applies because~$\xsa_i>0$. Therefore, $x\opt_i (\eps) = s( \gamma\opt(\eps), \xsa_i)$ is non-increasing on $(0,\bar{\eps})$. We also have $x_i\opt (\eps) \in (0,\xsa_i)$ for all~$\eps \in (0,\bar{\eps})$ thanks to Proposition~\ref{prop:gamma-reconstruction}, and we have $x_i\opt (0) = \xsa_i$ and $x_i\opt(\bar{\eps}) = 0$ by definition. All of this readily implies that~$x\opt_i (\eps)$ is non-increasing on~$[0,\bar{\eps}]$. In order to prove that $\lim_{\eps\uparrow \bar{\eps}} x\opt_i (\eps) = 0$, note first that $\lim_{\eps\uparrow \bar{\eps}} x\opt_i (\eps) $ must exist because $x\opt_i (\eps)$ is non-negative as well as non-increasing in~$\eps$. Next, recall from Lemma~\ref{lem:negative_subgradient} that the function $d_{\xsa_i}(x_i)=d(x_i, \xsa_i)$ is strictly decreasing on~$(0, \xsa_i)$. In fact, this monotonicity property extends to~$[0, \xsa_i]$ because~$d_{\xsa_i}$ is continuous thanks to Assumption~\ref{assu:D_form}\ref{assu:D_form_iii}. We then choose an arbitrary tolerance~$\delta > 0$ and assume without loss of generality that $\delta$ is smaller than the smallest non-vanishing component of~$\xsa$. Next, consider a vector~$x\in\mathbb R_+^p$ defined through~$x_i=0$ if~$\xsa_i = 0$ and~$x_i=\delta$ if~$\xsa_i > 0$, $i=1,\ldots, p$, and set~$\eps = \sum_{i = 1}^p d(x_i, \xsa_i)$. By construction, we have
\[ 
    \eps = \sum_{i = 1}^p d(x_i, \xsa_i) < \sum_{i = 1}^p d(0, \xsa_i) = \bar{\eps}, 
\]
where the strict inequality holds because~$\xsa$ has at least one strictly positive component and because $d(x_i,\xsa_i)<d(0,\xsa_i)$ whenever~$\xsa_i>0$ thanks to the monotonicity properties of~$d$ established above. Hence, $x$ is feasible in~\ref{eq:vector}, and~$\eps$ is consistent with Assumption~\ref{assu:data}\ref{assu:data-eps}. In addition, one readily verifies that the objective function value of~$x$ satisfies~$\|x\|_2^2 \le p \delta^2$. By the optimality of~$x\opt (\eps)$ in~\ref{eq:vector}, we thus find 
\[ 
    x_i\opt (\eps)^2\le \|x\opt (\eps)\|_2^2 \le p\delta^2 \quad \forall i = 1,\dots, p.
\]
Thus, for any sufficiently small~$\delta>0$ there exists~$\varepsilon>0$ with $x_i\opt (\eps)\leq \sqrt{p} \delta$. As~$x_i\opt (\eps)$ is non-increasing on~$[0,\bar{\eps}]$, this implies indeed that $\lim_{\eps \uparrow \bar{\eps}} x_i\opt (\eps) = 0$. It remains to be shown that~$X\opt$ constitutes a shrinkage estimator. This is now evident, however, because $\covsa = \Vsa \Diag(\xsa)\Vsa^\top = \Vsa \Diag(x\opt(0))\Vsa^\top$.
\end{proof}

\begin{proof}[Proof of Lemma~\ref{lemma:cond_decrease}]
Throughout this proof we fix any~$\gamma > 0$. We first aim to show that the function 
\[
    K(b)= \frac{1}{b} \frac{\partial d(s(\gamma, b), b)}{\partial a} 
\]
is non-decreasing on~$\mathbb{R}_{++}$. 
To this end, note that~$d(a, b)$ is twice continuously differentiable on~$\mathbb{R}^2_{++}$ by Assumption~\ref{assu:d_b_second_derivative}. Using the implicit function theorem as in Lemma~\ref{lem:a}, one can thus show that~$s(\gamma,b)$ is differentiable with respect to~$b$ and that~$s(\gamma, b) \in (0,b)$ for every~$b>0$. Recall also that $-\frac{2}{\gamma}s(\gamma, b) = \frac{\partial d}{\partial a} (s(\gamma, b),b)$ by the definition of~$s$ in~\eqref{eq:s}.
\begin{align}
    \label{eq:1}
    -\frac{2}{\gamma} \frac{\partial s(\gamma, b)}{\partial b} = \frac{\rm d}{{\rm d} b} \left( \frac{\partial d( s(\gamma, b) ,b)}{\partial a} \right) = \frac{\partial^2 d(s(\gamma, b), b)}{\partial a \partial b} + \frac{\partial^2 d(s(\gamma, b), b)}{\partial a^2} \frac{\partial s(\gamma, b)}{\partial b}.
\end{align}
This in turn implies that
\begin{equation}
    \label{eq:partial_a_partial_b}
    \frac{\partial s(\gamma, b)}{\partial b} = - \left( \frac{2}{\gamma} + \frac{\partial^2 d(s(\gamma, b), b)}{\partial a^2} \right)^{-1} \frac{\partial^2 d(s(\gamma, b), b)}{\partial a \partial b},
\end{equation}
which is well-defined because~$\gamma>0$ and~$d(\cdot,b)$ is convex by Assumption~\ref{assu:d_convex}. We then find
\begin{align*}
    \frac{{\rm d}K(b)}{{\rm d}b} & = -\frac{1}{b^2} \frac{\partial d(s(\gamma, b), b)}{\partial a} + \frac{1}{b} \frac{\rm d}{{\rm d} b}\left( \frac{\partial d (s(\gamma, b), b)}{\partial a} \right) .
\end{align*}
The second term on the right hand side of the above expression satisfies
\begin{align*}
    \frac{1}{b} \frac{\rm d}{{\rm d} b}\left( \frac{\partial d (s(\gamma, b), b)}{\partial a} \right) & = -\frac{2}{b\gamma} \frac{\partial s(\gamma, b)}{\partial b} = \frac{2}{b \gamma } \left( \frac{2}{\gamma} + \frac{\partial^2 d(s(\gamma, b), b)}{\partial a^2} \right)^{-1} \frac{\partial^2 d(s(\gamma, b), b)}{\partial a \partial b} \\
    & = \frac{2}{b  } \Bigg( \frac{ \frac{\partial^2 d(s(\gamma, b), b)}{\partial a \partial b}}{2 - \frac{2s(\gamma, b)}{\frac{\partial d(s(\gamma, b),b)}{\partial a}} \frac{\partial^2 d(s(\gamma, b), b)}{\partial a^2}}  \Bigg) = \frac{1}{b}  \Bigg(\frac{ \frac{\partial d(s(\gamma, b),b)}{\partial a} \frac{\partial^2 d(s(\gamma, b), b)}{\partial a \partial b}}{\frac{\partial d(s(\gamma, b),b)}{\partial a} - s(\gamma, b) \frac{\partial^2 d(s(\gamma, b), b)}{\partial a^2}} \Bigg),
\end{align*}
where the first and the second equalities follow from~\eqref{eq:1} and~\eqref{eq:partial_a_partial_b}, respectively, and the third equality follows from the defining equation of~$s$ in~\eqref{eq:s}. Combining the last two equations finally yields
\begin{align*}
    \frac{{\rm d}K(b)}{{\rm d}b} & = -\frac{1}{b^2} \frac{\partial d(s(\gamma, b), b)}{\partial a} \left( 1 -\frac{ b\, \frac{\partial^2 d(s(\gamma, b), b)}{\partial a \partial b}}{\frac{\partial d(s(\gamma, b),b)}{\partial a} - s(\gamma, b) \frac{\partial^2 d(s(\gamma, b), b)}{\partial a^2}}  \right) .
\end{align*}
Recall now that~$\frac{\partial d(a,b)}{\partial a} < 0$ for every~$a\in(0,b)$ thanks to Lemma~\ref{lem:negative_subgradient} and that $s(\gamma,b)\in (0,b)$ thanks to Lemma~\ref{lem:a}.
This implies that the derivative of~$K(b)$ is non-negative if and only if
\begin{equation}\label{ineq:1}
    \frac{\partial^2 d(s(\gamma, b), b)}{\partial a \partial b}  \ge \frac{1}{b} \left( \frac{\partial d(s(\gamma, b), b)}{\partial a}  - s(\gamma, b) \frac{\partial^2 d(s(\gamma, b), b)}{\partial a^2} \right).
\end{equation}
Assumption~\ref{assu:d_b_second_derivative} guarantees that~\eqref{ineq:1} holds indeed for all~$b>0$. Hence, $K(b)$ is a non-decreasing function.

We now prove the desired inequality.
By the defining equation of~$s$ in~\eqref{eq:s} we have
\begin{align*}
    -2\gamma\, b_1 \,s(\gamma, b_2) =  b_1 \frac{\partial d(s(\gamma, b_2), b_2)}{\partial a} \ge b_2 \frac{\partial d(s(\gamma, b_1), b_1)}{\partial a} = -2\gamma\, b_2\, s(\gamma, b_1)
\end{align*}
for any~$b_2 \ge b_1>0$, where and inequality follows from the monotonicity of~$K$ established above. This implies that $s(\gamma, b_2)/s(\gamma, b_1)  \le b_2/b_1$ for all $b_1,b_2\in \R_{++}$ with $b_2 \ge b_1$. Hence, the claim follows.
\end{proof}

\begin{proof}[Proof of Proposition~\ref{prop:consistency}]
Throughout the proof we use the shorthands~$x\opt_{i,n} = \lambda_i (X\opt_n)$ and~$\widehat{x}_{i,n} = \lambda_i (\covsa_n)$ for all $i = 1,\dots,p$ and $n\in\mathbb N$. By the strong consistency assumption, $\covsa_n$ converges almost surely to~$\cov_0$. Fix now temporarily a particular realization of the uncertainties, for which~$\covsa_n$ converges deterministically to~$\cov_0$. In this case, $\widehat{x}_{i,n}$ converges to~$\lambda_i (\cov_0)$ because the eigenvalue map $\lambda_i$ is continuous~\cite[Corollary~VI.1.6]{bhatia1997matrix}, and the sequence~$\{x\opt_{i,n}\}_{n\in\mathbb N}$ is bounded by Lemma~\ref{lem:opt_vector_sol_upper_bound}. Thus, any convergent subsequence $\{x\opt_{i,n_k}\}_{k\in\mathbb N}$ satisfies
\[
    \lim_{k\to \infty} x\opt_{i, n_k}\in [0, \lim_{k\to \infty} \xsa_{i,n_k}]= [0, \lim_{n\to \infty} \xsa_{i,n}] = [0, \lambda_i (\cov_0)] .
\]
In addition, we have
\[
    d( x\opt_{i,n_k}, \widehat{x}_{i,n_k} ) \leq \sum_{j=1}^p d(x\opt_{j,n_k}, \xsa_{j,n_k})  
    = D\left(X\opt_{n_k}, \covsa_{n_k}\right) \le \eps_{n_k} \quad \forall k\in\mathbb N,
\]
where the first equality holds because of Assumptions~\ref{assu:D_form}\ref{assu:D_form_ii} and~\ref{assu:D_form}\ref{assu:D_form_iii} and because~$X\opt_{n_k}$ and~$\covsa_{n_k}$ share the same eigenvectors. The second inequality follows from Proposition~\ref{prop:exist}\ref{prop:exist-ii}, which ensures that~$X\opt_{n_k}$ is feasible in problem~\eqref{eq:matrix}. As~$\eps_{n_k}$ converges to~$0$ and as~$d$ is continuous on~$\mathbb{R}_+ \times \mathbb{R}_{++}$, the above implies that
\[
     d( \lim_{k \to \infty} x\opt_{i,n_k}, \lambda_i(\cov_0) ) = d( \lim_{k \to \infty} x\opt_{i,n_k}, \lim_{k \to \infty} \widehat{x}_{i,n_k} ) = \lim_{k \to \infty} d( x\opt_{i,n_k}, \widehat{x}_{i,n_k} ) = 0  .
\]
Recall now from Assumption~\ref{assu:D_form} that~$d$ satisfies the identity of indiscernibles. Thus  we find $\lim_{k \to \infty} x\opt_{i,n_k} = \lambda_i (\cov_0)$. This shows that every convergent subsequence of the bounded sequence $\{ x\opt_{i,n} \}_{n\in\mathbb N}$ must have the same limit $\lambda_i (\cov_0) $. By~\cite[Exercise~2.5.5]{abbott2015understanding}, the eigenvalue~$x\opt_{i,n}$ therefore converges to~$\lambda_i (\cov_0)$. This reasoning applies to every uncertainty realization under which~$\covsa_n$ converges to~$\cov_0$. As~$\covsa_n$ converges almost surely to~$\cov_0$, we have thus shown that~$x\opt_{i,n}$ converges almost surely to~$\lambda_i (\cov_0)$. This in turn implies that
\begin{align*}
    \p[\lim_{n\to \infty} \|X\opt_n-\cov_0\|_\mathrm{F} = 0] \ge &\, \p[ \lim_{n\to \infty} \left(\| X\opt_n - \covsa_n \|_\mathrm{F} + \| \covsa_n - \cov_0 \|_\mathrm{F}\right) = 0 ]\\
    = &\, \p[ \lim_{n\to \infty} \left(\| x\opt_n - \xsa_n \|_2 + \| \covsa_n - \cov_0 \|_\mathrm{F}\right) = 0 ] \\
    \ge &\, \p[ \lim_{n\to \infty} \left(\| x\opt_n - \lambda(\cov_0) \|_2 + \| \lambda(\cov_0) - \xsa_n \|_2 + \| \covsa_n - \cov_0 \|_\mathrm{F}\right) = 0 ] = 1,
\end{align*}
where both inequalities hold thanks to the triangle inequality, the first equality follows from Theorem~\ref{thm:general_CSE}, which ensures that~$X\opt_n$ and~$\covsa_n$ share the same eigenvectors, and the second equality exploits the almost sure convergence of~$x\opt_n$ and~$\xsa_n$ to~$\lambda(\cov_0)$ established above and the almost sure convergence of~$\covsa_n$ to~$\cov_0$.
This shows that $X\opt_n$ converges almost surely to $\cov_0$ and therefore completes the proof.
\end{proof}

From now on we use $\|X\|_*$ to denote the nuclear norm of $X\in\mathbb S^p$ (\ie, the sum of all singular values of~$X$), which is the norm dual to the { spectral norm} $\|X\|$ (\ie, the largest singular value of $X$). The proof of Proposition~\ref{proposition:finite-sample-guarantees} relies on the following well-known result from high-dimensional statistics. 

\begin{lemma}[{\cite[Theorem~6.5]{wainwright2019high}}]
    \label{lem:operator_norm_tail}
    Under the assumptions of Proposition~\ref{proposition:finite-sample-guarantees}, there exists {a universal constant~$c_0 >0$ independent of~$\p$ such that}
    \begin{equation*}
        \p^n\left[ \|\covsa_n - \cov_0\| \le { \rho( p, n, \eta )} \right] \ge 1-\frac{\eta}{2}
    \end{equation*}
    for every $n \in\mathbb N$ and $\eta \in(0,1)$, where 
    \[ {\rho( p, n, \eta ) = c_0 \sigma^2 \left(  \frac{p + \log \eta^{-1}}{n} + \sqrt{\frac{p + \log \eta^{-1}}{n}}  \right) } .\]
\end{lemma}

\begin{proof}[Proof of Proposition~\ref{proposition:finite-sample-guarantees}]
    For any divergence function~$D$ from Table~\ref{table:structured_divergence} we will prove that there exist a constant~$c > 0$ and a function {$n_{\rm min}(p, \eta)=\mathcal O( p+ \log\eta^{-1})$ that may depend on~$\p$ via~$\sigma^2$ and~$\lambda_1(\cov_0)$} such~that
    \begin{equation}
        \label{ineq:divergence_operator_norm}
        \p^n \left[ D(\cov_0, \covsa_n) \le c \|\cov_0 - \covsa_n\| \right] \ge 1-\frac{\eta}{2}
    \end{equation}
    for all {$n\ge n_{\rm min}(p, \eta)$ and $\eta\in(0,1)$. Indeed, if such an inequality holds, then Lemma~\ref{lem:operator_norm_tail} and the union bound imply that $ \p^n[ D(\cov_0, \covsa_n) \le c \rho(p, n,  \eta)] \ge 1-\eta$. The claim then follows by setting $\eps_{\min} (p, n, \eta) = c\rho(p,  n, \eta )$.}

\textbf{Stein, Inverse Stein and Symmetrized Stein Divergences:} 
Note that the sum of the Stein and inverse Stein divergences equals twice the symmetrized Stein divergence. Recall also that all divergences are non-negative. Thus, if the ball of radius~$\varepsilon$ with respect to the symmetrized Stein divergence contains~$\cov_0$ with probability at least~$1-\eta$, then the ball of radius~$2\eps$ with respect to the Stein or inverse Stein divergence contains~$\cov_0$ with probability at least~$1-\eta$. It thus suffices to focus on the symmetrized Stein divergence. Suppose now that the smallest eigenvalue of~$\covsa_n$ is no smaller than half of the smallest eigenvalue of~$\cov_0$. As~$\cov_0\succ 0$, this implies in particular that $\covsa_n$ is positive definite and that $\covsa_n^{-1}$ exists. Rewriting the symmetrized Stein divergence as $\frac{1}{2}\Tr{(\cov_0^{-1}-\covsa_n^{-1})(\covsa_n-\cov_0)}$, we may then use the matrix H\"older's inequality to obtain
\[
    \Tr{(\cov_0^{-1}-\covsa_n^{-1})(\covsa_n-\cov_0)}\le \|\cov_0-\covsa_n\| \|\cov_0^{-1}-\covsa_n^{-1}\|_*.
\]
In the following we use $x_i = \lambda_i (\cov_0)$ and $\xsa_{i,n} = \lambda_i (\covsa_n)$ to denote $i$-th smallest population and sample eigenvalues for $i=1,\dots,p$, respectively. By the definitions of the nuclear and {spectral norms}, we then have
\begin{align*}
    \|\cov_0^{-1}-\covsa_n^{-1}\|_* \le &\, p \|\cov_0^{-1}-\covsa_n^{-1}\|\\
    = &\, p \max\left\{ \lambda_p(\cov_0^{-1}-\covsa_n^{-1}), - \lambda_1 (\cov_0^{-1}-\covsa_n^{-1}) \right\} \\
    \le &\, p \max\left\{ \lambda_p(\cov_0^{-1})-\lambda_1(\covsa_n^{-1}), \lambda_p(\covsa_n^{-1}) - \lambda_1 (\cov_0^{-1}) \right\} \\
    = &\, p \max\left\{ \frac{1}{x_1} - \frac{1}{\xsa_{p,n}}, \frac{1}{\xsa_{1,n}} - \frac{1}{x_p} \right\} \\
    \le &\, p \max\left\{ \frac{1}{x_1} , \frac{1}{\xsa_{1,n}}  \right\} \le  \frac{2p}{x_1},
\end{align*}
where the first equality holds because the singular values of a symmetric matrix coincide with the absolute values of the eigenvalues of that matrix. The second inequality follows from a classical result by Weyl, which asserts that $\lambda_1 (A + B) \le \lambda_1 (A) + \lambda_p(B) \le \lambda_p (A+B)$ for any~$A,B\in \Sym^p$, and the second equality holds because $\lambda_i (A^{-1}) = 1/\lambda_{p-i+1} (A)$ for any $i = 1,\dots, p$ and~$A\in \Sym^p_{++}$. The third inquality exploits our assumption that all population and sample eigenvalues are strictly positive, and the last inequality follows from the assumption that $\xsa_{1,n} \ge x_1/2$. We have thus shown that if $\xsa_{1,n} \ge x_1/2$, then $D(\cov_0, \covsa_n) \le \frac{p}{x_1} \|\cov_0 - \covsa_n\|$. Hence, we find
\[ 
    \p^n\Big[ D(\cov_0, \covsa_n) \le \frac{p}{x_1} \|\cov_0 - \covsa_n\| \Big]\ge \p^n \Big[  \xsa_{1,n} \ge \frac{x_1}{2} \Big] .
\]
As $\xsa_{1,n} \ge x_1 - \|\cov_0 - \covsa_n\|$ by virtue of Weyl's inequality and by Lemma~\ref{lem:operator_norm_tail}, the last probability satisfies
\begin{equation}
    \label{ineq:sample_min_eig_bound}
    \p^n\Big[ \xsa_{1,n} \ge \frac{x_1}{2} \Big] \ge \p^n\Big[ \|\cov_0 - \covsa_n\| \le \frac{x_1}{2} \Big] \ge \p^n\Big[ \|\cov_0 - \covsa_n\| \le {\rho(p, n, \eta)} \Big] \ge 1 - \frac{\eta}{2}
\end{equation}
whenever $x_1/2 \ge {\rho(p, n, \eta)}$. 
By the definition of~${\rho(p, n, \eta)}$, a sufficient condition for this inequality to hold is
\begin{equation*}
    n \ge { n_{\min} (p, \eta) = \max\left\{1, \frac{16c_0^2 \sigma^4}{x_1^2}\right\} (p+\log \eta^{-1}) }.
\end{equation*}
The above estimates imply that~\eqref{ineq:divergence_operator_norm} holds for all $n \ge { n_{\min} (p, \eta)}$ and $\eta\in(0,1)$ if we set $c=p/x_1$. In addition, the minimal sample size and the minimal radius of the uncertainty set satisfy ${ n_{\min} (p, \eta) = \mathcal{O}(p + \log \eta^{-1})}$ and
\[
    { \eps_{\min}(p, n,\eta) = c \rho(p, n, \eta) = \frac{p c_0 \sigma^2}{x_1}   \left(  \frac{p + \log \eta^{-1}}{n} + \sqrt{\frac{p + \log \eta^{-1}}{n}}  \right) = \mathcal{O}(p n^{-\frac{1}{2}} (p+ \log \eta^{-1})^{\frac{1}{2}})},
\]
where the last equality holds because $n \ge p+ \log \eta^{-1}$. This establishes the claim for the Stein, the inverse Stein and the symmetrized Stein divergences.

\textbf{Wasserstein Divergence:} From the proof of~\cite[Theorem~4]{ref:nguyen2021mean} we know that if $\xsa_{1,n} \ge \frac{x_1}{2}$, then
\[
    D(\cov_0, \covsa_n) \le \frac{1}{(\xsa_{1,n} + x_1)^2} \| \cov_0 - \covsa_n \|_{\mathrm{F}}^2 \le \frac{p}{(\xsa_{1,n} + x_1)^2} \| \cov_0 - \covsa_n \|^2 \le \frac{4p}{ 9 x_1^2} \| \cov_0 - \covsa_n \|^2.
\]
We also know from~\eqref{ineq:sample_min_eig_bound} that $\p^n[ \xsa_{1,n} \ge \frac{x_1}{2}] \ge 1 - \frac{\eta}{2}$ for all $n \ge {\mathcal{O}(p+\log\eta^{-1})}$. Thus, we have
\begin{equation}
    \label{eq:D-vs-squared-norm}
    \p^n \left[ D(\cov_0, \covsa_n) \le \frac{4p}{ 9 x_1^2} \| \cov_0 - \covsa_n \|^2 \right] \ge \p^n\left[ \xsa_{1,n} \ge \frac{x_1}{2} \right] \ge 1 - \frac{\eta}{2}
\end{equation}
for all $n \ge{\mathcal{O}(p+\log\eta^{-1})}$. Lemma~\ref{lem:operator_norm_tail} further implies that
\begin{equation}
    \label{ineq:operator_norm_1}
    \p^n \left[ \|\cov_0 - \covsa_n\| \le 1 \right] \ge \p^n \left[ \|\cov_0 - \covsa_n\| \le {\rho(p, n,  \eta)} \right] \ge 1 - \frac{\eta}{2},
\end{equation}
whenever 
\[
    1 \ge {\rho(p, n,  \eta) =  c_0 \sigma^2 \left(    \frac{p + \log \eta^{-1}}{n} + \sqrt{\frac{p + \log \eta^{-1}}{n}}  \right) } .
\]
A sufficient condition for this inequality to hold is that ${ n\ge \mathcal{O}(p + \log \eta^{-1})}$. Combining~\eqref{eq:D-vs-squared-norm} and~\eqref{ineq:operator_norm_1} and using the union bound implies that there is a function ${n_{\min} (p, \eta)}$ that grows at most as ${\mathcal{O}(p+\log \eta^{-1})}$ with
\[
    \p^n\left[ D(\cov_0, \covsa_n) \le \frac{4p}{9x_1^2} \|\cov_0 - \covsa_n\| \right]\ge 1-\eta
\]
for all ${n \ge n_{\min} (p, \eta)}$. Thus, \eqref{ineq:divergence_operator_norm} holds for all ${n \ge n_{\min} (p, \eta)}$ and $\eta\in(0,1)$ if we set $c=4p/(9x^2_1)$. Similar calculations as in the last part of the proof reveal that ${\eps_{\min}(p, n,\eta) = c \rho(p, n,  \eta)}$ grows at most as ${\mathcal{O}(pn^{-\frac{1}{2}} (p+\log \eta^{-1})^{\frac{1}{2}})}$. This establishes the claim for the Wasserstein divergence.

\textbf{Quadratic Divergence:} Since $\|A\|_{\rm F}\leq\sqrt{p}\|A\|$ for all $A\in\mathbb{S}^p$, we have  
\[
    D(\cov_0, \covsa_n) = \| \cov_0 - \covsa_n \|_{\rm F}^2 \le p \|\cov_0 - \covsa_n \|^2 . 
\]
From~\eqref{ineq:operator_norm_1} we already know that $\p^n[ \|\cov_0 - \covsa_n\| \le 1 ]  \ge 1 - \eta$ for all ${n\ge \mathcal{O}(p+ \log \eta^{-1})}$. Thus, there is a function ${n_{\min} (p, \eta)=\mathcal{O}(p+ \log \eta^{-1})}$ such that~\eqref{ineq:divergence_operator_norm} holds for all ${ n \ge n_{\min} (p, \eta)}$ and $\eta\in(0,1)$ if we set $c=p$. As usual, one verifies that ${\eps_{\min}(p, n,\eta) = c \rho(p, n, \eta) = \mathcal{O} (p n^{-\frac{1}{2}} (p+\log \eta^{-1})^{\frac{1}{2}})}$. This proves the claim for the quadratic divergence.

\textbf{Weighted Quadratic Divergence:} As $\Tr{AB}\le \|A\|  \|B\|_* \le p \|A\| \|B\|$ for all $A,B\in \Sym^p$, we have
\[ 
    D(\cov_0, \covsa_n) = \Tr{(\cov_0-\covsa_n)^2\covsa_n^{-1}}\le p \|(\cov_0-\covsa_n)^2\|\|\covsa_n^{-1}\|\le  \frac{p}{\xsa_{1,n}}\|\cov_0-\covsa_n\|^2 \le  \frac{2p}{x_1}\|\cov_0-\covsa_n\|^2 
\]
whenever $\xsa_{1,n} \ge \frac{x_1}{2}$. Recall also that $\covsa_n$ is indeed invertible under this assumption. 
Together with~\eqref{ineq:sample_min_eig_bound} and~\eqref{ineq:operator_norm_1}, the above inequality implies that there exists a function { $n_{\min} (p, \eta) = \mathcal{O}(p+\log \eta^{-1})$} such that
\[
    \p^n \left[ D(\cov_0, \covsa_n) \le \frac{2p}{x_1} \|\cov_0 - \covsa_n\| \right]\ge 1-\eta,
\]
for all~${n \ge n_{\min} (p, \eta)}$. Thus, \eqref{ineq:divergence_operator_norm} holds for all ${n \ge n_{\min} (p, \eta)}$ and $\eta\in(0,1)$ if we set $c=2p/x_1$. As usual, { we have $\eps_{\min}(p, n,\eta) = \mathcal{O} (p n^{-\frac{1}{2}} (p+\log \eta^{-1})^{\frac{1}{2}})$}. This proves the claim for the weighted quadratic divergence.

\textbf{Fisher-Rao Divergence:} As $\log^2 x \le x- 2 + x^{-1}$ for all~$x>0$, the Fisher-Rao divergence satisfies
\begin{align*}
D(X, Y) & = \sum_{i=1}^p \log^2 \lambda_i (XY^{-1})  \le \sum_{i=1}^p \left( \lambda_i (XY^{-1}) - 2 + \frac{1}{\lambda_i (XY^{-1})}\right)  = \Tr{XY^{-1}} - 2p + \Tr{ Y X^{-1} } 
\end{align*}
for all~$X,Y\in\Sym^p_{++}$, where the last expression equals twice the symmetrized Stein divergence of~$X$ and~$Y$. We have already shown that~\eqref{ineq:divergence_operator_norm} holds for symmetrized Stein divergence for all ${n \ge n_{\min} (p, \eta)=\mathcal{O}(p+ \log \eta^{-1})}$ and $\eta\in(0,1)$ provided that $c = \frac{p}{x_1}$. Thus, \eqref{ineq:divergence_operator_norm} must also hold for the Fisher-Rao divergence if $c = \frac{2p}{x_1}$. As usual, { we have $\eps_{\min}(p, n, \eta) = \mathcal{O}(p n^{-\frac{1}{2}} (p+ \log \eta^{-1})^{\frac{1}{2}})$}. This proves the claim for the Fisher-Rao divergence.
\end{proof}

\section{Verification of the Minimax Property}\label{app:ass-1-verification}
\begin{proposition}\label{prop:minimax_assumption}
All the divergences listed in Table~\ref{table:structured_divergence} satisfy Assumption~\ref{assu:inf_sup}.
\end{proposition}

\begin{proof}[Proof of Proposition~\ref{prop:minimax_assumption}]
Our goal is to prove the minimax equality
\begin{equation}\label{eq:minimax_proof}
    \min_{\X \in \PSD^p} \;\max_{\cov \in \B_\eps( \covsa)} \;  \Tr{X^2} - 2 \langle \cov , X \rangle
  = \max_{\cov \in \B_\eps( \covsa)}  \;\min_{\X \in \PSD^p} \;  \Tr{X^2} - 2 \langle \cov , X \rangle.
\end{equation}
If~$D$ is the Kullback-Leibler, Fisher-Rao, inverse Stein, symmetrized Stein or weighted quadratic divergence and if~$\covsa$ is singular, then $(\cov, \covsa) \not\in\mathrm{dom}(D)$ for every $\cov\in \PSD^p$. In this case, the uncertainty set $\B_\eps (\covsa) = \{\cov\in\PSD^p: D(\cov, \covsa)\le \eps\}$ is empty, and the minimax equality~\eqref{eq:minimax_proof} holds trivially because both sides of~\eqref{eq:minimax_proof} evaluate to~$\infty$. 
Thus, we may always assume that $\covsa \in \PD^p$ for these divergences.

The objective function $\Tr{X^2} - 2 \langle \cov , X \rangle$ of the minimax problem~\eqref{eq:minimax_proof} is convex and continuous in~$X$ for any fixed $\cov \in \B_\eps (\covsa)$, and it is concave and continuous in~$\cov$ for any fixed $X\in \PSD^p$. If $\B_\eps (\covsa)$ is convex and compact, then~\eqref{eq:minimax_proof} follows readily from Sion's classical minimax theorem. We will argue below that this is true for the Kullback-Leibler, Wasserstein, symmetrized Stein, quadratic, and weighted quadratic divergences. The uncertainty sets associated with the quadratic and weighted quadratic divergences constitute ellipsoids and are, therefore, trivially convex and compact. In addition, the convexity and compactness of the uncertainty set induced by the Wasserstein divergence follow from~\cite[Lemma~A.6]{ref:nguyen2019bridging}. We next show that the Kullback-Leibler and symmetrized Stein divergences also induce convex and compact uncertainty sets.

\textbf{Kullback-Leibler Divergence:} For any fixed $\covsa\in \PD^p$, the Kullback-Leibler divergence $D(\cov, \covsa)$ constitutes a continuous extended real-valued function of~$\cov$. Indeed, one can show that $D(\cov,\covsa)$ tends to infinity as~$\cov$ approaches the boundary of~$\PSD^p$ and~$\covsa\in \PD^p$ is kept fixed. Therefore, the uncertainty set~$\B_\eps (\covsa)$ is closed as a sublevel set of a continuous function. As $t - 1 - \log t\geq 0$ for every~$t>0$, any~$\cov\in \B_\eps (\covsa)$ satisfies 
\begin{align*}
        \eps \ge D(\cov, \covsa) = \frac{1}{2} \sum_{i=1}^p \left( \lambda_i (\covsa^{-1} \cov) - 1 - \log \lambda_i (\covsa^{-1} \cov) \right) \ge \frac{1}{2}\left( \lambda_p (\covsa^{-1} \cov) - 1 - \log \lambda_p (\covsa^{-1} \cov) \right).
\end{align*}
Note that the function $t - 1 - \log t$ grows indefinitely as~$t$ tends to infinity. Consequently, the above inequality implies that there exists~$\overline{\lambda} >0$ with $\lambda_p (\covsa^{-1} \cov) \le \overline{\lambda}$ for all $\cov \in\B_\eps (\covsa)$. Recall now that the { spectral norm} of any positive definite matrix coincides with its maximum eigenvalue. For any $\cov \in\B_\eps (\covsa)$ we thus have 
\[  
    \| \cov\| = \| \covsa^{\frac{1}{2}} \covsa^{-\half} \cov \covsa^{-\half}  \covsa^{\frac{1}{2}} \|\le \| \covsa^{-\half} \cov \covsa^{-\half} \| \|\covsa\|  = \lambda_p(\cov \covsa^{-1} )\lambda_p(\covsa) \le \overline{\lambda}\, \lambda_p(\covsa), 
\]
where the second equality holds because  $\| \covsa^{-\half} \cov \covsa^{-\half} \|=\lambda_p( \covsa^{-\half} \cov \covsa^{-\half})$ and because $\cov \covsa^{-1}$ has the same eigenvalues as $\covsa^{-\half} \cov \covsa^{-\half}$. This shows that $\B_\eps (\covsa)$ is bounded and thus compact. Finally, note that $D(\cov, \covsa)$ is convex in~$\cov$ because $\Tr{\covsa^{-1} \cov }$ is linear and $\log \det (\covsa \cov^{-1})$ is convex in $\cov$. Hence, $\B_\eps (\covsa)$ is convex.

\textbf{Symmetrized Stein Divergence:} For any fixed $\covsa\in \PD^p$, the symmetrized Stein divergence~$D(\cov, \covsa)$ is continuous in~$\cov$. Thus, the corresponding uncertainty set~$\B_\eps (\covsa)$ is closed. Also, any $\cov\in \B_\eps (\covsa)$ satisfies 
\begin{align*}
        \eps \ge D(\cov, \covsa) = \frac{1}{2} \sum_{i=1}^p \left( \lambda_i (\covsa^{-1} \cov) + \lambda_i^{-1} (\covsa^{-1}\cov )  - 2  \right) 
        \ge \frac{1}{2}\left( \lambda_p (\covsa^{-1} \cov) + \lambda_p^{-1} (\covsa^{-1}\cov )  - 2  \right),
\end{align*}
where the second inequality holds because all eigenvalues of $\covsa^{-1} \cov$ are positive. Note that $t + t^{-1} - 2$ grows indefinitely as~$t$ tends to infinity. Hence, there exists $\overline{\lambda} >0$ with $\lambda_p (\covsa^{-1} \cov) \le \overline{\lambda}$ for all~$\cov \in\B_\eps (\covsa)$. By using a similar reasoning as for the Kullback-Leibler divergence, we can thus show that $\B_\eps (\covsa)$ is compact. To prove convexity, we need to show that~$D(\cov, \covsa)$ is a convex function of~$\cov$. But this follows from \cite[Exercise~3.18(a)]{ref:boyd2004convex}.

The uncertainty sets induced by the Fisher-Rao and inverse Stein divergences fail to be convex in the standard Euclidean sense; see Section~\ref{subsec:non-quasi-convex}. We will show, however, that these uncertainty sets are geodesically convex with respect to a certain Riemannian geometry on the cone~$\PD^p$. This will allow us to prove the minimax equality~\eqref{eq:minimax_proof} by appealing to Theorem~\ref{thm:Sion-manifold}, which establishes a generalized version of Sion's minimax theorem for geodesic quasi-convex-quasi-concave minimax problems on Hadamard manifolds. 

In order to apply Theorem~\ref{thm:Sion-manifold}, we embed the feasible set~$\PSD^p$ of the minimization problem in~\eqref{eq:minimax_proof} into~$\Sym^p$ equipped with the usual Euclidean geometry. Recall from Example~\ref{exam:Euclidean_Hadamard} that~$\Sym^p$ can be viewed as a Hadamard manifold and that the associated geodesic convexity coincides with the usual Euclidean convexity. Thus, the feasible set~$\PSD^p$ constitutes a convex subset of the Hadamard manifold~$\Sym^p$. In addition, we embed the feasible set~$\B_\eps (\covsa)$ of the maximization problem in~\eqref{eq:minimax_proof} into~$\PD^p$. Recall from Example~\ref{exam:PD_Hadamard} that~$\PD^p$ also constitutes a Hadamard manifold. The objective function $\Tr{X^2} - 2 \langle \cov , X \rangle$ of~\eqref{eq:minimax_proof} is ostensibly convex and continuous in~$X$.
Similarly, by Lemma~\ref{lemma:convex-inner}, the objective function is geodesically concave and continuous in~$\cov$. Hence, Theorem~\ref{thm:Sion-manifold} applies, and the desired minimax equality~\eqref{eq:minimax_proof} follows if we can prove that~$\B_\eps (\covsa)$ is geodesically convex as well as compact with respect to the metric topology induced by the Riemannian geometry on~$\PD^p$. By Remark~\ref{rmk:manifold_topology}, however, this  notion of compactness is equivalent to the usual compactness notion with respect to the Euclidean space~$\Sym^p$. Therefore, it suffices to show that $\B_\eps (\covsa)$ is compact in the usual sense.

As for the Fisher-Rao divergence, the compactness and geodesic convexity of~$\B_\eps (\covsa)$ follow from Lemma~\ref{lem:FR_compact_g-convex}. It thus remains to prove the desired properties of~$\B_\eps (\covsa)$ for the inverse Stein divergence.
 
\textbf{Inverse Stein Divergence:} For any fixed $\covsa\in \PD^p$, the inverse Stein divergence~$D(\cov, \covsa)$ is continuous in~$\cov$. Therefore, the corresponding uncertainty set~$\B_\eps (\covsa)$ is closed. In addition, any~$\cov\in \B_\eps (\covsa)$ satisfies
    \begin{align*}
        \eps \ge D(\cov, \covsa)=\frac{1}{2} \sum_{i=1}^p \left( \lambda_i (\cov^{-1} \covsa) - 1 - \log \lambda_i (\cov^{-1} \covsa) \right) \ge \frac{1}{2}\left( \lambda_1 (\cov^{-1} \covsa) - 1 - \log \lambda_1 (\cov^{-1} \covsa) \right),
\end{align*}
where the second inequality holds because~$ t - 1 - \log t\ge 0$ for all~$t>0$. As $t - 1 - \log t$ grows indefinitely when~$t$ tends to~$0$, the above inequality implies that there exists~$\underline{\lambda} >0$ with~$\lambda_1 (\cov^{-1} \covsa) \ge \underline{\lambda}$ for all~$\cov \in\B_\eps (\covsa)$. This in turn implies that $\lambda_p (\covsa^{-1}  \cov) = \lambda_1^{-1} (\cov^{-1} \covsa) \le \underline{\lambda}^{-1}$ for all $\cov \in\B_\eps (\covsa)$. We may thus conclude that~$\B_\eps (\covsa)$ is compact. Finally, since $ D(\cov, \covsa) = \frac{1}{2}\left( \Tr{\cov^{-1} \covsa} - p + \log \det \cov - \log \det  \covsa \right) $, $D(\cov, \covsa)$ is a geodesically convex function of~$\cov$ thanks to Lemmas~\ref{lemma:convex-inner}\ref{lemma:convex-inner_ii} and \ref{lemma:convex-inner}\ref{lemma:convex-inner_iii}. Therefore, $\B_\eps (\covsa)$ is a geodesically convex set by virtue of Proposition~\ref{prop:level_set_geodesic_convex}.
\end{proof}

\subsection{Inapplicability of Sion's Minimax Theorem} \label{subsec:non-quasi-convex}
We now show through counterexamples that if~$D(\cov, \covsa)$ is the Fisher-Rao or inverse Stein divergence, then the corresponding uncertainty set $\B_\eps (\covsa) = \left\{ \cov\in \PSD^p: D(\cov, \covsa) \le \eps \right\}$ fails to be a convex subset of~$\Sym^p$. Hence, for these divergences, we cannot appeal to Sion's classical minimax theorem to prove~\eqref{eq:minimax_proof}. More precisely, we will show that $D(\cov, \covsa)$ fails to be quasi-convex and thus has non-convex sublevel sets.
\begin{definition}[Quasi-convex function] 
	A function $\psi : \PSD^p \to \overline{\R}$ is quasi-convex if for any $\cov_1, \cov_2 \in \PSD^p$ and $\lambda \in [0, 1]$, we have $\psi \left(\lambda \cov_1 + (1-\lambda) \cov_2 \right) \leq \max\{ \psi (\cov_1), \psi (\cov_2) \}$.
\end{definition}

\begin{example}[Non-convexity of the Fisher-Rao uncertainty set] \label{example:quasi-convex}
The divergence $D(\cov, \covsa) = \| \log (\covsa^{-\half} \cov \covsa^{-\half}) \|_{\mathrm{F}}^2$ is not quasi-convex in~$\cov$ for any fixed $\covsa \in \PD^3$. To see this, assume first that $\covsa = I_3$.~Setting
\begin{equation*}
\cov_1 = \begin{pmatrix}
    33  &  -5  & -10\\
    -5  &   6  &   3\\
   -10  &   3  &   4
\end{pmatrix}\quad\text{and}\quad \cov_2 = \begin{pmatrix}
     6  &  -4  &   5\\
    -4  &  11  &  -2\\
     5  &  -2  &  18
\end{pmatrix},
\end{equation*}
one readily verifies that $\cov_1, \cov_2 \succ 0$, while $D (\cov_1, I_3) = 16.4501$ and $D (\cov_2, I_3) = 16.2111$. In addition, we find 
\begin{equation*}
    \textstyle D( \frac{1}{2}\cov_1 + \frac{1}{2} \cov_2, I_3) = 18.6796 > \max\{ 16.4501,\ 16.2111 \} = \max\{ D(\cov_1, I_3), D (\cov_2, I_3) \} .
\end{equation*}
This shows that $D(\cov, I_3)$ fails to be quasi-convex in~$\cov$. For a generic $\covsa \in \PD^3$, we define $\cov_1' = \covsa^{\half} \cov_1\covsa^{\half}$ and $\cov_2' = \covsa^{\half} \cov_2\covsa^{\half}$. The above inequality then immediately implies that
\[
    \textstyle D ( \frac{1}{2}\cov_1' + \frac{1}{2} \cov_2' , \covsa)  >  \max\{ D (\cov_1', \covsa), D (\cov_2', \covsa) \} .
\]
Consequently, the function $D(\cov, \covsa)$ fails to be quasi-convex in~$\cov$ irrespective of $\covsa\in\PD^3$.
\end{example}

\begin{example}[Non-convexity of the inverse Stein uncertainty set] \label{example:quasi-convex:inverse-stein}

The function $D(\cov, \covsa) = \frac{1}{2}( \Tr{\cov^{-1} \covsa} - 3 + \log \det (\cov \covsa^{-1}))$ is not quasi-convex in~$\cov$ for any fixed $\covsa\in \PD^3$. Indeed, if $\covsa = I_3$, we may set
\begin{equation*}
\cov_1 = \begin{pmatrix}
    30  &  13  & 23\\
    13  &   12  &   9\\
   23  &   9  &   20
\end{pmatrix}\quad\text{and}\quad \cov_2 = \begin{pmatrix}
     27  &  13  &   23\\
    13  &  10  &  14\\
     23  &  14  &  30
\end{pmatrix}.
\end{equation*}
It can be verified that $\cov_1, \cov_2 \succ 0$, while $D (\cov_1, I_3) = 4.0427$ and $D (\cov_2, I_3) = 4.3020$. In addition, we find
\begin{equation*}
    \textstyle D ( \frac{1}{2}\cov_1 + \frac{1}{2} \cov_2, I_3) = 4.3262 > \max\{ 4.0427,\ 4.3020 \} = \max\{ D (\cov_1, I_3), D (\cov_2, I_3) \} .
\end{equation*}
This shows that $D (\cov, I_3)$ fails to be quasi-convex in~$\cov$. For a generic $\covsa \in \PD^3$, we define $\cov_1' = \covsa^{\half} \cov_1\covsa^{\half}$ and $\cov_2' = \covsa^{\half} \cov_2\covsa^{\half}$. The above inequality then immediately implies that
\[
    \textstyle D ( \frac{1}{2}\cov_1' + \frac{1}{2} \cov_2', \covsa)  >  \max\{ D (\cov_1', \covsa), D (\cov_2', \covsa) \} 
\]
that is, the function $D (\cov, \covsa)$ fails to be quasi-convex in~$\cov$ irrespective of $\covsa\in\PD^3$.
\end{example}

\subsection{Riemannian Geometry and Geodesic Convexity}
\label{subsec:RG_GC}

In order to keep this paper self-contained, we now briefly review some basic concepts from Riemannian geometry. For a more comprehensive survey of this topic, we refer to the excellent textbooks~\cite{lang2012fundamentals,lee2006riemannian}. 
\begin{definition}[Riemannian manifold]
    A Riemannian manifold is a pair $(\mc M, \{\langle \cdot, \cdot \rangle_u\}_{u\in\mc M})$ consisting of a  differentiable manifold~$\mc M$ and a smooth family of inner products $\{\langle \cdot, \cdot \rangle_u\}_{u\in\mc M}$ defined on the tangent spaces $T_u \mc M$ of $\mc M$. That is, for any $u\in\mc M$, $\langle\cdot, \cdot \rangle_u$ represents a symmetric, positive definite bilinear map on~$T_u \mc M$. The family $\{\langle \cdot, \cdot \rangle_u\}_{u\in\mc M}$ of inner products is called a Riemannian metric on $\mc M$.
\end{definition}
Throughout this paper we will restrict attention to Hadamard manifolds.
\begin{definition}[Hadamard manifolds]\label{def:Hadamard_manifold}
    A Hadamard manifold is a complete, simply connected Riemannian manifold that has everywhere non-positive sectional curvature.
\end{definition}

Intuitively, the sectional curvature of a Riemannian manifold is non-positive at a point~$u$ if and only if the area of any small two-dimensional disc centered at~$u$ is larger or equal to the area of a disc with the same radius in flat space. For a formal definition see~\cite[p.~236]{lang2012fundamentals} or~\cite[p.~154]{lee2006riemannian}. All piecewise continuously differentiable curves on a Riemannian manifold---and, in particular, on a Hadamard manifold---can be assigned a length. 
\begin{definition}[Length of a curve]
The length of a continuously differentiable curve $c:[0,1] \to \mc M$ on a Riemannian manifold $(\mc M, \{\langle \cdot, \cdot \rangle_u\}_{u\in\mc M})$ is defined as
\begin{equation*}
        L(c) = \int_0^1 \sqrt{ \langle \dot{c}(t), \dot{c}(t)\rangle_{c(t)} } \, {\rm d}t.
\end{equation*}
If $c$ is piecewise continuously differentiable, then its length is defined as the sum of the lengths of its pieces.
\end{definition}

The Riemannian distance between two points $u_1,u_2\in\mc M$ is defined as $d_{\mc M} (u_1, u_2) = \min_{c} L(c)$, where the minimum is over all continuously differentiable curves~$c$ with constant speed $(\langle\dot c(t), \dot c(t)\rangle_{c(t)})^\half$ that connect $u_1$ and $u_2$. For complete and connected Riemannian manifolds, the minimum is guaranteed to exist, and any minimizer is a geodesic.
Moreover, by the Hopf-Rinow theorem~\cite{lang2012fundamentals,lee2006riemannian}, any two points on a Hadamard manifold are connected by a unique geodesic. This greatly simplifies the study of convexity on such manifolds.

\begin{definition}[Geodesically convex sets]
    If $(\mc M, \{\langle \cdot, \cdot \rangle_u\}_{u\in\mc M})$ is a Hadamard manifold, then $\mc U\subseteq \mc M$ is geodesically convex if, for any $u_1,u_2\in\mc U$, the image of the geodesic connecting $u_1$ and $u_2$ lies within~$\mc U$.
\end{definition}

\begin{definition}[Geodesically (quasi-)convex function]
    \label{def:geodesically-convex-fct}
    If $(\mc M, \{\langle \cdot, \cdot \rangle_u\}_{u\in\mc M})$ is a Hadamard manifold and $\mc U\subseteq \mc M$ is geodesically convex, then the function $\psi : \mc U\to \mathbb{R}$ is geodesically (quasi-)convex if the composition $\psi\circ c: [0,1] \to \mathbb{R}$ is (quasi-)convex function in the usual Euclidean sense for every geodesic~$c$ connecting two arbitrary points in $\mc U$. In addition, $\phi$ is geodesically (quasi-)concave if $-\phi$ is geodesically (quasi-)convex.
\end{definition}
Definition~\ref{def:geodesically-convex-fct} makes sense because a geodesic is always parametrized proportionally to arc length. It readily implies that all sublevel sets of a geodesically quasi-convex function are geodesically convex.

\begin{proposition}[{\cite[Theorem~3.4]{udriste2013convex}}]
    \label{prop:level_set_geodesic_convex}
    If $(\mc M, \{\langle \cdot, \cdot \rangle_u\}_{u\in\mc M})$ is a Hadamard manifold and $\psi:\mc M\to \mathbb{R}$ is geodesically quasi-convex, then the sublevel set $\{ u\in\mc M: \psi (u) \le \alpha \}$ is geodesically convex for any $ \alpha \in \mathbb{R}$.
\end{proposition}

The examples below are useful for our theoretical development and used in the proof of Proposition~\ref{prop:minimax_assumption}.
\begin{example}\label{exam:Euclidean_Hadamard}
The Euclidean spaces $\mathbb{R}^p$ and $\Sym^p$ equipped with their usual inner products constitute Hadamard manifolds. In both cases, geodesic convexity (of sets as well as functions) reduces to Euclidean convexity.
\end{example}

\begin{example}\label{exam:PD_Hadamard}
The cone of positive definite matrices $\PD^p$ represents a differentiable manifold~\cite{ref:bhatia2007pd,lang2012fundamentals}. The tangent space $T_\cov \PD^p$ at $\cov\in \PD^p$ is naturally identified with~$\mathbb{S}^p$, that is, all tangent vectors constitute symmetric matrices. We can assign every~$\cov\in\PD^p$ an inner product $\langle \cdot, \cdot\rangle_{\cov} : \mathbb{S}^p\times \mathbb{S}^p \to \mathbb{R}$ defined through
\[
    \langle \cov_1, \cov_2\rangle_{\cov} = \Tr{\cov^{-1} \cov_1 \cov^{-1} \cov_2} \quad\forall \cov_1,\cov_2 \in \mathbb{S}^p.
\]
By \cite[Theorem~XII~1.2]{lang2012fundamentals}, $\PD^p$ equipped with the inner products $\langle \cdot, \cdot\rangle_{\cov}$, $\cov\in\PD^p$, is a Hadamard manifold.
\end{example}

\begin{remark}\label{rmk:manifold_topology}
By definition, any Hadamard manifold $(\mc M, \{\langle \cdot, \cdot \rangle_u\}_{u\in\mc M})$ is simply connected and therefore, in particular, connected. Hence, \cite[Theorem 13.29]{lee2013smooth} implies that the metric topology on~$\mc M$ induced by the Riemannian distance~$d_{\mathcal{M}}$ coincides with the manifold topology. For instance, the metric topology on the Hadamard manifold $\PD^p$ from Example~\ref{exam:PD_Hadamard} coincides with the subspace topology on~$\PD^p$ inherited from the ambient vector space~$\Sym^p$, which is the standard (Euclidean norm) topology used for matrices.
\end{remark}

In the following lemmas, we treat $\PD^p$ as a Hadamard manifold in the sense of Example~\ref{exam:PD_Hadamard}. 
\begin{lemma}[Compactness and convexity {\cite[Theorem~2.5]{nguyen2019calculating}}]\label{lem:FR_compact_g-convex}
For any fixed $\cov'\in\PD^p$, the set 
\[
     \left\{ \cov \in \PD^p: \, \| \log ({\cov'}^{-\half} \cov {\cov'}^{-\half}) \|_{\mathrm{F}}^2 \leq \eps^2 \right\}
\]
constitutes a compact and geodesically convex subset of~$\PD^p$.
\end{lemma}

We now show that several popular matrix functions are geodesically convex. Here, we adopt the standard terminology whereby a function that is both geodesically convex and concave is called geodesically linear.

\begin{lemma}[Geodesic convexity of popular matrix functions] \label{lemma:convex-inner}
    The following hold.
    \begin{enumerate}[label = (\roman*)]
        \item\label{lemma:convex-inner_i} $g(\cov)= \Tr{X\cov}$ is geodesically convex on~$\PD^p$ for every $X\in \PSD^p$.
        \item\label{lemma:convex-inner_ii} $g(\cov)= \Tr{X\cov^{-1}}$ is geodesically convex on~$\PD^p$ for every $X\in \PSD^p$.
        \item\label{lemma:convex-inner_iii} $g(\cov)= \log \det \cov$  is geodesically linear on~$\PD^p$.
    \end{enumerate}   
\end{lemma}

\begin{proof}[Proof of Lemma~\ref{lemma:convex-inner}]
We can prove assertion~\ref{lemma:convex-inner_i} by showing that, for every fixed $\cov\in\PD^p$, the Riemannian Hessian of the function $g(\cov)= \Tr{X\cov}$ is positive semidefinite on the tangent space $T_\cov \PD^p \cong \mathbb{S}^p$ \cite{absil2009optimization, udriste2013convex}. To this end, note first that the Euclidean gradient of~$g$ is given by~$\nabla g(\cov) = X$ and that the Euclidean Hessian $\nabla^2 g(\cov)$ coincides with the zero map from~$\Sym^p$ to~$\Sym^p$. By \cite[\S~4.2]{ferreira2019gradient}, the Riemannian Hessian of~$g$ thus satisfies
\begin{align*}
        \mathrm{Hess}\, g(\Sigma) [S] &= \half ( SX\Sigma + \Sigma XS) \quad \forall S \in \mathbb{S}^p. 
\end{align*}
This implies that
\[
    \inner{\mathrm{Hess}\, g(\Sigma)[S] }{S}_\Sigma =  \Tr{ SXS \Sigma^{-1} } \ge 0 \quad \forall S \in \mathbb{S}^p,
\]
where the inequality holds because $SXS\in \PSD^p$ and $\cov^{-1}\in \PD^p$. Thus, the Riemannian Hessian of~$g$ is positive semidefinite on the tangent space $T_\cov \PD^p \cong \mathbb{S}^p$. As $\cov\in\PD^p$ was chosen freely, this shows via~\cite[Theorem~6.2]{udriste2013convex} that~$g$ is geodesically convex throughout~$\PD^p$.

Assertions~\ref{lemma:convex-inner_ii} and~\ref{lemma:convex-inner_iii} are proved similarly. As for assertion~\ref{lemma:convex-inner_ii}, note that the gradient of~$g(\cov)= \Tr{X\cov^{-1}}$ is given by $\nabla g( \cov) = - \cov^{-1} X \cov^{-1}$ \cite[\S~2.2]{ref:petersen2012cookbook}. Also, the Hessian of $g$ is a linear operator on $\Sym^p$ satisfying 
\begin{align*}
    \nabla^2 g(\cov) [S] & = \left.\frac{\mathrm{d} \nabla g(\cov + t S  )}{\mathrm{d} t}\right|_{t=0} = - \left.\frac{\mathrm{d} (\cov + t S)^{-1}}{\mathrm{d} t}\right|_{t=0} X \cov^{-1} - \cov^{-1} X \left.\frac{\mathrm{d} (\cov + t S)^{-1}}{\mathrm{d} t}\right|_{t=0}\\
    & = \cov^{-1} S \cov^{-1} X \cov^{-1} + \cov^{-1} X\cov^{-1} S \cov^{-1},
\end{align*}
where the third equality exploits~\cite[\S~2.2]{ref:petersen2012cookbook}. By \cite[\S~4.2]{ferreira2019gradient}, the Riemannian Hessian of~$g$ thus satisfies
\begin{align*}
        \mathrm{Hess}\, g(\Sigma) [S] &= \half ( S\Sigma^{-1}X + X\Sigma^{-1} S) \quad \forall S \in \mathbb{S}^p. 
\end{align*}
This implies that
\[
    \inner{\mathrm{Hess}\, g(\Sigma)[S] }{S}_\Sigma = \half \Tr{ \cov^{-1} ( S\Sigma^{-1}X + X\Sigma^{-1} S) \cov^{-1} S } = \Tr{ S\cov^{-1} S \cov^{-1}X\cov^{-1} } \ge 0 \quad \forall S \in \mathbb{S}^p,
\]
where the inequality holds because $S\cov^{-1} S$ and $\cov^{-1}X\cov^{-1}$ are positive semidefinite. Thus, the Riemannian Hessian of~$g$ is positive semidefinite on the tangent space $T_\cov \PD^p \cong \mathbb{S}^p$, and the claim follows.

As for assertion~\ref{lemma:convex-inner_iii}, the gradient of $g(\cov)= \log \det \cov$ is given by $\nabla g( \cov) = - \cov^{-1} $, and the Hessian of $g$ is a linear operator on $\Sym^p$ satisfying $\nabla^2 g(\cov) [S]  =\cov^{-1} S \cov^{-1}$ \cite[\S~2.2]{ref:petersen2012cookbook}. By \cite[\S~4.2]{ferreira2019gradient}, the Riemannian Hessian of~$g$ thus satisfies
$\mathrm{Hess}\, g(\Sigma) [S] = 0$ for all $S \in \mathbb{S}^p$. Hence, $g$ is both geodesically convex and concave on~$\PD^p$.
\end{proof}

\subsection{A Riemannian Generalization of Sion's Minimax Theorem}\label{sec:riemannian_sion}
We now present a generalization of Sion's minimax theorem for geodesically convex-concave saddle functions on Hadamard manifolds.\footnote{While finalizing this paper, we discovered a concurrent work describing a result akin to Theorem~\ref{thm:Sion-manifold} \cite{zhang2022minimax}.}

\begin{theorem}[Sion's minimax theorem for geodesically convex-concave saddle problems] \label{thm:Sion-manifold}
    Let~$\mc U$ and~$\mc V$ be geodesically convex subsets of two Hadamard manifolds, and assume that $\mc U$ is compact. Also, let $\psi: \mc U\times \mc V \rightarrow \R$ be a function with $\psi(u, \cdot)$ being upper semi-continuous and geodesically quasi-concave on $\mc V $ for any fixed~$u \in \mc U $ and with $\psi(\cdot, v)$ being lower semi-continuous and geodesically quasi-convex on $\mc U $ for every fixed $v \in \mc V $. Then,
	\[
	\Min{u \in \mc U } \Sup{v \in \mc V } ~ \psi(u, v) = \Sup{v \in \mc V } \Min{u \in \mc U } ~ \psi(u, v).
	\]
\end{theorem}

The proof of this Riemannian minimax theorem is omitted because it closely follows the approach in~\cite{ref:komiya1988elementary}, with the natural modification that line segments between two points in linear spaces are replaced by (unique) geodesics connecting two points on Hadamard manifolds.

\section{Verification of the Rearrangement Property}\label{app:verify-ass2c}

\begin{proposition}\label{prop:rearrange}
All the divergences listed in Table~\ref{table:structured_divergence} satisfy Assumption~\ref{assu:D_form}\ref{assu:D_form_iv}.
\end{proposition}

\begin{proof}[Proof of Proposition~\ref{prop:rearrange}]
Let~$D$ be the Kullback-Leibler, Fisher-Rao, inverse Stein or symmetric Stein divergence. In either case, if~$x$ or~$y$ contains any vanishing entry, then both sides of the rearrangement inequality in Assumption~\ref{assu:D_form}\ref{assu:D_form_iv} evaluate to~$+\infty$; see the definitions in Table~\ref{table:structured_divergence}. Thus, Assumption~\ref{assu:D_form}\ref{assu:D_form_iv} is trivially satisfied. It therefore suffices to prove the inequality for $x,y\in \R_{++}^p$. Next, let~$D$ be the weighted quadratic divergence. Hence, if~$y$ contains any vanishing entry, then both sides of the rearrangement inequality evaluate again to~$+\infty$, and Assumption~\ref{assu:D_form}\ref{assu:D_form_iv} is trivially satisfied. It therefore suffices to assume that~$y\in\R_{++}^p$. With these assumptions in place, both sides of the rearrangement inequality are guaranteed to be finite.

The subsequent proof requires additional notation. We use~$\sigma_i (S)$ to denote the $i$-th smallest singular value of the matrix~$S\in \Sym^p$. The vector $\sigma(S)\in\R_+^p$ is then defined through $(\sigma (S))_i = \sigma_i(S)$ for all $i = 1,\dots,p$. Any univariate function $g :\R\to \R$ naturally induces multivariate functions $g:\R^p \to \R^p$ and $g:\Sym^p\to \Sym^p$, which, by slight abuse of notation, are represented by the same symbol~$g$. Specifically, for any $x\in\R^p$, we define $g(x) \in \R^p$ through $( g(x) )_i = g(x_i)$ for all $i = 1,\dots,p$. Similarly, for any $S\in\Sym^p$ with eigenvalue decomposition $S = V_S \Diag (\lambda (S)) V_S^\top$ with $V_S\in \mathcal{O}_p$, we define $g(S)\in\Sym^p$ through $g(S) = V_S \Diag(g(\lambda (S))) V_S^\top$.

Observe now that all divergences listed in Table~\ref{table:structured_divergence} are representable as
\begin{equation}\label{eq:D_form_1}
    D(X, Y) = \sum_{i=1}^p \big( h_1 (\lambda_i(X)) + h_2(\lambda_i(Y)) \big) + \sum_{i = 1}^p f\big(\lambda_i( g_2( Y^{\half}) g_1 (X) g_2( Y^{\half}) )\big)
\end{equation}
for some functions $f$, $h_1$, $h_2$, $g_1$ and $g_2$ from~$\R$ to~$\overline{\R}$ as specified in Table~\ref{table:structured_divergence_f_h}.
\bgroup
\def\arraystretch{1.8}
\begin{table}[t]
\begin{tabular}{|l||c|c|c|c|c|c|c|}
\hline
Divergence & $h_1(t)$ & $h_2(t)$ & $g_1(t )$ & $g_2(t)$ & $f(t)$ & $tf'(t)$\\ \hline\hline

Kullback-Leibler & $-\tfrac{1}{4}$& $-\tfrac{1}{4}$ & $t$  & $\frac{1}{t}$ & $\frac{1}{2}\left( t- \log t \right)$& $\frac{1}{2} (t - 1)$\\ \hline

Wasserstein & $t$ & $t$ &  $t$ & $t$ & $-2 \sqrt{t}$ & $-\sqrt{t}$\\ \hline

Fisher-Rao & $0$ & $0$ & $t$ & $\frac{1}{t}$ & $\left( \log t  \right)^2$ & $2\log t$ \\ \hline

Inverse Stein & $-\tfrac{1}{4}$& $-\tfrac{1}{4}$ & $t$ & $\frac{1}{t}$ & $\frac{1}{2}\left(\frac{1}{t} + \log t\right)$ & $ \frac{1}{2}(1- \frac{1}{t}) $ \\ \hline

Symmetrized Stein & $-\tfrac{1}{2}$ & $-\tfrac{1}{2}$ & $t$ & $\frac{1}{t}$ & $\frac{1}{2}\left(t+\frac{1}{t}\right)$ & $\frac{1}{2}(t - \frac{1}{t})$ \\ \hline

Quadratic & $t^2$ & $t^2$& $t$ & $t$ & $-2t$ & $-2t$\\ \hline

Weighted quadratic & $-2t$ & $t$& $t^2$ & $\frac{1}{t}$ & $t$ 
& $t$ \\ \hline
\end{tabular}
\vspace{2mm}
\caption{Functions $h_1$, $h_2$, $g_1$, $g_2$ and $f$ in the representation~\eqref{eq:D_form_1} of the divergences of Table~\ref{table:structured_divergence}.}
\label{table:structured_divergence_f_h}
\end{table}
\egroup
\noindent As the spectrum of any matrix is invariant under conjugation with an orthogonal matrix $V\in {\mathcal{O}_p}$, we have 
\begin{equation*}
    \sum_{i=1}^p \left( h_1 (\lambda_i( V \Diag(x^\uparrow) V^\top )) + h_2(\lambda_i(\Diag(y^\uparrow))) \right) = \sum_{i=1}^p \left( h_1 (\lambda_i(  \Diag(x^\uparrow)  )) + h_2(\lambda_i(\Diag(y^\uparrow))) \right)
\end{equation*}
for all~$x,y \in \mathbb{R}^p$. In view of the representation~\eqref{eq:D_form_1} and the above identity, it remains to be shown that 
\begin{equation}\label{ineq:rearrange_f}
    \sum_{i = 1}^p f \left(\lambda_i( \Diag(\sqrt{y}^\uparrow ) V g_1 ( \Diag(x^\uparrow) ) V^\top g_2( \Diag(\sqrt{y}^\uparrow )) )\right) \ge \sum_{i = 1}^p f\left(\lambda_i( g_1 ( \Diag(x^\uparrow) ) g_2( \Diag(y^\uparrow)) )\right) 
\end{equation}
for all $x, y\in\R_+^p$ and $V\in {\mathcal{O}_p}$. Table~\ref{table:structured_divergence_f_h} shows that always either of the following two conditions holds:
\begin{itemize}
    \item $t\mapsto t f'(t)$ is strictly increasing, $g_1$ is strictly increasing and $g_2$ is is strictly decreasing;
    
    \item $t\mapsto t f'(t)$ is strictly decreasing, and $g_1$ and $g_2$ are both strictly increasing.
\end{itemize}
The desired inequality~\eqref{ineq:rearrange_f} then follows from \cite[Theorem~3]{yue2020matrix}. 
Inspecting the proofs of \cite[Theorem~3 and Lemma~1]{yue2020matrix} further reveals that \eqref{ineq:rearrange_f} holds if and only if $V g_1 ( \Diag(x^\uparrow) ) V^\top = g_1 ( \Diag(x^\uparrow) )$, which is equivalent to $V  \Diag(x^\uparrow)  V^\top =  \Diag(x^\uparrow) $ because $g_1$ is strictly increasing.
This observation completes the proof.
\end{proof}

\section{Proofs of Section~\ref{sec:new_CSE}}\label{app:proofs_section-4}
\begin{proof}[Proof of Theorem~\ref{thm:verification}]
We prove the assumptions one by one. Note first that, by Proposition~\ref{prop:minimax_assumption}, every divergence~$D$ in Table~\ref{table:structured_divergence} satisfies the minimax property specified in Assumption~\ref{assu:inf_sup}. 

Assumption~\ref{assu:D_form} requires~$D$ to be a spectral divergence. To show that~$D$ is orthogonally equivariant, recall that the spectrum of a matrix is preserved under similarity transformations. As the trace and the determinant are spectral functions, the orthogonal equivariance of all divergences in Table~\ref{table:structured_divergence} is easily verified using elementary rules of matrix algebra. It is also straightforward to verify that every divergence~$D$ in Table~\ref{table:structured_divergence} is spectral with generator~$d$ as specified in Table~\ref{table:a}. In addition, the domain of~$d$ contains a point~$(a,b)$ with~$b>0$, and~$d$ is ostensibly continuous throughout its domain. The rearrangement property holds thanks to Proposition~\ref{prop:rearrange}.

Assumption~\ref{assu:d_convex} follows immediately from definitions of the generators in Table~\ref{table:a}. For example, it is clear that the generator~$d_b(\cdot)=d(\cdot,b)=( \log (\cdot/b))^2$ of the Fisher-Rao divergence is twice continuously differentiable on~$\R_{++}$ for any fixed~$b>0$. In addition, we have $d_b'' (a) = 2( 1 - \log(a/b))/a^2 > 0$ for any~$a \in (0,b]$ and~$b>0$, which shows that~$d_b$ is convex on~$[0,b]$. Similarly, one can prove Assumption~\ref{assu:d_convex} for all other divergences.

It remains to be shown that all generators in Table~\ref{table:a} satisfy the differential inequality of Assumption~\ref{assu:d_b_second_derivative}. For example, the generator $d(a,b)=( \log (a/b))^2$ of the Fisher-Rao divergence satisfies
\[
    \frac{\partial}{\partial a}d(a,b) = \frac{2}{a}\log\frac{a}{b}, \quad \frac{\partial^2}{\partial a^2} d(a,b)= \frac{2}{a^2}\left( 1 - \log\frac{a}{b} \right)\quad \text{and}\quad  \frac{\partial^2}{\partial b\partial a}d(a,b) = -\frac{2}{ab} \quad\forall a,b\in\R_{++}.
\]
Therefore, we obtain
\[ 
    a \, \frac{\partial^2}{\partial a^2}d(a,b) + b\, \frac{\partial^2}{\partial a\partial b}d(a,b)  -\frac{\partial}{\partial a}d(a,b)= \frac{2}{a}\left( 1 - \log\frac{a}{b} \right) - \frac{2}{a} - \frac{2}{a}\log\frac{a}{b} = - \frac{4}{a}\log \frac{a}{b} > 0
\]
for all for any $b > a >0$. Hence, Assumption~\ref{assu:d_b_second_derivative} holds for the Fisher-Rao divergence. Similarly, Assumption~\ref{assu:d_b_second_derivative} can be proved for all other divergences using the basic rules of calculus.
\end{proof}

We now prove Corollaries~\ref{thm:KL}, \ref{thm:Wass} and~\ref{thm:FR}, which characterize the eigenvalue map as well as the inverse shrinkage intensity of the KL, Wasserstein and Fisher-Rao covariance shrinkage estimators, respectively.

\begin{proof}[Proof of Corollary~\ref{thm:KL}]
The generator of the KL divergence is given by $d(a,b) = \frac{1}{2}(\frac{a}{b} - 1 - \log \frac{a}{b})$; see Table~\ref{table:a}.
Note that Assumptions~\ref{assu:inf_sup}, \ref{assu:D_form}, \ref{assu:d_convex} and~\ref{assu:d_b_second_derivative} hold by Theorem~\ref{thm:verification}, Assumption~\ref{assu:data}\ref{assu:data-d} holds because~$\covsa\in \PD^p$, and Assumption~\ref{assu:data}\ref{assu:data-eps} holds because $d(0,b) = +\infty$ for any $b > 0$. Therefore, Theorem~\ref{thm:general_CSE} applies, which implies that problem~\eqref{eq:CSE} is uniquely solved by $X \opt= \Vsa \Diag (x\opt) \Vsa^\top$, where $x\opt_i = s(\gamma\opt, \xsa_i)$ for every $i = 1,\dots,p$. Next, we construct the eigenvalue map~$s$ defined in~\eqref{eq:s}. If~$b>0$, then~$s(\gamma, b)$ is the unique solution~$a^\star\geq 0$ of
\[
    0 = 2a\opt + \gamma \frac{\partial}{\partial a}  d(a\opt, b) = 2 a^\star + \frac{\gamma }{2 }\left( \frac{1}{b} - \frac{1}{a^\star} \right).
\]
We thus obtain
\[ 
    s(\gamma, b) = \frac{-\gamma  + \sqrt{{\gamma }^2 + 16 b^2\gamma }}{8b}. 
\]
It remains to find a formula for~$\gamma\opt$. By Theorem~\ref{thm:general_CSE}, $\gamma^\star$ is the unique positive root of the equation
\begin{equation*} 
    \sum_{i = 1}^p d(s(\gamma^\star,\xsa_i), \xsa_i) -\eps = 0\quad \iff\quad  2\varepsilon +p + \sum_{i=1}^p \left[-\frac{s(\gamma\opt, \xsa_i)}{\xsa_i} +\log\frac{s(\gamma\opt, \xsa_i)}{\xsa_i} \right]=0.
\end{equation*} 
To show that~$\dualvar_{\KL}$ provides an upper bound on~$\gamma^\star$, note that the above equation implies that
\begin{align*}
    0 & = 2\eps + p + \sum_{i=1}^p \left[-\frac{s(\gamma\opt, \xsa_i)}{\xsa_i} +\log\frac{s(\gamma\opt, \xsa_i)}{\xsa_i} \right] \ge 2\eps + \sum_{i=1}^p \log\frac{s(\gamma\opt, \xsa_i)}{\xsa_i} \ge 2\eps + p  \log\frac{s(\gamma\opt, \xsa_p)}{\xsa_p}.
\end{align*}
Here, the two inequalities follow from Lemmas~\ref{lem:a} and~\ref{lemma:cond_decrease}, which imply that $s(\gamma,b) < b$ for all~$\gamma,b>0$ and that $s(\gamma, b)/b$ is non-increasing in~$b$, respectively. Rearranging the above inequality yields $\xsa_p \,e^{-\frac{2\eps}{p} } \ge s(\gamma\opt, \xsa_p)$. As $s(\gamma, \xsa_p)$ is strictly increasing in~$\gamma$ by virtue of Lemma~\ref{lem:a}\ref{lem:a-1}, the unique solution~$\dualvar_{\KL}$ of the equation
\begin{align*}
    \xsa_p \,e^{-\frac{2\eps}{p} } = s(\dualvar_{\KL}, \xsa_p) = \frac{-\dualvar_{\KL} + \sqrt{{\dualvar_{\KL}}^2 + 16\xsa_p^2 \dualvar_{\KL}}}{8\xsa_p}
\end{align*}
provides an upper bound on~$\gamma\opt$. The desired formula for~$\dualvar_{\KL}$ is obtained by solving this equation.
\end{proof}

\begin{proof}[Proof of Corollary~\ref{thm:Wass}]
The generator of the Wasserstein divergence is given by $d(a,b) = a+b-2\sqrt{ab}$; see Table~\ref{table:a}. Assumptions~\ref{assu:inf_sup}, \ref{assu:D_form}, \ref{assu:d_convex} and~\ref{assu:d_b_second_derivative} hold by Theorem~\ref{thm:verification}, Assumption~\ref{assu:data}\ref{assu:data-d} holds because~$\covsa\in\PSD^p$, and Assumption~\ref{assu:data}\ref{assu:data-eps} holds because $\eps \in ( 0, \Tr{\covsa} ) $, which implies that $\sum_{i=1}^p d(0, \xsa_i) = \sum_{i=1}^p \xsa_i = \Tr{\covsa} > \eps$. Thus, Theorem~\ref{thm:general_CSE} applies. Recall now from~\eqref{eq:s} that if~$\gamma>0$, then~$s(\gamma, b)$ is defined as the unique solution~$a^\star\geq 0$ of
\[
    0 = 2a\opt + \gamma \frac{\partial}{\partial a}  d(a\opt, b) = 2a\opt + \gamma \left(1-\sqrt{\frac{b}{a\opt}}\right).
\]
Solving a cubic equation in~$\sqrt{a\opt}$ thus reveals that~$s(\gamma, b)$ is given by~\eqref{eq:Wass-xiopt}. Theorem~\ref{thm:general_CSE} further implies that the inverse shrinkage intensity~$\gamma\opt$ is the unique positive root of the equation~\eqref{eq:Wass-FOC}. To show that~$\dualvar_{\W}$ provides an upper bound on~$\gamma\opt$, let $ i'\in \{1,\dots,p\}$ be the smallest index~$i$ with~$\xsa_i > 0$. As~$s(\gamma\opt, 0) = 0$, \eqref{eq:Wass-FOC} implies
\begin{align}
    0 & = \eps - \sum_{i= i'}^p \left(\sqrt{\xsa_i} - \sqrt{s(\gamma\opt, \xsa_i)} \right)^2\notag \ge \eps - \xsa_p \sum_{i=i'}^p \left( 1 - \sqrt{\frac{s(\gamma\opt,\xsa_i )}{\xsa_i}} \right)^2 \notag\\
    & \ge \eps - p \xsa_p \left( 1- \sqrt{\frac{s(\gamma\opt,\xsa_p )}{\xsa_p}} \right)^2 = \eps - p \left( \sqrt{\xsa_p} - \sqrt{s(\gamma\opt,\xsa_p)} \right)^2, \label{eq:Wass_a_proof_2}
\end{align}
where the first inequality holds because~$\xsa_i \le \xsa_p$, and the second inequality follows from Lemmas~\ref{lemma:cond_decrease} and~\ref{lem:a}, which imply that $s(\gamma, b)/b$ is non-increasing in~$b$ and that $0< s(\gamma, b) < b$ for all~$\gamma,b>0$, respectively. The defining equation for~$s(\gamma\opt,\xsa_p)$ further implies that
\begin{equation}
    \label{eq:Wass_proof_3}
    \left( \sqrt{\xsa_p} - \sqrt{s(\gamma\opt,\xsa_p)} \right)^2 = \frac{4 s(\gamma\opt, \xsa_p)^3}{{\gamma\opt}^2}.
\end{equation}
Substituting \eqref{eq:Wass_proof_3} into \eqref{eq:Wass_a_proof_2} yields
\begin{align*}
    0 \ge \eps - \frac{4 p s(\gamma\opt, \xsa_p)^3}{{\gamma\opt}^2} \ge \eps - \frac{4 p \xsa_p^3}{{\gamma\opt}^2} \quad \iff\quad \gamma\opt \le 2\sqrt{\frac{p \xsa_p^3 }{\eps}} =\dualvar_{\W}.
\end{align*}
This observation completes the proof.
\end{proof}

\begin{proof}[Proof of Corollary~\ref{thm:FR}]
The generator of the Fisher-Rao divergence is $d(a,b) = (\log\frac{a}{b})^2$; see Table~\ref{table:a}. Assumptions~\ref{assu:inf_sup}, \ref{assu:D_form}, \ref{assu:d_convex} and~\ref{assu:d_b_second_derivative} hold by Theorem~\ref{thm:verification}, Assumption~\ref{assu:data}\ref{assu:data-d} holds because~$\covsa\in \PD^p$, and Assumption~\ref{assu:data}\ref{assu:data-eps} holds because~$d(0,b) = +\infty$ for any~$b>0$. Thus, Theorem~\ref{thm:general_CSE} applies. If~$b>0$, $s(\gamma, b)$ is the unique solution~$a^\star\geq 0$~of
\[
    0 = 2a\opt + \gamma \frac{\partial}{\partial a}  d(a\opt, b) = 2 a\opt + \frac{2\gamma }{a\opt}\log\frac{a\opt}{b} \quad \iff\quad \frac{2(a\opt)^2}{\gamma} e^{\frac{2(a\opt)^2}{\gamma}} = \frac{2b^2}{\gamma}.
\]
Recall now that, for any $t > -e^{-1}$, the principal branch of the Lambert $W$-function is defined as the unique solution~$W_0 (t)$ of the equation $W e^{W} = t$. Identifying~$W$ with $2(a\opt)^2/\gamma$ and~$t$ with~$2b^2/\gamma>0$, we thus find
\begin{equation}
    \label{eq:FR_a_proof_1}
    s(\gamma, b) = \sqrt{\frac{\gamma}{2} W_0\left( \tfrac{2b^2}{\gamma} \right)} = b \exp\left( - \frac{1}{2} W_0 (\tfrac{2b^2}{\gamma}) \right),
\end{equation} 
where the second equality holds because $W_0(t) =te^{-W_0(t)}$. This proves~\eqref{eq:FR-xiopt}.
Theorem~\ref{thm:general_CSE} further implies that the inverse shrinkage intensity~$\gamma\opt$ is the unique positive root of the equation~\eqref{eq:FR_opt_gamma}. It remains to prove that~$\dualvar_{\FR}$ upper bounds~$\gamma\opt$. Recalling that $0 \le W_0 (t) = t \exp(-W_0(t)) \le t$ for any $t\ge 0$, \eqref{eq:FR_opt_gamma} implies that
    \begin{align*}
        4 \eps = \sum_{i = 1}^p W_0^2 \left( \frac{2\xsa_i^2}{\gamma^\star} \right) \le \sum_{ i = 1 }^p \frac{4\xsa_i^4}{{\gamma^\star}^2} \quad\implies \quad \gamma^\star \le  \sqrt{\sum_{i=1}^p \frac{\xsa_i^4}{\eps} } \le \|\covsa\|_{\mathrm{F}}^2\sqrt{\varepsilon}=\dualvar_{\FR}.
    \end{align*}
This observation completes the proof.
\end{proof}

\end{appendix}

\end{document}

%% file: comments.tex
\makeatletter
\def\munderbar#1{\underline{\sbox\tw@{$#1$}\dp\tw@\z@\box\tw@}}
\makeatother

\newtheorem{theorem}{Theorem}
\newtheorem{informaltheorem}{Theorem}
\newtheorem*{theorem*}{Theorem}
\newtheorem{definition}{Definition}
\newtheorem{lemma}{Lemma}
\newtheorem{remark}{Remark}
\newtheorem{corollary}{Corollary}
\newtheorem{proposition}{Proposition}

\newtheorem{assumption}{Assumption}
\newtheorem{example}{Example}

\newcommand{\be}{\begin{equation}}
\newcommand{\ee}{\end{equation}}
\newcommand{\bea}{\begin{equation*}\begin{aligned}}
\newcommand{\eea}{\end{aligned}\end{equation*}}

\newcommand{\R}{\mathbb{R}}

\newcommand{\Min}{\min\limits_}
\newcommand{\Sup}{\sup\limits_}
\newcommand{\Inf}{\inf\limits_}
\newcommand{\Tr}[1]{\Trace [ #1 ]}

\newcommand{\wh}{\widehat}
\newcommand{\mc}{\mathcal}
\newcommand{\mbb}{\mathbb}
\newcommand{\inner}[2]{\big \langle #1, #2 \big \rangle }
\newcommand{\norm}[1]{\left\|#1\right\|}
\newcommand{\cov}{\Sigma} 
\newcommand{\covsa}{\wh{\Sigma}} 
\newcommand{\Vsa}{\wh{V}}
\newcommand{\vsa}{\hat{v}}
\newcommand{\xsa}{\hat{x}}

\newcommand{\trueP}{\mbb P}

\DeclareMathOperator{\Trace}{Tr}

\DeclareMathOperator{\Diag}{Diag}

\DeclareMathOperator{\st}{s.t.}

\DeclareMathOperator{\KL}{KL}
\DeclareMathOperator{\W}{W}

\DeclareMathOperator{\FR}{FR}
\DeclareMathOperator*{\Argmin}{Argmin}
\DeclareMathOperator{\Minval}{Min}
\DeclareMathOperator{\Fea}{Fea}

\newcommand{\PSD}{\mathbb{S}_{+}} 
\newcommand{\PD}{\mathbb{S}_{++}} 

\newcommand{\opt}{^\star}
\newcommand{\eps}{\varepsilon}

\newcommand{\BB}{\mbb U}
\newcommand{\B}{\mc B}
\newcommand{\X}{X}

\newcommand{\EE}{\mathbb{E}}

\newcommand{\half}{\frac{1}{2}}
\newcommand{\dualvar}{\gamma}

\newcommand{\ie}{{\em i.e.}}

\newcommand{\p}{\mathbb P}
\newcommand{\psa}{\wh \p}

\newcommand{\Sym}{\mbb S}
\newcommand{\m}{\mu}
\newcommand{\msa}{\wh \m}

%% file: Fig1.tex
\tikzset{every picture/.style={line width=0.75pt}} 

\begin{tikzpicture}[x=0.75pt,y=0.75pt,yscale=-1,xscale=1]

\draw   (10,62) .. controls (10,55.37) and (15.37,50) .. (22,50) -- (278,50) .. controls (284.63,50) and (290,55.37) .. (290,62) -- (290,98) .. controls (290,104.63) and (284.63,110) .. (278,110) -- (22,110) .. controls (15.37,110) and (10,104.63) .. (10,98) -- cycle ;
\draw    (292,80) -- (408,80) ;
\draw [shift={(410,80)}, rotate = 180] [color={rgb, 255:red, 0; green, 0; blue, 0 }  ][line width=0.75]    (10.93,-3.29) .. controls (6.95,-1.4) and (3.31,-0.3) .. (0,0) .. controls (3.31,0.3) and (6.95,1.4) .. (10.93,3.29)   ;
\draw [shift={(290,80)}, rotate = 0] [color={rgb, 255:red, 0; green, 0; blue, 0 }  ][line width=0.75]    (10.93,-3.29) .. controls (6.95,-1.4) and (3.31,-0.3) .. (0,0) .. controls (3.31,0.3) and (6.95,1.4) .. (10.93,3.29)   ;
\draw   (410,186) .. controls (410,177.16) and (417.16,170) .. (426,170) -- (654,170) .. controls (662.84,170) and (670,177.16) .. (670,186) -- (670,234) .. controls (670,242.84) and (662.84,250) .. (654,250) -- (426,250) .. controls (417.16,250) and (410,242.84) .. (410,234) -- cycle ;
\draw    (280,210) -- (408,210) ;
\draw [shift={(410,210)}, rotate = 180] [color={rgb, 255:red, 0; green, 0; blue, 0 }  ][line width=0.75]    (10.93,-3.29) .. controls (6.95,-1.4) and (3.31,-0.3) .. (0,0) .. controls (3.31,0.3) and (6.95,1.4) .. (10.93,3.29)   ;
\draw    (540,112) -- (540,168) ;
\draw [shift={(540,170)}, rotate = 270] [color={rgb, 255:red, 0; green, 0; blue, 0 }  ][line width=0.75]    (10.93,-3.29) .. controls (6.95,-1.4) and (3.31,-0.3) .. (0,0) .. controls (3.31,0.3) and (6.95,1.4) .. (10.93,3.29)   ;
\draw [shift={(540,110)}, rotate = 90] [color={rgb, 255:red, 0; green, 0; blue, 0 }  ][line width=0.75]    (10.93,-3.29) .. controls (6.95,-1.4) and (3.31,-0.3) .. (0,0) .. controls (3.31,0.3) and (6.95,1.4) .. (10.93,3.29)   ;
\draw    (150,170) -- (150,112) ;
\draw [shift={(150,110)}, rotate = 90] [color={rgb, 255:red, 0; green, 0; blue, 0 }  ][line width=0.75]    (10.93,-3.29) .. controls (6.95,-1.4) and (3.31,-0.3) .. (0,0) .. controls (3.31,0.3) and (6.95,1.4) .. (10.93,3.29)   ;
\draw   (410,62) .. controls (410,55.37) and (415.37,50) .. (422,50) -- (658,50) .. controls (664.63,50) and (670,55.37) .. (670,62) -- (670,98) .. controls (670,104.63) and (664.63,110) .. (658,110) -- (422,110) .. controls (415.37,110) and (410,104.63) .. (410,98) -- cycle ;
\draw   (20,186) .. controls (20,177.16) and (27.16,170) .. (36,170) -- (264,170) .. controls (272.84,170) and (280,177.16) .. (280,186) -- (280,234) .. controls (280,242.84) and (272.84,250) .. (264,250) -- (36,250) .. controls (27.16,250) and (20,242.84) .. (20,234) -- cycle ;

\draw (150,52) node [anchor=north][inner sep=0.75pt]   [align=left] {\textbf{Problem~\eqref{eq:CSE}}};
\draw (150,75) node [anchor=north][inner sep=0.75pt]    {$\{X\opt\}\triangleq\displaystyle\Argmin_{X\in\PSD^p} \max_{\cov \in \B_\eps (\covsa)} \Tr{X^2} - 2 \Tr{ \cov  X }$};
\draw (540,52) node [anchor=north][inner sep=0.75pt]   [align=left] {\textbf{Problem~\eqref{eq:matrix}}};
\draw (540,75) node [anchor=north][inner sep=0.75pt]    {$\{\cov\opt\}\triangleq\displaystyle\Argmin_{\cov\in\PSD^p}\left\{\|\cov\|_{\rm F}^2:\; D(\cov,\covsa)\leq\varepsilon \right\}$};
\draw (350,58) node [anchor=north][inner sep=0.75pt]    {$X^{\star } =\Sigma ^{\star }$};
\draw (350,85) node [anchor=north][inner sep=0.75pt]   [align=left] {Proposition~\ref{prop:exist}};
\draw (540,175) node [anchor=north][inner sep=0.75pt]   [align=left] {\textbf{Problem~\eqref{eq:vector}}};
\draw (540,195) node [anchor=north][inner sep=0.75pt]    {$\displaystyle\{x\opt\}\triangleq\Argmin_{x\in\mathbb{R}_+^p}\bigg\{\|x\|_2^2:\; \sum_{i=1}^p d(x_i,\widehat{x}_i)\leq\varepsilon\bigg\}$};
\draw (150,175) node [anchor=north][inner sep=0.75pt]   [align=left] {\textbf{First-order condition}};
\draw (150,193) node [anchor=north][inner sep=0.75pt]    {$\gamma ^{\star}>0$ unique solution of};
\draw (345,188) node [anchor=north][inner sep=0.75pt]    {$x_{i}^{\star } =s\left( \gamma ^{\star } ,\widehat{x}_{i}\right)\; \forall i$};
\draw (345,215) node [anchor=north][inner sep=0.75pt]   [align=left] {Proposition~\ref{prop:gamma-reconstruction}};
\draw (450,140) node [anchor=west][inner sep=-3pt]   [align=left] {Proposition~\ref{prop:vec-equivalence}};
\draw (545,139) node [anchor=south west][inner sep=0.75pt]    {$\Sigma ^{\star } =\widehat{V} \Diag(x\opt)\ \widehat{V}^{\top }$};
\draw (545,139) node [anchor=north west][inner sep=0.75pt]    {$x\opt_i=\lambda_i(\cov\opt)\;\forall i$};
\draw (150,205) node [anchor=north][inner sep=0.75pt]    {$\displaystyle\sum _{i=1}^{p} d( s( \gamma ,\widehat{x}_{i}) ,\widehat{x}_{i}) -\varepsilon =0$};
\draw (80,140) node [anchor=west][inner sep=-3pt]   [align=left] {Theorem~\ref{thm:general_CSE}};
\draw (155,139) node [anchor=west][inner sep=0.75pt]    {$X^{\star } =\widehat{V}\Diag(s(\gamma\opt,\widehat{x}_i))\widehat{V}^\top$};

\end{tikzpicture}